\documentclass[twoside,11pt]{article}

\usepackage{jmlr2e}

\usepackage{algorithm}  
\usepackage{algpseudocode}  
\usepackage{amsmath}
\usepackage{bbm}
\usepackage{multirow}
\usepackage{chngpage}
\usepackage{verbatim}
\usepackage{mathtools}
\usepackage{booktabs}
\usepackage{float}
\usepackage{array}
\newcolumntype{P}[1]{>{\centering\arraybackslash}p{#1}}
\usepackage{tabularx}
\usepackage{enumitem}
\usepackage{caption}
\usepackage{multicol}
\newtheorem{assumption}{Assumption}

\hypersetup{
    colorlinks=true,
    linkcolor=blue,
    filecolor=magenta,      
    urlcolor=cyan,
    citecolor={blue}
}

\firstpageno{1}

\title{Calibrating the Adaptive Learning Rate to Improve Convergence of ADAM}

\author{\name Qianqian Tong \email qianqian.tong@uconn.edu \\
       \name Guannan Liang \email guannan.liang@uconn.edu \\
       \name Jinbo Bi \email jinbo.bi@uconn.edu\\
       \addr Computer Science and Engineering\\University of Connecticut,
       Storrs, CT 06269}
       
\begin{document}
\editor{ }
\maketitle

\begin{abstract}
Adaptive gradient methods (AGMs) have become popular in optimizing the nonconvex problems in deep learning area. We revisit AGMs and identify that the adaptive learning rate (A-LR) used by AGMs varies significantly across the dimensions of the problem over epochs (i.e., anisotropic scale), which may lead to issues in convergence and generalization. All existing modified AGMs actually represent efforts in revising the A-LR. Theoretically, we provide a new way to analyze the convergence of AGMs and prove that the convergence rate of \textsc{Adam} also depends on its hyper-parameter $\epsilon$, which has been overlooked previously. Based on these two facts, we propose a new AGM by calibrating the A-LR with an activation ({\em softplus}) function, resulting in the  \textsc{Sadam} and \textsc{SAMSGrad} methods. We further prove that these algorithms enjoy better convergence speed under nonconvex, non-strongly convex, and Polyak-{\L}ojasiewicz conditions compared with \textsc{Adam}.  Empirical studies support our observation of the anisotropic A-LR and show that the proposed methods outperform existing AGMs and generalize even better than S-Momentum in multiple deep learning tasks.
\end{abstract}

\section{Introduction}
Many machine learning problems can be formulated as the minimization of an objective function $f$ of the form: $\min_{x\in \mathbb{R}^d}f(x) = \frac{1}{n} \sum_{i=1}^n f_{i}(x),$ where both $f$ and $f_i$ maybe nonconvex in deep learning.
Stochastic gradient descent (SGD), its variants such as SGD with momentum (S-Momentum) \citep{ghadimi2013stochastic,wright1999numerical,wilson2016lyapunov,yang2016unified}, and adaptive gradient methods (AGMs) \citep{duchi2011adaptive,kingma2014adam,zeiler2012adadelta} play important roles in deep learning area due to simplicity and wide applicability. In particular, AGMs often exhibit fast initial progress in training and are easy to implement in solving large scale optimization problems. The updating rule of AGMs can be generally written as:
\begin{equation}
\label{eq1}
    x_{t+1}=x_{t}- \frac{\eta_t}{\sqrt{v_t}}\odot m_t,
\end{equation}
where $\odot$ calculates element-wise product of the first-order momentum $m_t$ and the learning rate (LR) $\frac{\eta_t}{\sqrt{v_t}}$.
There is fairly an agreement on how to compute $m_t$, which is a convex combination of previous $m_{t-1}$ and current stochastic gradient $g_t$, i.e., $m_t=\beta_1 m_{t-1}+(1-\beta_1)g_t$, $\beta_1\in [0,1]$. The LR consists of two parts: the base learning rate (B-LR) $\eta_t$ is a scalar which can be constant or decay over iterations. In our convergence analysis, we consider the B-LR as constant $\eta$. The adaptive learning rate (A-LR), $\frac{1}{\sqrt{v_t}}$, varies adaptively across dimensions of the problem, where $v_t\in \mathbb{R}^d$ is the second-order momentum calculated as a combination of previous and current squared stochastic gradients. Unlike the first-order momentum, the formula to estimate the second-order momentum varies in different AGMs. As the core technique in AGMs, A-LR opens a new regime of controlling LR, and allows the algorithm to move with different step sizes along the search direction at different coordinates.

The first known AGM is \textsc{Adagrad} \citep{duchi2011adaptive} where the second-order momentum is estimated as $v_t = \sum_{i=1}^{t}g_i^2$. It works well in sparse settings, but the A-LR often decays rapidly for dense gradients. To tackle this issue, \textsc{Adadelta} \citep{zeiler2012adadelta}, \textsc{Rmsprop} \citep{tieleman2012lecture}, \textsc{Adam} \citep{kingma2014adam} have been proposed to use exponential moving averages of past squared gradients, i.e., $v_t = \beta_2 v_{t-1}+ (1-\beta_2)g_t^2$, $\beta_2\in [0,1]$ and calculate the A-LR by $\frac{1}{\sqrt{v_t}+\epsilon}$ where $\epsilon>0$ is used in case that $v_t$ vanishes to zero. In particular, \textsc{Adam} has become the most popular optimizer in the deep learning area due to its effectiveness in early training stage. Nevertheless, it has been empirically shown that \textsc{Adam} generalizes worse than S-Momentum to unseen data and leaves a clear generalization gap \citep{he2016deep,zagoruyko2016wide,huang2017densely}, and even fails to converge in some cases \citep{reddi2018convergence,luo2019adaptive}. AGMs decrease the objective value rapidly in early iterations, and then stay at a plateau whereas SGD and S-Momentum continue to show dips in the training error curves, and thus continue to improve test accuracy over iterations. 
It is essential to understand what happens to \textsc{Adam} in the later learning process, so we can revise AGMs to enhance their generalization performance.

Recently, a few modified AGMs have been developed, such as, \textsc{AMSGrad} \citep{reddi2018convergence}, \textsc{Yogi} \citep{zaheer2018adaptive}, and \textsc{AdaBound} \citep{luo2019adaptive}. 
\textsc{AMSGrad} is the first method to theoretically address the non-convergence issue of \textsc{Adam} by taking the largest second-order momentum estimated in the past iterations, i.e., $v_t = max\{v_{t-1}, \tilde{v}_{t}\}$ where $\tilde{v}_t = \beta_2 {\tilde{v}}_{t-1}+ (1-\beta_2)g_t^2$, and proves its convergence in the convex case. The analysis is later extended to other AGMs (such as \textsc{RMSProp} and \textsc{AMSGrad}) in nonconvex settings \citep{zhou2018convergence,chen2018convergence,de2018convergence,staib2019escaping}. 
\textsc{Yogi} claims that the past $g_t^2$'s are forgotten in a fairly fast manner in \textsc{Adam} and proposes $v_t = v_{t-1} - (1-\beta_2)sign( v_{t-1} - g_t^2)g_t^2$ to adjust the decay rate of the A-LR and $\epsilon$ is adjusted to $10^{-3}$ to improve performance. 
\textsc{PAdam}\footnote{The \textsc{PAdam} in \citep{chen2018closing} actually used \textsc{AMSGrad}, and for clear  comparison, we named it \textsc{PAMSGrad}. In our experiments, we also compared with the \textsc{Adam} that used the A-LR formula with $p$, which we named \textsc{PAdam}.} \citep{chen2018closing,zhou2018convergence} claims that the A-LR in \textsc{Adam} and \textsc{AMSGrad} are ``overadapted'', and proposes to replace the A-LR updating formula by $1/((v_t)^p +\epsilon)$ where $p \in (0, 1/2]$.
\textsc{AdaBound} confines the LR to a predefined range by applying $Clip(\frac{\eta}{\sqrt{v_t}}, \eta_l, \eta_r)$, where LR values outside the interval $[\eta_l, \eta_r]$ are clipped to the interval edges. 
However, a more effective way is to softly and smoothly calibrate the A-LR rather than hard-thresholding the A-LR at all coordinates.

Our main contributions are summarized as follows:
\begin{itemize}
  \item We study AGMs from a new perspective: the range of the A-LR. Through experimental studies, we find that the A-LR is always anisotropic. This anisotropy may lead the algorithm to focus on a few dimensions (those with large A-LR), which may exacerbate generalization performance. We analyze the existing modified AGMs to help explain how they close the generalization gap. 
  
  \item Theoretically, we are the first to include hyper-parameter $\epsilon$ into the convergence analysis and clearly show that the convergence rate is upper bounded by a $1/\epsilon^2$ term, verifying prior observations that $\epsilon$ affects performance of \textsc{Adam} empirically. We provide a new approach to convergence analysis of AGMs under the nonconvex, non-strongly convex, or Polyak-${\L}$ojasiewicz (P-L) condition.
 
  \item Based on the above two results, we propose to calibrate the A-LR using an activation function, particularly we implement the {\em softplus} function with a hyper-parameter $\beta$, which can be combined with any AGM. In this work, we combine it with \textsc{Adam} and \textsc{AMSGrad} to form the \textsc{Sadam} and \textsc{SAMSGrad} methods. 
  
  \item We also provide theoretical guarantees of our methods, which enjoy better convergence speed than \textsc{Adam} and recover the same convergence rate as  SGD in terms of the maximum iteration $T$ as $O(1/\sqrt{ T})$ rather than the known result: $O(\log(T) /\sqrt{T})$ in \citep{chen2018convergence}. 
  Empirical evaluations show that our methods obviously increase test accuracy, and outperform many AGMs and even S-Momentum in multiple deep learning models.
\end{itemize}

\section{Preliminary}
\textbf{Notation.} For any vectors $a,b \in \mathbb{R}^d$, we use $a\odot b$ for element-wise product, $a^2$ for element-wise square, $\sqrt{a}$ for element-wise square root, $a/b$ for element-wise division; we use $a^k$ to denote element-wise power of $k$, and $\|a\|$ to denote its $l_2$-norm. We use $\langle a,b \rangle$ to denote their inner product, $max\{a,b\}$ to compute element-wise maximum. $e$ is the Euler number, $\log(\cdot)$ denotes logarithm function with base $e$, and $O(\cdot)$ to hide constants which do not rely on the problem parameters.

\noindent\textbf{Optimization Terminology.}
In convex setting, the optimality gap, $f(x_t)-f^*$, is examined where $x_t$ is the iterate at iteration $t$, and $f^*$ is the optimal value attained at $x^*$ assuming that $f$ does have a minimum. When $f(x_t)-f^* \leq \delta$, it is said that the method reaches an optimal solution with $\delta$-accuracy. However, in the study of AGMs, the average regret $\frac{1}{T}\sum_{t=1}^T (f(x_t) - f^*)$ (where the maximum iteration number $T$ is pre-specified) is used to approximate the optimality gap to define  $\delta$-accuracy. 
Our analysis moves one step further to examine if $f(\frac{1}{T}\sum_{t=1}^T x_t) - f^* \le \delta$ by applying Jensen's inequality to the regret.

In nonconvex setting, finding the global minimum or even local minimum is NP-hard, so optimality gap is not examined. Rather, it is common to evaluate if a first-order stationary point has been achieved \citep{reddi2015variance,reddi2018convergence,zaheer2018adaptive}. More precisely, we evaluate if $E[\|\nabla f(x_t)\|^2] \leq \delta$ (e.g., in the analysis of SGD \citep{ghadimi2013stochastic}). The convergence rate of SGD is $O(1/\sqrt{T})$ in both non-strongly convex and nonconvex settings. Requiring $O(1/\sqrt{T})\leq \delta$ yields the maximum number of iterations $T=O(1/\delta^2)$. Thus, SGD can obtain a $\delta$-accurate solution in $O(1/\delta^2)$ steps in non-strongly convex and nonconvex settings. Our results recover the rate of SGD and S-Momentum in terms of $T$.

\begin{assumption}\label{assumption}
The loss $f_{i}$ and the objective $f$ satisfy:
 \renewcommand{\labelenumii}{\Roman{enumii}} 
 \begin{enumerate}
   \item \textbf{L-smoothness.} $\forall x, y \in \mathbb{R}^d, \forall i \in \{1,...,n\}$, $\|\nabla f_{i}(x)-\nabla f_{i}(y)\|\leq L \|x-y\|$.
   \item \textbf{Gradient bounded.} $\forall x \in \mathbb{R}^d, \forall i \in \{1,...,n\}$, $\|\nabla f_i(x)\|\leq G$, $G\geq 0$.
   \item \textbf{Variance bounded.} $\forall x\in \mathbb{R}^d$, $t\geq 1$, $E[g_t]=\nabla f(x_t)$, $E[\| g_t-\nabla f(x_t)\|^2]\leq \sigma^2$.
 \end{enumerate}
\end{assumption}

\begin{definition}\label{def}
Suppose $f$ has the global minimum, denoted as $f^* = f(x^*)$. Then for any $x, y\in \mathbb{R}^d$,
\renewcommand{\labelenumii}{\Roman{enumii}}
 \begin{enumerate}
   \item \textbf{Non-strongly convex.} $f(y)\geq f(x) + \nabla f(x)^T (y-x)$.
   \item \textbf{Polyak-{\L}ojasiewicz (P-L) condition.} $\exists \lambda >0$ such that $\|\nabla f (x)\|^2 \geq 2\lambda (f(x)-f^*)$.
   \item \textbf{Strongly convex.} $\exists~ \mu >0$ such that $f(y)\geq f(x) + \nabla f(x)^T (y-x) + \frac{\mu}{2}\|y-x\|^2$.
 \end{enumerate}
\end{definition}

\section{Our New Analysis of \textsc{Adam}}

First, we empirically observe that \textsc{Adam} has anisotropic A-LR, which may lead to poor generalization performance. Second,we theoretically show \textsc{Adam} method is sensitive to $\epsilon$, supporting observations in previous work.

\subsection{Anisotropic A-LR.}
We investigate how the A-LR in \textsc{Adam} varies over time and across problem dimensions, and plot four examples in Figure \ref{fig:range} (more figures in Appendix) where we run \textsc{Adam} to optimize a convolutional neural network (CNN) on the MNIST dataset, and ResNets or DenseNets on the CIFAR-10 dataset.
The curves in Figure \ref{fig:range} exhibit very irregular shapes, and the median value is hardly placed in the middle of the range, the range of A-LR across the problem dimensions is anisotropic for AGMs. As a general trend, the A-LR becomes larger when $v_t$ approaches $0$ over iterations.
The elements in the A-LR vary significantly across dimensions and there are always some coordinates in the A-LR of AGMs that reach the maximum $10^8$ determined by $\epsilon$ (because we use $\epsilon = 10^{-8}$ in \textsc{Adam}). 

This anisotropic scale of A-LR across dimensions makes it difficult to determine the B-LR, $\eta$. On the one hand, $\eta$ should be set small enough so that the LR $\frac{\eta}{\sqrt{v_t}+\epsilon}$ is appropriate, or otherwise some coordinates will have very large updates because the corresponding A-LR's are big, likely resulting in performance oscillation \citep{kleinberg2018alternative}.
This may be due to that exponential moving average of past gradients is different, hence the speed of $m_t$ diminishing to zero is different from the speed of $\sqrt{v_t}$ diminishing to zero. Besides, noise generated in stochastic algorithms has nonnegligible influence to the learning process. 
On the other hand, very small $\eta$ may harm the later stage of the learning process  since the small magnitude of $m_t$ multiplying with a small step size (at some coordinates) will be too small to escape sharp local minimal, which has been shown to lead to poor generalization \citep{keskar2016large,chaudhari2016entropy,li2018visualizing}. Further, in many deep learning tasks, stage-wise policies are often taken to decay the LR after several epochs, thus making the LR even smaller. To address the dilemma, it is essential to control the A-LR, especially when stochastic gradients get close to 0.

By analyzing previous modified AGMs that aim to close the generalization gap, we find that all these works can be summarized into one technique: constraining the A-LR, $1/(\sqrt{v_t}+\epsilon)$, to a reasonable range. Based on the observation of anisotropic A-LR, we propose a more effective way to calibrate the A-LR according to an activation function rather than hard-thresholding the A-LR at all coordinates, empirically improve generalization performance with theoretical guarantees of both optimization and generalization error analysis. 

\begin{figure}[t]
  \centering
  \includegraphics[width=0.9\linewidth]{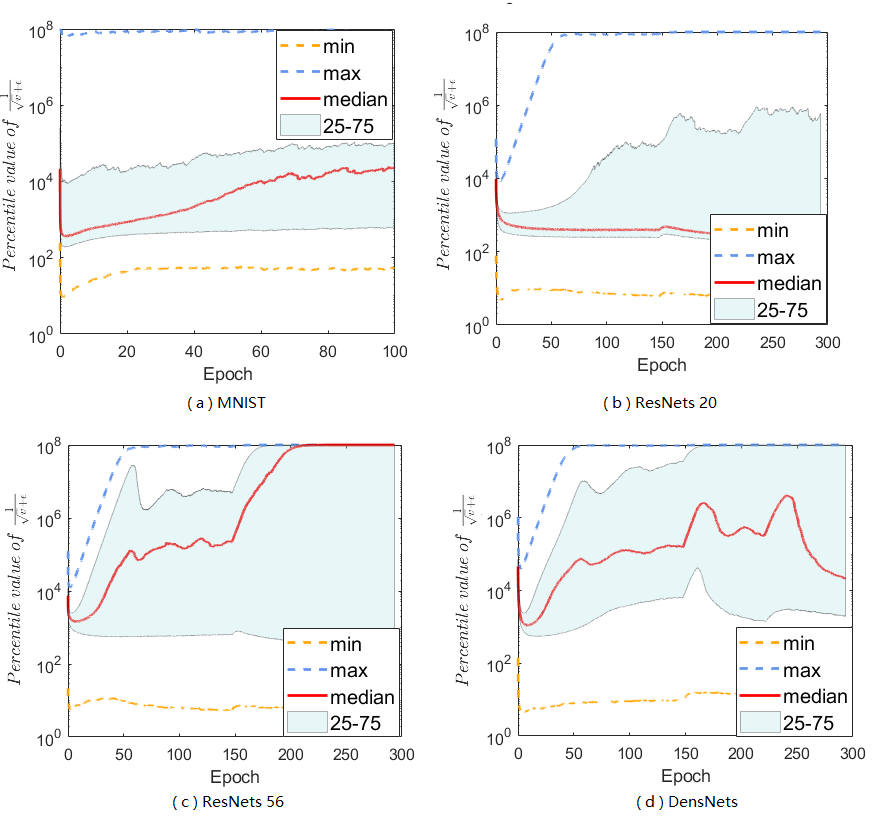}
  \caption{Range of the A-LR in \textsc{Adam} over iterations in four settings: (a) CNN on MNIST, (b) ResNet20 on CIFAR-10, (d) ResNet56 on CIFAR-10, (d) DenseNets on CIFAR-10. We plot the min, max, median, and the 25 and 75 percentiles of the A-LR across dimensions (the elements in  
  $\frac{1}{\sqrt{v_t}+\epsilon}$).}
  \label{fig:range}
\end{figure}

\subsection{Sensitive to $\epsilon$.}
As a hyper-parameter in AGMs, $\epsilon$ is originally introduced to avoid the zero denominator issue when $v_t$ goes to 0, and has never been studied in the convergence analysis of AGMs. However, it has been empirically observed that AGMs can be sensitive to the choice of $\epsilon$ \citep{de2018convergence,zaheer2018adaptive}. As shown in Figure \ref{fig:range}, a smaller $\epsilon = 10^{-8}$ leads to a wide span of the A-LR across the different dimensions, whereas a bigger $\epsilon = 10^{-3}$ as used in \textsc{Yogi}, reduces the span. The setting of $\epsilon$ is the main force causing anisotropy, unsatisfied, there has no theoretical result explains the effect of $\epsilon$ on AGMs. Inspired by our observation, we believe that the current convergence analysis for ADAM is not complete if omitting $\epsilon$. 

Most of the existing convergence analysis follows the line in \citep{reddi2018convergence} to first project the sequence of the iterates into a minimization problem as $x_{t+1} =x_t- \frac{\eta}{\sqrt{v_t}}m_t= \min_x \|v_t^{1/4}\big(x-(x_t- \frac{\eta}{\sqrt{v_t}}m_t)\big)\|$, and then examine if $||v_t^{1/4} (x_{t+1} - x^*)||$ decreases over iterations. Hence, $\epsilon$ is not discussed in this line of proof because it is not included in the step size. In our later convergence analysis section, we introduce an important lemma, bounded A-LR, and by using the bounds of the A-LR (specifically, the lower bound $\mu_1$ and upper bound $\mu_2$ both containing $\epsilon$ for \textsc{Adam}), we give a new general framework of prove (details in Appendix) to show the convergence rate for reaching an $x$ that satisfies $E[\|\nabla f(x_t)\|^2] \leq \delta$ in the nonconvex setting. Then, we also derive the optimality gap from the stationary point in the convex and P-L settings (strongly convex).

\begin{theorem}\label{th1}
\textbf{[Nonconvex]}
Suppose $f(x)$ is a nonconvex function that satisfies Assumption \ref{assumption}. Let $\eta_t = \eta = O(\frac{1}{\sqrt{T}})$, \textsc{ADAM} has
\begin{equation*}
    \min_{t=1,\dots,T}E[\|\nabla f(x_t)\|^2] \leq O( \frac{1}{\epsilon^2 \sqrt{T}}+ \frac{d}{\epsilon T} + \frac{d}{\epsilon^2 T\sqrt{T}}).
\end{equation*}
\end{theorem}

\begin{theorem}\label{th2}
\textbf{[Non-strongly Convex]}
Suppose $f(x)$ is a convex function that satisfies Assumption \ref{assumption}. Assume that $\forall t$, $E[\|x_{t}-x^* \| \leq D$, for any $m \ne n$, $E[\|x_{m}-x_{n} \|] \leq D_{\infty}$, let $\eta_t = \eta = O(\frac{1}{\sqrt{T}})$, \textsc{ADAM} has convergence rate $f(\Bar{x}_t)-f^*\leq O(\frac{d}{\epsilon^2\sqrt{T}})$, where $\Bar{x}_t = \frac{1}{T}\sum_{t=1}^T x_t$.
\end{theorem}

\begin{theorem}\label{th3}
\textbf{[P-L Condition]}
Suppose $f(x)$ has P-L condition (with parameter $\lambda$) holds under convex case, satisfying Assumption \ref{assumption}. Let $\eta_t = \eta = O(\frac{1}{T^2})$, \textsc{ADAM} has the convergence rate:
$\;\;E[f(x_{T+1})-f^*]\leq (1-\frac{2\lambda \mu_1}{T^2})^T E[f(x_1)-f^*]+O(\frac{1}{T})$,
\end{theorem}

The P-L condition is weaker than strongly convex, and for the strongly-convex case, we also have:

\begin{corollary}\textbf{[Strongly Convex]}
Suppose $f(x)$ is $\mu$-strongly convex function that satisfies Assumption \ref{assumption}. Let $\eta_t = \eta = O(\frac{1}{T^2})$, \textsc{Adam} has the convergence rate:
$E[f(x_{T+1})-f^*]\leq (1-\frac{2\mu \mu_1}{T^2})^T E[f(x_1)-f^*]+O(\frac{1}{T})$
\end{corollary}

This is the first time to theoretically include $\epsilon$ into analysis. As expected, the convergence rate of \textsc{Adam} is highly related with $\epsilon$. A bigger $\epsilon$ will enjoy a better convergence rate since $\epsilon$ will dominate the A-LR and behaves like S-Momentum; A smaller $\epsilon$ will preserve stronger ``adaptivity'', we need to find a better way to control $\epsilon$. 

\section{The Proposed Algorithms}

We propose to use activation functions to calibrate AGMs, and specifically focus on using {\em softplus} funciton on top of \textsc{Adam} and \textsc{AMSGrad} methods.

\subsection{Activation Functions Help}
Activation functions (such as sigmoid, ELU, $\tanh$) transfer inputs to outputs are widely used in deep learning area. As a well-studied activation function,  $softplus(x)=\frac{1}{\beta}\log(1+e^{\beta x})$ is known to keep large values unchanged (behaved like function $y = x$) while smoothing out small values (see Figure \ref{fig:soft} (a)). The target magnitude to be smoothed out can be adjusted by a hyper-parameter $\beta \in \mathbb{R}$. 
In our new algorithms, we introduce $softplus(\sqrt{v_t})= \frac{1}{\beta}\log(1+e^{\beta\cdot \sqrt{ v_t}})$ to smoothly calibrate the A-LR. This calibration brings the following benefits: (1) constraining extreme large-valued A-LR in some coordinates (corresponding to the small-values in $v_t$) while keeping others untouched with appropriate $\beta$. For the undesirable large values in the A-LR, the {\em softplus} function condenses them smoothly instead of hard thresholding. For other coordinates, the A-LR largely remains unchanged; (2) removing the sensitive parameter $\epsilon$ because the {\em softplus} function can be lower-bounded by a nonzero number when used on non-negative variables, $softplus(\cdot)\geq\frac{1}{\beta} \log2$.

\begin{figure}[h]
  \centering
  \includegraphics[width= 0.85\linewidth]{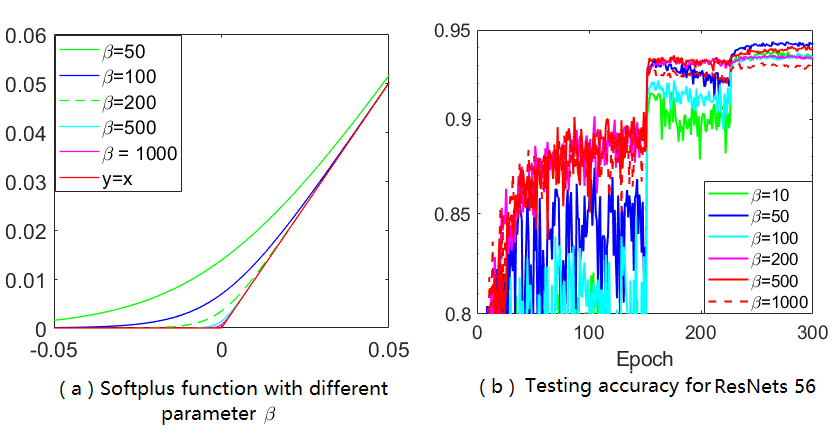}
  \caption{Behavior of the softplus function, and the test performance of our \textsc{Sadam} algorithm.}
  \label{fig:soft}
\end{figure}

After calibrating $\sqrt{v_t}$ with a {\em softplus} function, the anisotropic A-LR becomes much more regulated (see Figure \ref{fig:beta} and Appendix), and we clearly observe improved  test accuracy (Figure \ref{fig:soft} (b) and more figures in Appendix). We name this method ``\textsc{Sadam}'' to represent the calibrated \textsc{Adam} with {\em softplus} function, here we recommend using {\em softplus} function but it is not limited to that, and the later theoretical analysis can be easily extended to other activation functions. More empirical evaluations have shown that the proposed methods significantly improve the generalization performance of \textsc{Adam} and \textsc{AMSGrad}.

\subsection{Calibrated AGMs}

With activation function, we develop two new variants of AGMs: \textsc{Sadam} and \textsc{SAMSGrad} (Algorithms \ref{alg:sadam} and \ref{alg:samsgrad}), which are developed based on \textsc{Adam} and \textsc{AMSGrad} respectively. 

\begin{algorithm}[h]
\caption{\textsc{Sadam}}
\label{alg:sadam}
\begin{algorithmic}[1]
     	\State {\bfseries Input:} $x_1\in \mathbb{R}^d$, learning rate $\{\eta_t\}_{t=1}^T$, parameters $0\leq \beta_1, \beta_2 < 1$, $\beta$.
		\State {\bfseries Initialize} $m_0 = 0$, $v_0 = 0$
		\For {$t = 1$ to $T$}
		\State Compute stochastic gradient $g_t$
		\State $m_t = \beta_1 m_{t-1}+(1-\beta_1)g_t$
		\State  $v_t = \beta_2 v_{t-1}+ (1-\beta_2)g_t^2$  
		\State $x_{t+1}=x_t-\frac{\eta_t}{softplus(\sqrt{v_t})}\odot m_t$
	    \EndFor
\end{algorithmic}
\end{algorithm}
\begin{algorithm}[h]
\caption{\textsc{SAMSGrad}}
\label{alg:samsgrad}
\begin{algorithmic}[1]
     	\State {\bfseries Input:} $x_1\in \mathbb{R}^d$, learning rate $\{\eta_t\}_{t=1}^T$, parameters $0\leq \beta_1, \beta_2 < 1$, $\beta$.
		\State {\bfseries Initialize} $m_0 = 0$, $\tilde{v}_0 = 0$
		\For {$t = 1$ to $T$}
		\State Compute stochastic gradient $g_t$
		\State $m_t = \beta_1 m_{t-1}+(1-\beta_1)g_t$
		\State  $\tilde{v}_t = \beta_2 \tilde{v}_{t-1}+ (1-\beta_2)g_t^2$
		\State $v_t = max\{v_{t-1}, \tilde{v}_{t}\}$ 
		\State $x_{t+1}=x_t-\frac{\eta_t}{softplus(\sqrt{v_t})}\odot m_t$
	    \EndFor 
\end{algorithmic}
\end{algorithm}

The key step lies in the way to design the adaptive functions, instead of using the generalized square root function only, we apply $softplus(\cdot)$ on top of the square root of the second-order momentum, which serves to regulate A-LR's anisotropic behavior and replace the tolerance parameter $\epsilon$ by the hyper-parameter $\beta$ used in the {\em softplus} function.

In our algorithms, the hyper-parameters are recommended as $\beta_1 = 0.9$, $\beta_2=0.999$.
For clarity, we omit the bias correction step proposed in the original \textsc{Adam}. However, our arguments and theoretical analysis are applicable to the bias correction version as well \citep{kingma2014adam,dozat2016incorporating,zaheer2018adaptive}. Using the {\em softplus} function, we introduce a new hyper-parameter $\beta$, which performs as a controller to smooth out anisotropic A-LR, and connect the \textsc{Adam} and S-Momentum methods automatically. When $\beta$ is set to be small, \textsc{Sadam} and \textsc{SAMSGrad} perform similarly to S-Momentum; when $\beta$ is set to be big, $softplus(\sqrt{v_t})= \frac{1}{\beta}\log(1+e^{\beta\cdot \sqrt{ v_t}}) \approx \frac{1}{\beta}\log(e^{\beta\cdot \sqrt{ v_t}})=\sqrt{ v_t}$, and the updating formula becomes $x_{t+1} = x_{t} - \frac{\eta_t}{\sqrt{v_t}}\odot m_t$, which is degenerated into the original AGMs. The hyper-parameter $\beta$ can be well tuned to achieve the best performance for different datasets and tasks. Based on our empirical observations, we recommend to use $\beta = 50$.

As a calibration method, the {\em softplus} function has better adaptive behavior than simply setting $\epsilon$. More precisely, when $\epsilon$ is large or $\beta$ is small, \textsc{Adam} and \text{AMSGrad} amount to S-Momentum, but when $\epsilon$ is small as commonly suggested $10^{-8}$ or $\beta$ is taken large, the two methods are different because comparing Figure \ref{fig:range} and \ref{fig:beta} yields that \textsc{Sadam} has more regulated A-LR distribution. The proposed calibration scheme regulates the massive range of A-LR back down to a moderate scale. The median of A-LR in different dimensions is now well positioned to the middle of the $25$-$75$ percentile zone. 
Our approach opens up a new direction to examine other activation functions (not limited to the {\em softplus} function) to calibrate the A-LR.

\begin{figure}[H]
  \centering
  \includegraphics[width=0.85\linewidth]{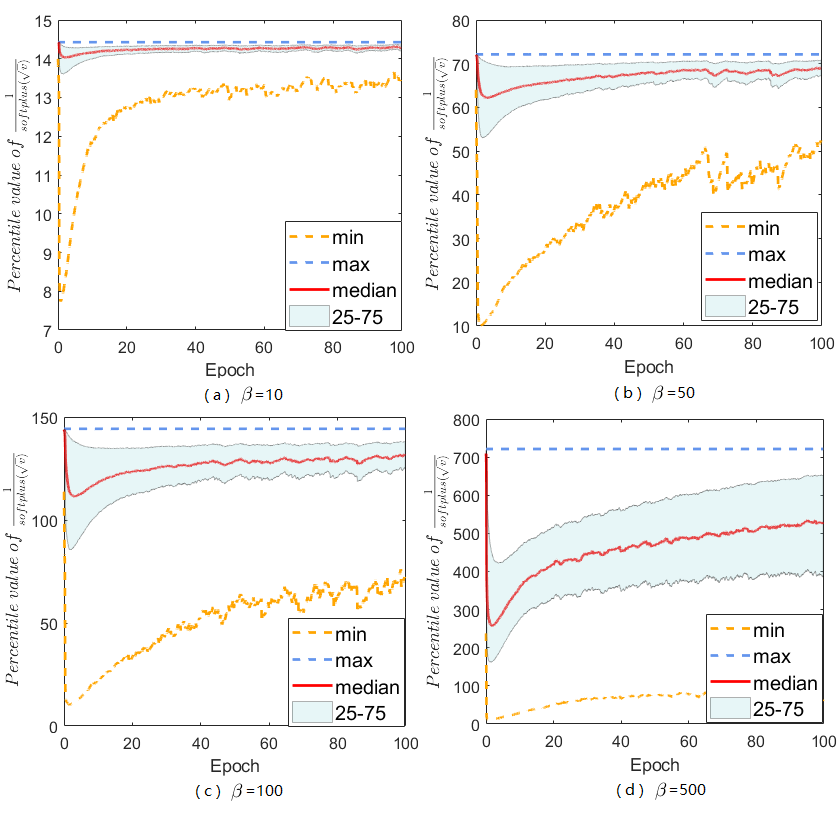}
  \caption{Behavior of the A-LR in the \textsc{Sadam} method with different choices of $\beta$ (CNN on the MNIST data).}
  \label{fig:beta}
\end{figure}

The proposed \textsc{Sadam} and \textsc{SAMSGrad} can be treated as members of a class of AGMs that use the {\em softplus} (or another suitable activation) function to better adapt the step size. It can be readily combined with any other AGM, e.g., \textsc{Rmsrop}, \textsc{Yogi}, and \textsc{PAdam}. These methods may easily go back to the original ones by choosing a big $\beta$. 

\section{Convergence Analysis}

We first demonstrate an important lemma to highlight that every coordinate in the A-LR is both upper and lower bounded at all iterations, which is consistent with empirical observations (Figure 1), and forms the foundation of our proof.

\begin{lemma}\textbf{[Bounded A-LR]}\label{bound}
With Assumption \ref{assumption}, for any $t\geq 1$, $j \in [1, d]$, $\beta_2\in [0,1]$, and $\epsilon$ in \textsc{Adam}, $\beta$ in \textsc{Sadam}, anisotropic A-LR is bounded in AGMs,

\textsc{Adam} has $(\mu_1, \mu_2)$-bounded A-LR:
$$\mu_1 \leq \frac{1}{\sqrt{v_{t,j}}+\epsilon} \leq \mu_2,$$
\textsc{Sadam} has $(\mu_3, \mu_4)$-bounded A-LR: 
$$ \mu_3 \leq \frac{1}{softplus(\sqrt{v_{t,j}})} \leq \mu_4,$$
where $ 0 <\mu_1\leq \mu_2$, and $ 0 <\mu_3\leq \mu_4$. 
\end{lemma}
        
\begin{remark}
Besides the square root function and {\em softplus} function, the A-LR calibrated by any positive monotonically increasing function can be bounded. All of the bounds can be shown to be related to $\epsilon$ or $\beta$ (see Appendix). Bounded A-LR is an essential foundation in our analysis, we provide a different way of proof from previous works, and the proof procedure can be easily extended to other gradient methods as long as bounded LR is satisfied.          
\end{remark}

\begin{remark}
These bounds can be applied to all AGMs, including \textsc{Adagrad}. In fact, the lower bounds actually are not the same in \textsc{Adam} and \textsc{Adagrad}, because \textsc{Adam} will have smaller $\sqrt{v_{t,j}}$ due to moment decay parameter $\beta_2$. To achieve a unified result, we use the same relaxation to derive the fixed lower bound $\mu_1$.
\end{remark}

We now describe our main results of \textsc{Sadam} (and \textsc{SAMSGrad}) in the nonconvex case, we clearly show that similar to Theorem \ref{th1}, the convergence rate of \textsc{Sadam} is related to the bounds of the A-LR. Our methods have improved the convergence rate of \textsc{Adam} when comparing self-contained parameters $\epsilon$ and $\beta$. 

\begin{theorem}\textbf{[Nonconvex]}
Suppose $f(x)$ is a nonconvex function that satisfies Assumption \ref{assumption}. Let $\eta_t = \eta = O(\frac{1}{\sqrt{T}})$, \textsc{Sadam} method has 
\begin{align*}
    \min_{t=1,\dots,T}E[\|\nabla f(x_t)\|^2] &\leq O( \frac{\beta^2}{\sqrt{T}}+ \frac{d \beta}{T} + \frac{d \beta^2}{ T\sqrt{T}}).
\end{align*}
\end{theorem}

\begin{remark}
Compared with the rate in Theorem \ref{th1}, the convergence rate of \textsc{Sadam} relies on $\beta$, which can be a much smaller number ($\beta=50$ as recommended) than $\frac{1}{\epsilon}$ (commonly $\epsilon=10^{-8}$ in AGMs), showing that our methods have a better convergence rate than \textsc{Adam}. When $\beta$ is huge, \textsc{Sadam}'s rate is comparable to the classic \textsc{Adam}. When $\beta$ is small, the convergence rate will be  $O(\frac{1}{\sqrt{T}})$ which recovers that of SGD \citep{ghadimi2013stochastic}.
\end{remark}

\begin{corollary}
Treat $\epsilon$ or $\beta$ as a constant, then the \textsc{Adam}, \textsc{Sadam} (and \textsc{SAMSGrad}) methods with fixed $ L, \sigma, G, \beta_1$, and $\eta = O(\frac{1}{\sqrt{T}})$, have complexity of $O(\frac{1}{\sqrt{T}})$, and thus call for $O(\frac{1}{\delta^2})$ iterations to achieve $\delta$-accurate solutions.
\end{corollary}

\begin{theorem}\textbf{[Non-strongly Convex]}
Suppose $f(x)$ is a convex function that satisfies Assumption \ref{assumption}. Assume that $E[\|x_{t}-x^* \|] \leq D$, $\forall t$, and $E[\|x_{m}-x_{n} \|] \leq D_{\infty}$, $\forall$ $m \neq n$, let $\eta_t = \eta = O(\frac{1}{\sqrt{T}})$, \textsc{Sadam} has  $f(\Bar{x}_t)-f^*\leq O(\frac{1}{\sqrt{T}})$, where $\Bar{x}_t = \frac{1}{T}\sum_{t=1}^T x_t$.
\end{theorem}

The accurate convergence rate will be $O(\frac{d}{\epsilon^2\sqrt{T}})$ for \textsc{Adam} and $O(\frac{d\beta^2}{\sqrt{T}})$ for \textsc{Sadam} with fixed $L$, $\sigma$, $G$, $\beta_1$, $D$, $D_{\infty}$. Some works may specify additional sparsity assumptions on stochastic gradients, and in other words, require $\sum_{t=1}^T \sum_{j=1}^d \|g_{t,j}\| \ll \sqrt{dT}$  \citep{duchi2011adaptive,reddi2018convergence,zhou2018convergence,chen2018closing} to reduce the order from $d$ to $\sqrt{d}$. Some works may use the element-wise bounds $\sigma_j$ or $G_j$, and apply $\sum_{j=1}^d \sigma_j= \sigma$, and $\sum_{j=1}^d G_j= G$ to hide $d$. In our work, we do not assume sparsity, so we use $\sigma$ and $G$ throughout the proof. Otherwise, those techniques can also be used to hide $d$ from our convergence rate.

\begin{corollary}
If $\epsilon$ or $\beta$ is treated as constants, then \textsc{Adam}, \textsc{Sadam} (and \textsc{SAMSGrad}) methods with fixed $ L, \sigma, G, \beta_1$, and $\eta = O(\frac{1}{\sqrt{T}})$ in the convex case will call for $O(\frac{1}{\delta^2})$ iterations to achieve $\delta$-accurate solutions.
\end{corollary}

\begin{theorem}\textbf{[P-L Condition]}
Suppose $f(x)$ satisfies the P-L condition (with parameter $\lambda$) and Assumption \ref{assumption} in the convex case. Let $\eta_t = \eta = O(\frac{1}{T^2})$,
\textsc{Sadam} has:
$$\;\;E[f(x_{T+1})-f^*]\leq (1-\frac{2\lambda \mu_3}{T^2})^T E[f(x_1)-f^*]+O(\frac{1}{T}).$$
\end{theorem}


\begin{corollary}\textbf{[Strongly Convex]}
Suppose $f(x)$ is $\mu$-strongly convex function that satisfies Assumption \ref{assumption}.  Let $\eta_t = \eta = O(\frac{1}{T^2})$,
\textsc{Sadam} has the convergence rate:
$$E[f(x_{T+1})-f^*]\leq (1-\frac{2\mu \mu_3}{T^2})^T E[f(x_1)-f^*]+O(\frac{1}{T}).$$
\end{corollary}

In summary, our methods share the same convergence rate as \textsc{Adam}, and enjoy even better convergence speed if comparing the common values chosen for the parameters $\epsilon$ and $\beta$. Our convergence rate recovers that of SGD  and S-Momentum in terms of $T$ for a small $\beta$.

\section{Experiments}
We compare \textsc{Sadam} and \textsc{SAMSGrad} against several state-of-the-art optimizers including S-Momentum, \textsc{Adam}, \textsc{AMSGrad}, \textsc{Yogi}, \textsc{PAdam}, \textsc{PAMSGrad}, \textsc{AdaBound}, and \textsc{AmsBound}. More results and architecture details are in Appendix.

\textbf{Experimental Setup.} We use three datasets for image classifications: MNIST, CIFAR-10 and CIFAR-100 and two datasets for LSTM language models:  Penn Treebank dataset (PTB) and the WikiText-2 (WT2) dataset.  The MNIST dataset is tested on a CNN with 5 hidden layers. The CIFAR-10 dataset is tested on Residual Neural Network with 20 layers (ResNets 20) and 56 layers (ResNets 56) \citep{he2016deep}, and DenseNets with 40 layers \citep{huang2017densely}. The CIFAR-100 dataset is tested on VGGNet \citep{simonyan2014very} and Residual Neural Network with 18 layers (ResNets 18) \citep{he2016deep}. The Penn Treebank dataset (PTB) and the WikiText-2 (WT2) dataset are tested on 3-layer LSTM models \citep{merityRegOpt}. 

We train CNN on the MNIST data for 100 epochs,  ResNets/DenseNets on CIFAR-10 for 300 epochs, with a weight decay factor of $5\times 10^{-4}$ and a batch size of 128, VGGNet/ResNets on CIFAR-100 for 300 epochs, with a weight decay factor of $0.025$ and a batch size of 128 and LSTM language models on 200 epochs. For the CIFAR tasks, we use a fixed multi-stage LR decaying scheme: the B-LR decays by $0.1$ at the $150$-th epoch and $225$-th epoch, which is a popular decaying scheme and used in many works \citep{keskar2017improving,staib2019escaping}. For the language tasks, we use a fixed multi-stage LR decaying scheme: the B-LR decays by $0.1$ at the $100$-th epoch and $150$-th epoch. All algorithms perform grid search for hyper-parameters to choose from $\{10, 1, 0.1, 0.01, 0.001, 0.0001\}$ for B-LR, $\{0.9, 0.99\}$ for $\beta_1$ and $\{0.99, 0.999\}$ for $\beta_2$.  For algorithm-specific hyper-parameters, they are tuned around the recommended values, such as  $p \in \{ \frac{1}{8}, \frac{1}{16}\}$in \textsc{PAdam} and \textsc{PAMSGrad}.
For our algorithms, $\beta$ is selected from  $\{10,50,100\}$ in \textsc{Sadam} and \textsc{SAMSGrad}, though we do observe fine-tuning $\beta$ can achieve better test accuracy most of time. All experiments on CIFAR tasks are repeated for 6 times to obtain the mean and standard deviation for each algorithm.

\textbf{Image Classification Tasks.} 
As a sanity check, experiment on MNIST has been done and its results are in Figure \ref{fig:mnist}, which shows the learning curve for all baseline algorithms and our algorithms on both training and test datasets. As expected, all methods can reach  the zero loss quickly, while for test accuracy, our \textsc{SAMSGrad} shows increase in test accuracy and outperforms competitors within 50 epochs.

\begin{figure}[H]
 \centering
  \includegraphics[width=0.85\linewidth]{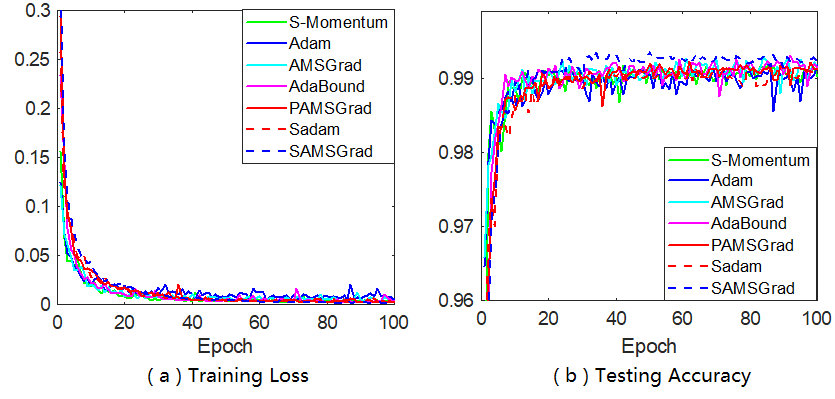}
  \caption{Training loss and test accuracy on MNIST.}
  \label{fig:mnist}
\end{figure}

\begin{table*}[t]
   \centering 
   \caption{Test Accuracy(\%) of CIFAR-10 for ResNets 20, ResNets 56 and DenseNets.} \label{tab:res}
    \begin{tabular}{|p{5cm}|c c c c c c|}
        \hline
         Method & B-LR & $\epsilon$ & $\beta$ & ResNets 20 & ResNets 56 & DenseNets\\
         \hline\hline
         S-Momentum \citep{he2016deep,huang2017densely} & - & - & -  & 91.25 & 93.03 &94.76 \\
         \textsc{Adam} \citep{zaheer2018adaptive} & $10^{-3}$  & $10^{-3}$ & - & $92.56 \pm 0.14$ &  $93.42 \pm 0.16$ & $93.35 \pm 0.21$ \\
        \textsc{Yogi} \citep{zaheer2018adaptive} & $10^{-2}$ & $10^{-3}$  & - & $92.62 \pm 0.17$ & $93.90 \pm 0.21$ & $94.38 \pm 0.26$ \\
         \hline
         S-Momentum & $10^{-1}$ & - &  -  & $92.73 \pm 0.05$ & $94.11\pm 0.15$ &  $95.03 \pm 0.15$\\
         \textsc{Adam} & $10^{-3}$ & $10^{-8}$ &  - & $91.68 \pm 0.12$ &  $92.82 \pm 0.09$ &  $93.32 \pm 0.06$\\
        \textsc{AMSGrad}& $10^{-3}$ &  $10^{-8}$ &  - & $91.7 \pm 0.12$ &  $93.10 \pm 0.11$ &  $93.71 \pm 0.05$\\
        \textsc{PAdam} & $10^{-1}$ & $10^{-8}$ & - & $92.7 \pm 0.10$ & $94.12 \pm 0.12$ & $95.06 \pm 0.06$\\
        \textsc{PAMSGrad} & $10^{-1}$ &  $10^{-8}$ &  - & $92.74 \pm 0.12$ &  $94.18 \pm 0.06$ & $\textbf{95.21} \pm 0.10$ \\
        \textsc{AdaBound} & $10^{-2}$ & $10^{-8}$ & - & $91.59 \pm 0.24$ &  $93.09 \pm 0.14$ &  $94.16 \pm 0.10$ \\
       \textsc{AmsBound} & $10^{-2}$ &  $10^{-8}$ &  - & $91.76 \pm 0.16$ &  $93.08 \pm 0.09$ & $94.03\pm 0.11$ \\
       \hline
       \textsc{Sadam} & $10^{-2}$ &  -  & 50 & $\textbf{93.01} \pm 0.16$ &  $\textbf{94.26} \pm 0.10$& $ 95.19 \pm 0.18$ \\
       \textsc{SAMSGrad} & $10^{-2}$ & - & 50 & $\textbf{92.88} \pm 0.10$ & $\textbf{94.32} \pm 0.18$ & $\textbf{95.31} \pm 0.15$ \\
       \hline
   \end{tabular}
\end{table*}

\begin{figure}[h] 
  \centering
  \includegraphics[width=0.8\linewidth]{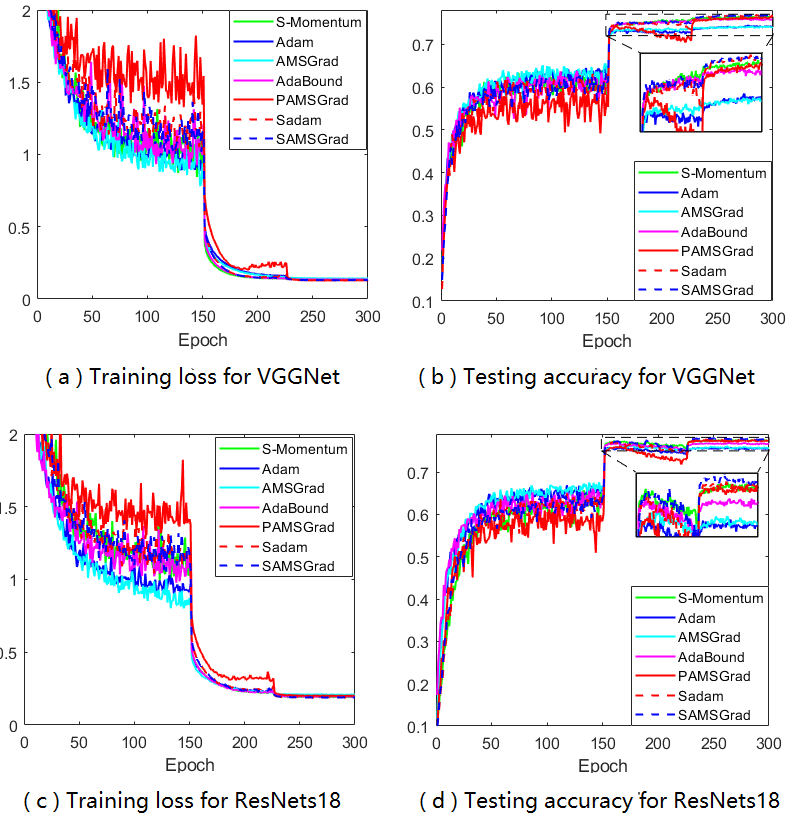}
  \caption{Training loss and test accuracy of two CNN architectures on CIFAR-100. 
  }
  \label{fig:cifar100}
  \vspace{-15pt}
\end{figure}

\begin{figure}[h]
  \centering
  \includegraphics[width=0.8\linewidth]{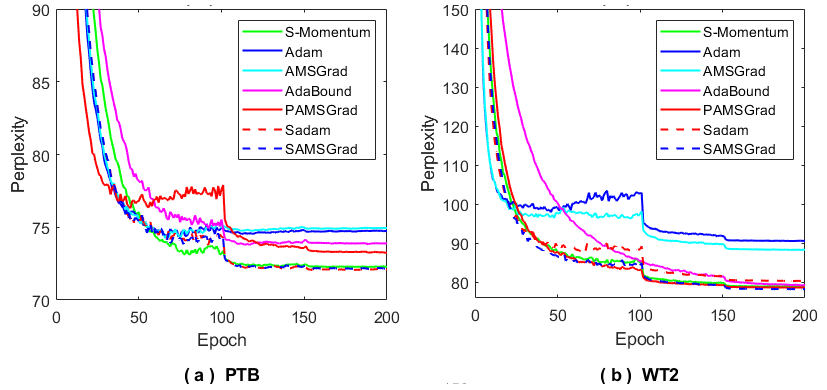}
  \caption{Perplexity curves on the test set on 3-layer LSTM models over PTB and WT2 datasets}
  \label{LSTM_results}
  \vspace{-10pt}
\end{figure}

Using the PyTorch framework, we first run the ResNets 20 model on CIFAR10 and results are shown in Table \ref{tab:res}. The original \textsc{Adam} and \textsc{AMSGrad} have lower test accuracy in comparison with S-Momentum, leaving a clear generalization gap exactly same as what is previously reported.
For our methods, \textsc{Sadam} and \textsc{SAMSGrad} clearly close the gap, and \textsc{Sadam} achieves the best test accuracy among competitors.
We further test all methods with CIFAR10 on ResNets 56 with greater network depth, and the overall performance of each algorithm has been improved. For the experiments with DenseNets, we use a DenseNet with $40$ layers and a growth rate $k = 12$ without bottleneck, channel reduction, or dropout. The results are reported in the last column of Table \ref{tab:res}, \textsc{SAMSGrad} still achieves the best test performance, and the proposed two methods largely improve the performance of \textsc{Adam} and \textsc{AMSGrad} and close the gap with S-Momentum. 

Furthermore, two popular CNN architectures: VGGNet \citep{simonyan2014very} and ResNets18 \citep{he2016deep} are tested on CIFAR-100 dataset to compare different algorithms. Results can be found in Figure \ref{fig:cifar100} and repeated results are in Appendix. Our proposed methods again perform slightly better than S-Momentum in terms of test accuracy.

\textbf{LSTM Language Models.}
Observing the significant improvements in deep neural networks for image classification tasks, we further conduct experiments on the language models with LSTM. For comparing the efficiency of our proposed methods, two LSTM models over the Penn Treebank dataset (PTB) \citep{mikolov2010recurrent} and the WikiText-2
(WT2) dataset \citep{bradbury2016quasi} are tested.
We present the single-model perplexity results for both our proposed methods and other competitive methods in Figure \ref{LSTM_results} and our methods achieve both fast convergence and best generalization performance.

In summary, our proposed methods show great efficacy on several standard benchmarks in both training and testing results, and outperform most optimizers in terms of generalization performance. 

\section{Conclusion}
In this paper, we study adaptive gradient methods from a new perspective that is driven by the observation that the adaptive learning rates are anisotropic at each iteration. Inspired by this observation, we propose to calibrate the adaptive learning rates using an activation function, and in this work, we  examine {\em softplus} function. We combine this calibration scheme with \textsc{Adam} and \textsc{AMSGrad} methods and empirical evaluations show obvious improvement on their generalization performance in multiple deep learning tasks. Using this calibration scheme, we replace the hyper-parameter $\epsilon$ in the original methods by a new parameter $\beta$ in the {\em softplus} function.
A new mathematical model has been proposed to analyze the convergence of adaptive gradient methods. Our analysis shows that the convergence rate is related to $\epsilon$ or $\beta$, which has not been previously revealed, and the dependence on $\epsilon$ or $\beta$ helps us justify the advantage of the proposed methods. 
In the future, the calibration scheme can be designed based on other suitable activation functions, and used in conjunction with any other adaptive gradient method to improve generalization performance. 

\newpage
\bibliographystyle{plain}
\bibliography{ADAMref}

\begin{thebibliography}{30}
\providecommand{\natexlab}[1]{#1}
\providecommand{\url}[1]{\texttt{#1}}
\expandafter\ifx\csname urlstyle\endcsname\relax
  \providecommand{\doi}[1]{doi: #1}\else
  \providecommand{\doi}{doi: \begingroup \urlstyle{rm}\Url}\fi

\bibitem[Bradbury et~al.(2016)Bradbury, Merity, Xiong, and
  Socher]{bradbury2016quasi}
James Bradbury, Stephen Merity, Caiming Xiong, and Richard Socher.
\newblock Quasi-recurrent neural networks.
\newblock \emph{arXiv preprint arXiv:1611.01576}, 2016.

\bibitem[Chaudhari et~al.(2016)Chaudhari, Choromanska, Soatto, LeCun, Baldassi,
  Borgs, Chayes, Sagun, and Zecchina]{chaudhari2016entropy}
Pratik Chaudhari, Anna Choromanska, Stefano Soatto, Yann LeCun, Carlo Baldassi,
  Christian Borgs, Jennifer Chayes, Levent Sagun, and Riccardo Zecchina.
\newblock Entropy-sgd: Biasing gradient descent into wide valleys.
\newblock \emph{arXiv preprint arXiv:1611.01838}, 2016.

\bibitem[Chen and Gu(2018)]{chen2018closing}
Jinghui Chen and Quanquan Gu.
\newblock Closing the generalization gap of adaptive gradient methods in
  training deep neural networks.
\newblock \emph{arXiv preprint arXiv:1806.06763}, 2018.

\bibitem[Chen et~al.(2018)Chen, Liu, Sun, and Hong]{chen2018convergence}
Xiangyi Chen, Sijia Liu, Ruoyu Sun, and Mingyi Hong.
\newblock On the convergence of a class of adam-type algorithms for non-convex
  optimization.
\newblock \emph{arXiv preprint arXiv:1808.02941}, 2018.

\bibitem[De et~al.(2018)De, Mukherjee, and Ullah]{de2018convergence}
Soham De, Anirbit Mukherjee, and Enayat Ullah.
\newblock Convergence guarantees for rmsprop and adam in non-convex
  optimization and an empirical comparison to nesterov acceleration.
\newblock 2018.

\bibitem[Dozat(2016)]{dozat2016incorporating}
Timothy Dozat.
\newblock Incorporating nesterov momentum into adam.
\newblock 2016.

\bibitem[Duchi et~al.(2011)Duchi, Hazan, and Singer]{duchi2011adaptive}
John Duchi, Elad Hazan, and Yoram Singer.
\newblock Adaptive subgradient methods for online learning and stochastic
  optimization.
\newblock \emph{Journal of Machine Learning Research}, 12\penalty0
  (Jul):\penalty0 2121--2159, 2011.

\bibitem[Ghadimi and Lan(2013)]{ghadimi2013stochastic}
Saeed Ghadimi and Guanghui Lan.
\newblock Stochastic first-and zeroth-order methods for nonconvex stochastic
  programming.
\newblock \emph{SIAM Journal on Optimization}, 23\penalty0 (4):\penalty0
  2341--2368, 2013.

\bibitem[He et~al.(2016)He, Zhang, Ren, and Sun]{he2016deep}
Kaiming He, Xiangyu Zhang, Shaoqing Ren, and Jian Sun.
\newblock Deep residual learning for image recognition.
\newblock In \emph{Proceedings of the IEEE conference on computer vision and
  pattern recognition}, pages 770--778, 2016.

\bibitem[Huang et~al.(2017)Huang, Liu, Van Der~Maaten, and
  Weinberger]{huang2017densely}
Gao Huang, Zhuang Liu, Laurens Van Der~Maaten, and Kilian~Q Weinberger.
\newblock Densely connected convolutional networks.
\newblock In \emph{Proceedings of the IEEE conference on computer vision and
  pattern recognition}, pages 4700--4708, 2017.

\bibitem[Keskar and Socher(2017)]{keskar2017improving}
Nitish~Shirish Keskar and Richard Socher.
\newblock Improving generalization performance by switching from adam to sgd.
\newblock \emph{arXiv preprint arXiv:1712.07628}, 2017.

\bibitem[Keskar et~al.(2016)Keskar, Mudigere, Nocedal, Smelyanskiy, and
  Tang]{keskar2016large}
Nitish~Shirish Keskar, Dheevatsa Mudigere, Jorge Nocedal, Mikhail Smelyanskiy,
  and Ping Tak~Peter Tang.
\newblock On large-batch training for deep learning: Generalization gap and
  sharp minima.
\newblock \emph{arXiv preprint arXiv:1609.04836}, 2016.

\bibitem[Kingma and Ba(2014)]{kingma2014adam}
Diederik~P Kingma and Jimmy Ba.
\newblock Adam: A method for stochastic optimization.
\newblock \emph{arXiv preprint arXiv:1412.6980}, 2014.

\bibitem[Kleinberg et~al.(2018)Kleinberg, Li, and
  Yuan]{kleinberg2018alternative}
Robert Kleinberg, Yuanzhi Li, and Yang Yuan.
\newblock An alternative view: When does sgd escape local minima?
\newblock In \emph{International Conference on Machine Learning}, pages
  2703--2712, 2018.

\bibitem[Li et~al.(2018)Li, Xu, Taylor, Studer, and
  Goldstein]{li2018visualizing}
Hao Li, Zheng Xu, Gavin Taylor, Christoph Studer, and Tom Goldstein.
\newblock Visualizing the loss landscape of neural nets.
\newblock In \emph{Advances in Neural Information Processing Systems}, pages
  6389--6399, 2018.

\bibitem[Luo et~al.(2019)Luo, Xiong, Liu, and Sun]{luo2019adaptive}
Liangchen Luo, Yuanhao Xiong, Yan Liu, and Xu~Sun.
\newblock Adaptive gradient methods with dynamic bound of learning rate.
\newblock \emph{arXiv preprint arXiv:1902.09843}, 2019.

\bibitem[Merity et~al.(2017)Merity, Keskar, and Socher]{merityRegOpt}
Stephen Merity, Nitish~Shirish Keskar, and Richard Socher.
\newblock {Regularizing and Optimizing LSTM Language Models}.
\newblock \emph{arXiv preprint arXiv:1708.02182}, 2017.

\bibitem[Mikolov et~al.(2010)Mikolov, Karafi{\'a}t, Burget, {\v{C}}ernock{\`y},
  and Khudanpur]{mikolov2010recurrent}
Tom{\'a}{\v{s}} Mikolov, Martin Karafi{\'a}t, Luk{\'a}{\v{s}} Burget, Jan
  {\v{C}}ernock{\`y}, and Sanjeev Khudanpur.
\newblock Recurrent neural network based language model.
\newblock In \emph{Eleventh annual conference of the international speech
  communication association}, 2010.

\bibitem[Reddi et~al.(2015)Reddi, Hefny, Sra, Poczos, and
  Smola]{reddi2015variance}
Sashank~J Reddi, Ahmed Hefny, Suvrit Sra, Barnabas Poczos, and Alexander~J
  Smola.
\newblock On variance reduction in stochastic gradient descent and its
  asynchronous variants.
\newblock In \emph{Advances in Neural Information Processing Systems}, pages
  2647--2655, 2015.

\bibitem[Reddi et~al.(2018)Reddi, Kale, and Kumar]{reddi2018convergence}
Sashank~J Reddi, Satyen Kale, and Sanjiv Kumar.
\newblock On the convergence of adam and beyond.
\newblock 2018.

\bibitem[Simonyan and Zisserman(2014)]{simonyan2014very}
Karen Simonyan and Andrew Zisserman.
\newblock Very deep convolutional networks for large-scale image recognition.
\newblock \emph{arXiv preprint arXiv:1409.1556}, 2014.

\bibitem[Staib et~al.(2019)Staib, Reddi, Kale, Kumar, and
  Sra]{staib2019escaping}
Matthew Staib, Sashank~J Reddi, Satyen Kale, Sanjiv Kumar, and Suvrit Sra.
\newblock Escaping saddle points with adaptive gradient methods.
\newblock \emph{arXiv preprint arXiv:1901.09149}, 2019.

\bibitem[Tieleman and Hinton(2012)]{tieleman2012lecture}
Tijmen Tieleman and Geoffrey Hinton.
\newblock Lecture 6.5-rmsprop: Divide the gradient by a running average of its
  recent magnitude.
\newblock \emph{COURSERA: Neural networks for machine learning}, 4\penalty0
  (2):\penalty0 26--31, 2012.

\bibitem[Wilson et~al.(2016)Wilson, Recht, and Jordan]{wilson2016lyapunov}
Ashia~C Wilson, Benjamin Recht, and Michael~I Jordan.
\newblock A lyapunov analysis of momentum methods in optimization.
\newblock \emph{arXiv preprint arXiv:1611.02635}, 2016.

\bibitem[Wright and Nocedal(1999)]{wright1999numerical}
Stephen~J Wright and Jorge Nocedal.
\newblock Numerical optimization.
\newblock \emph{Springer Science}, 35\penalty0 (67-68):\penalty0 7, 1999.

\bibitem[Yang et~al.(2016)Yang, Lin, and Li]{yang2016unified}
Tianbao Yang, Qihang Lin, and Zhe Li.
\newblock Unified convergence analysis of stochastic momentum methods for
  convex and non-convex optimization.
\newblock \emph{arXiv preprint arXiv:1604.03257}, 2016.

\bibitem[Zagoruyko and Komodakis(2016)]{zagoruyko2016wide}
Sergey Zagoruyko and Nikos Komodakis.
\newblock Wide residual networks.
\newblock \emph{arXiv preprint arXiv:1605.07146}, 2016.

\bibitem[Zaheer et~al.(2018)Zaheer, Reddi, Sachan, Kale, and
  Kumar]{zaheer2018adaptive}
Manzil Zaheer, Sashank Reddi, Devendra Sachan, Satyen Kale, and Sanjiv Kumar.
\newblock Adaptive methods for nonconvex optimization.
\newblock In \emph{Advances in Neural Information Processing Systems}, pages
  9815--9825, 2018.

\bibitem[Zeiler(2012)]{zeiler2012adadelta}
Matthew~D Zeiler.
\newblock Adadelta: an adaptive learning rate method.
\newblock \emph{arXiv preprint arXiv:1212.5701}, 2012.

\bibitem[Zhou et~al.(2018)Zhou, Tang, Yang, Cao, and Gu]{zhou2018convergence}
Dongruo Zhou, Yiqi Tang, Ziyan Yang, Yuan Cao, and Quanquan Gu.
\newblock On the convergence of adaptive gradient methods for nonconvex
  optimization.
\newblock \emph{arXiv preprint arXiv:1808.05671}, 2018.

\end{thebibliography}

\newpage

\section*{Appendix}

\section{Architecture Used in Our Experiments}

Here we mainly introduce the MNIST architecture with Pytorch used in our empirical study, ResNets and DenseNets are well-known architectures used in many works and we do not include details here.

\begin{center}
 \begin{tabular}{||c | c ||} 
 \hline
 layer & layer setting \\ [0.5ex] 
 \hline\hline
 F.relu(self.conv1(x)) & self.conv1 = nn.Conv2d(1, 6, 5)\\ 
 \hline
 F.max\_pool2d(x, 2, 2) & \\
 \hline
 F.relu(self.conv2(x))  &self.conv2 = nn.Conv2d(6, 16, 5) \\
 \hline
 x.view(-1, 16*4)  & \\
 \hline
 F.relu(self.fc1(x)) & self.fc1 = nn.Linear(16*4*4, 120)\\ 
 \hline
 x= F.relu(self.fc2(x)) &self.fc2 = nn.Linear(120, 84) \\ 
 \hline
  x = self.fc3(x)& self.fc3 = nn.Linear(84, 10)\\
  \hline
  F.log\_softmax(x, dim=1) & \\
 \hline
\end{tabular}
\end{center}

\section{More Empirical Results}

In this section, we perform multiply experiments to study the property of anisotropic A-LR exsinting in AGMs and the performance of {\em softplus} function working on A-LR. We first show the A-LR range of popular \textsc{Adam}-type methods, then present how the parameter $\beta$ in \textsc{Sadam} and \textsc{SAMSGrad} reduce the range of A-LR and improve both training and testing performance.

\subsection{ A-LR Range of AGMs}

Besides the A-LR range of \textsc{Adam} method, which has shown in main paper, we further want to study more other \textsc{Adam}-type methods, and do experiments focus on \textsc{AMSGrad}, \textsc{PAdam}, and \textsc{PAMSGrad} on different tasks (Figure \ref{fig:mnist_range}, \ref{fig:resnet20_range}, and \ref{fig:desenet_range}). \textsc{AMSGrad} also has extreme large-valued coordinates, and will encounter the ``small learning rate dilemma" as well as \textsc{Adam}. With partial parameter $p$, the value range of A-LR can be largely narrow down, and the maximum range will be reduced around $10^2$ with \textsc{PAdam}, and less than $10^2$ with \textsc{PAMSGrad}. This reduced range, avoiding the ``small learning rate dilemma", may help us understand what ``trick" works on \textsc{Adam}'s A-LR can indeed improve the generalization performance. Besides, the range of A-LR in \textsc{Yogi}, \textsc{AdaBound} and \textsc{AmsBound} will be reduced or controlled by specific $\epsilon$ or $clip$ function, we don't show more information here.

\begin{figure}[H]
 \centering
  \includegraphics[width=0.85\linewidth]{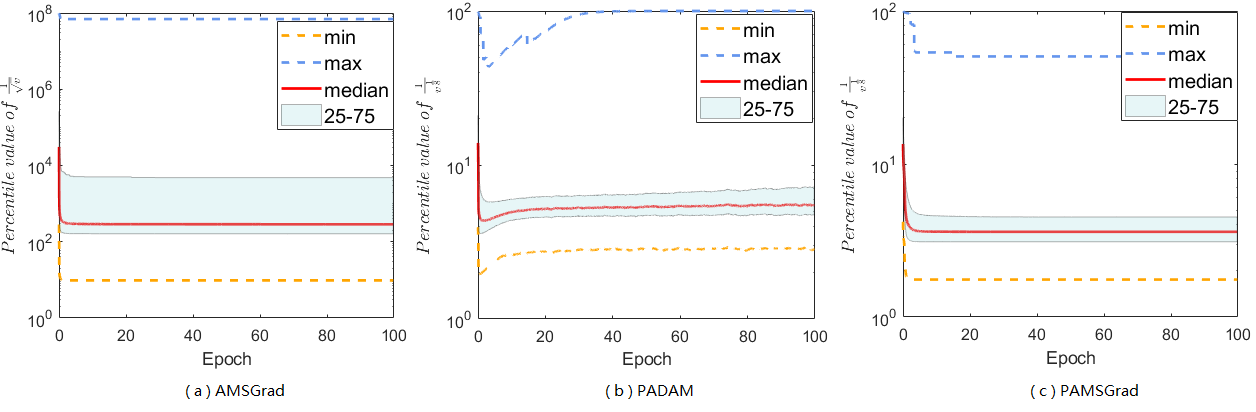}
  \caption{A-LR range of \textsc{AMSGrad} (a), \textsc{PAdam} (b), and \textsc{PAMSGrad} (c) on MNIST.}
  \label{fig:mnist_range}
\end{figure}

\begin{figure}[H]
 \centering
  \includegraphics[width=0.85\linewidth]{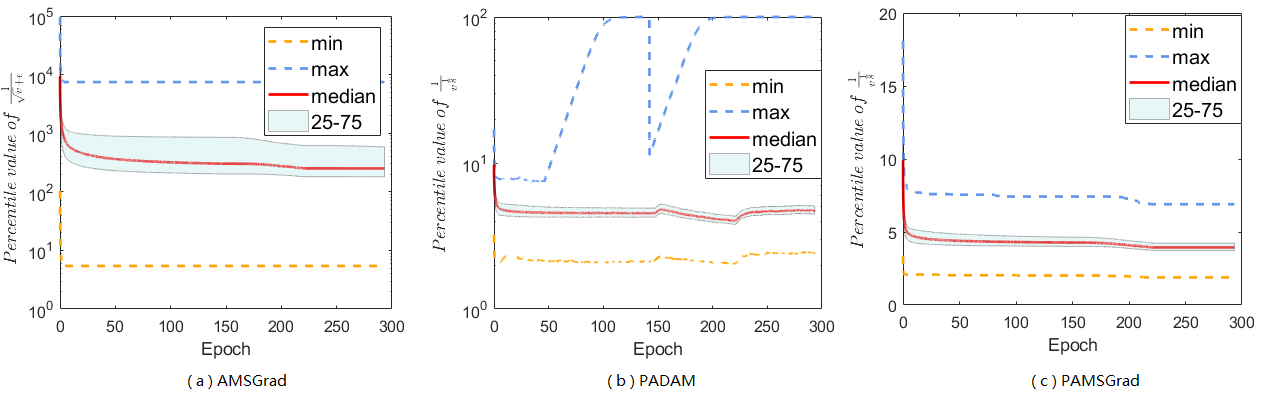}
  \caption{A-LR range of \textsc{AMSGrad} (a), \textsc{PAdam} (b), and \textsc{PAMSGrad} (c) on ResNets 20.}
  \label{fig:resnet20_range}
\end{figure}

\begin{figure}[H]
 \centering
  \includegraphics[width=0.85\linewidth]{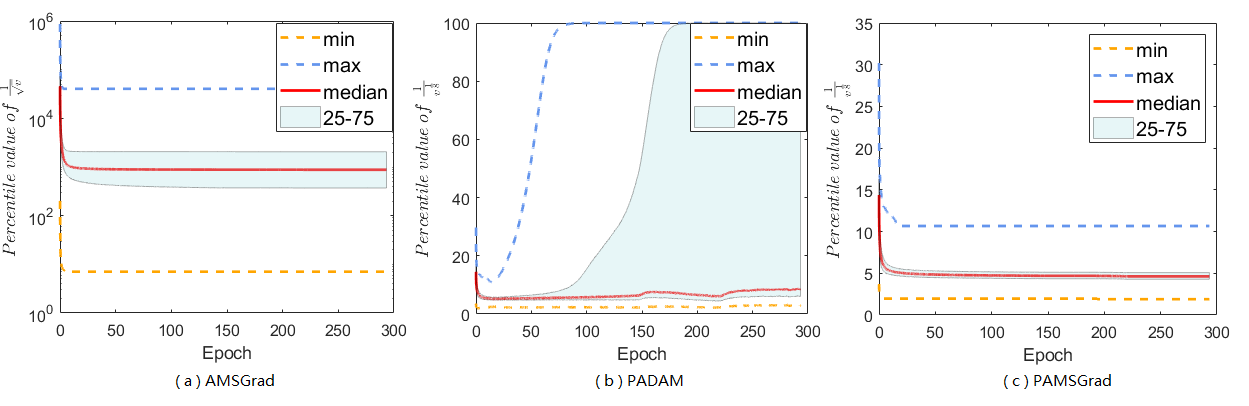}
  \caption{A-LR range of \textsc{AMSGrad} (a), \textsc{PAdam} (b), and \textsc{PAMSGrad} (c) on DenseNets.}
  \label{fig:desenet_range}
\end{figure}

\subsection{Parameter $\beta$ Reduces the Range of A-LR}

The main paper has discussed about {\em softplus} function, and mentions that it does help to constrain large-valued coordinates in A-LR while keep others untouched, here we give more empirical support. No matter how does $\beta$ set, the modified A-LR will have a reduced range. By setting various $\beta$'s, we can find an appropriate $\beta$ that performs the best for specific tasks on datasets. Besides the results of A-LR range of \textsc{Sadam} on MNIST with different choices of $\beta$, we also study \textsc{Sadam} and \textsc{SAMSGrad} on ResNets 20 and DenseNets. 

\begin{figure}[H]
 \centering
  \includegraphics[width=0.9\linewidth]{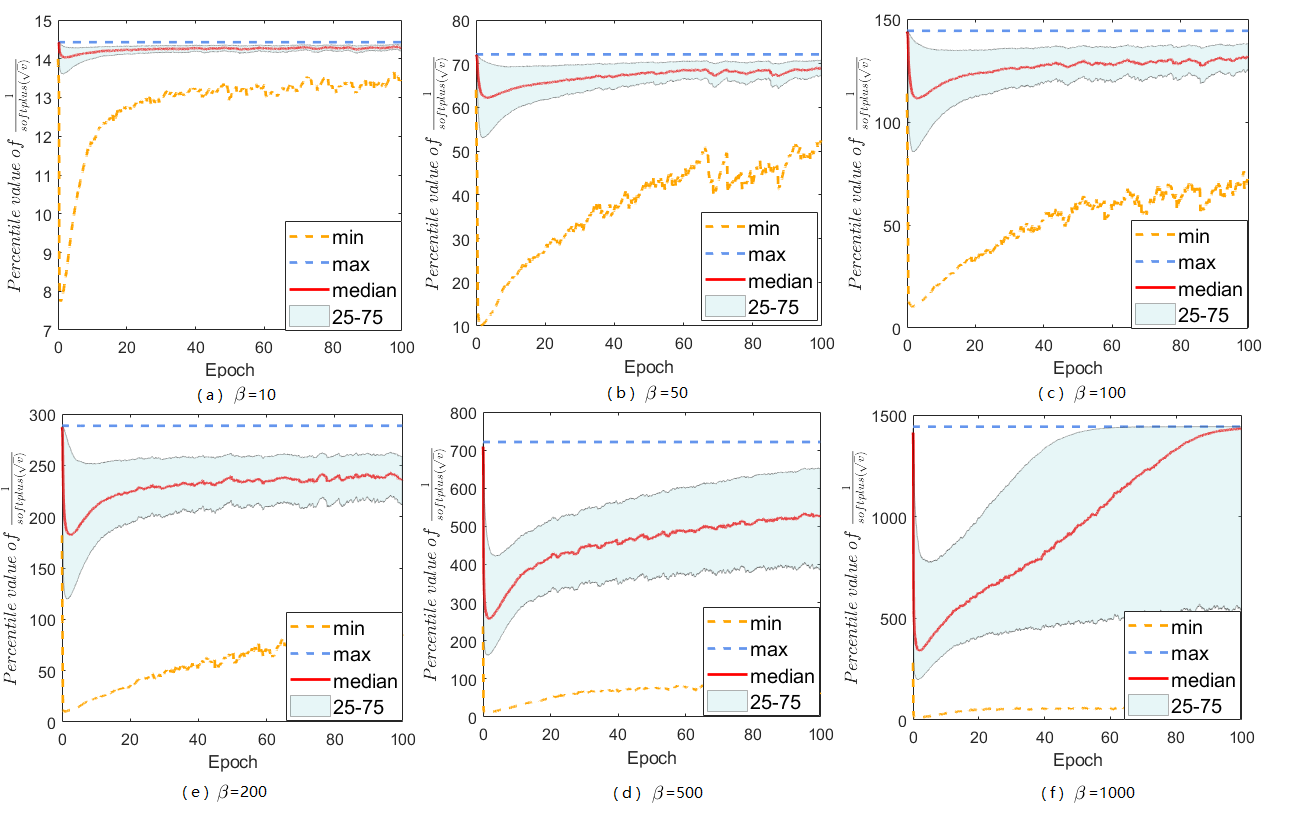}
  \caption{The range of A-LR: $1/softplus( \sqrt{v_t})$ over iterations for \textsc{Sadam} on MNIST with different choices of $\beta$. The maximum ranges in all figures are compressed to a reasonable smaller value compared with $10^{8}$. }
  \label{fig:mnist_Sadam}
\end{figure}

\begin{figure}[H]
 \centering
  \includegraphics[width=0.9\linewidth]{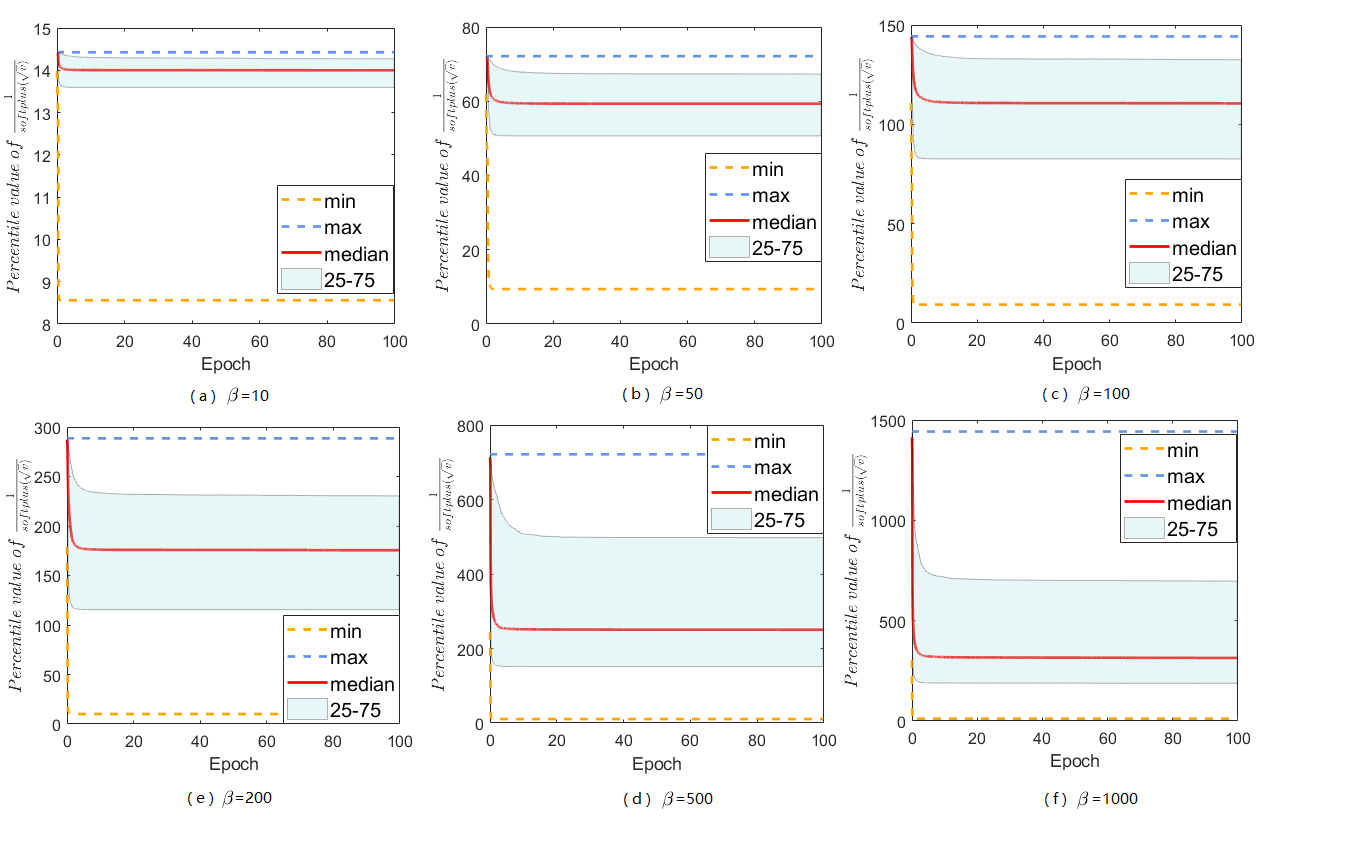}
  \caption{The range of A-LR: $1/softplus( \sqrt{v_t})$, $v_t = max \{v_{t-1}, \tilde{v}_{t}\}$ over iterations for \textsc{SAMSGrad} on MNIST with different choice of $\beta$.  The maximum ranges in all figures are compressed to a reasonable smaller value compared with those of  \textsc{AMSGrad} on MNIST.}
  \label{fig:mnist_Samsgrag}
\end{figure}

\begin{figure}[H]
 \centering
  \includegraphics[width=0.9\linewidth]{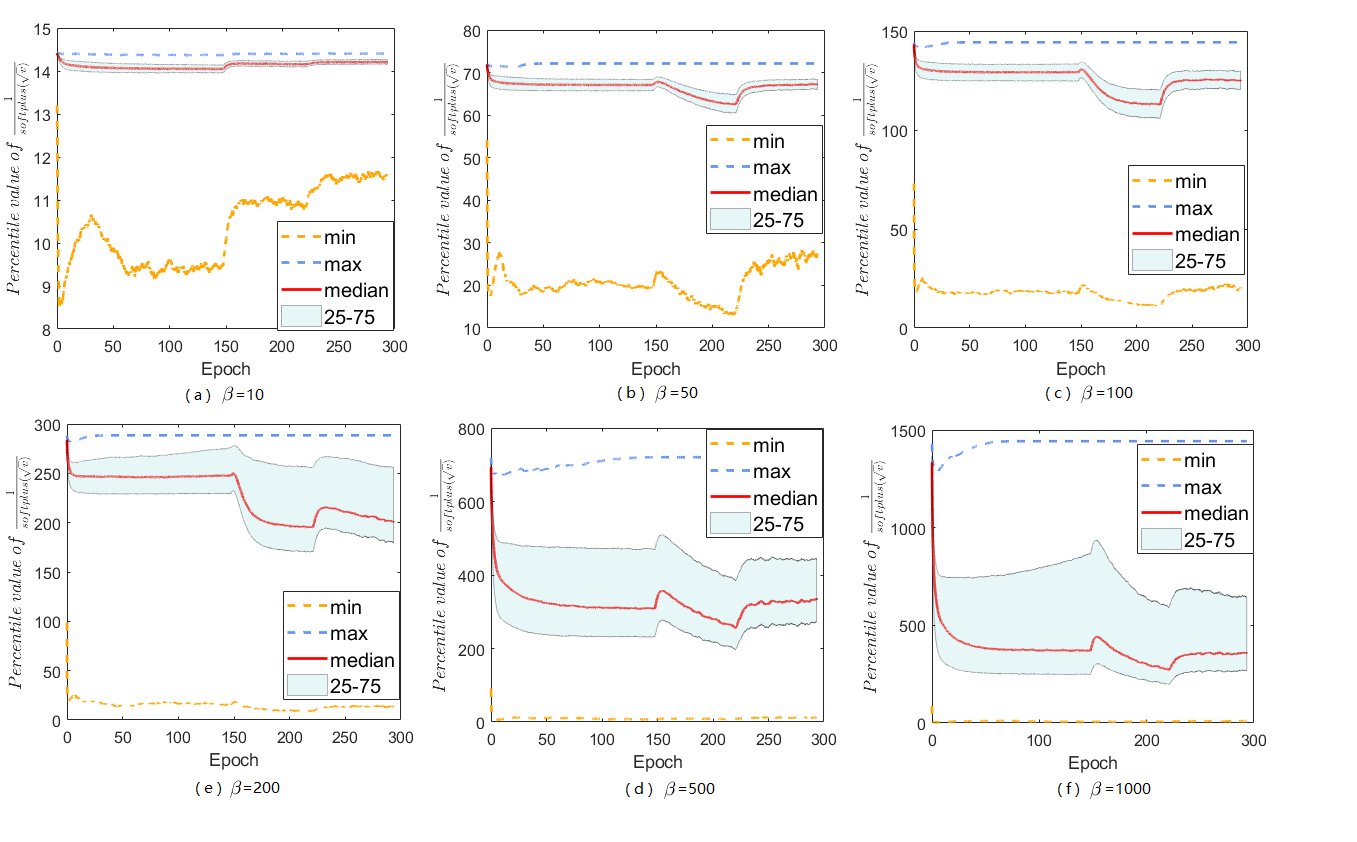}
  \caption{The range of A-LR: $1/softplus( \sqrt{v_t})$ over iterations for \textsc{Sadam} on ResNets 20 with different choices of $\beta$.}
  \label{fig:resnet20_Sadam}
\end{figure}

\begin{figure}[H]
 \centering
  \includegraphics[width=0.9\linewidth]{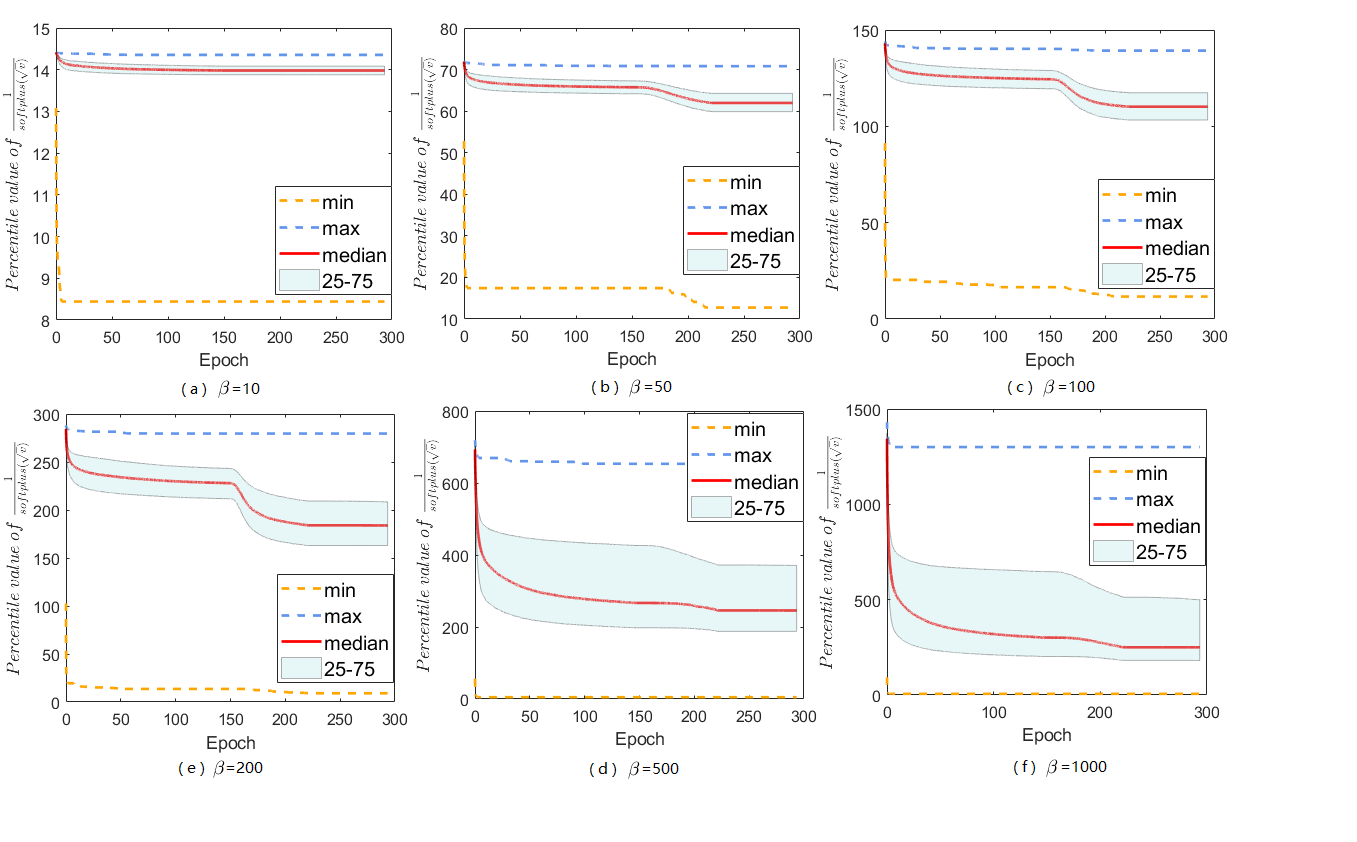}
  \caption{The range of A-LR: $1/softplus( \sqrt{ v_t})$, $v_t = max \{ v_{t-1}, \tilde{v}_{t}\}$ over iterations for \textsc{SAMSGrad} on ResNets 20 with different choices of $\beta$.}
  \label{fig:resnet20_Samsgrag}
\end{figure}

\begin{figure}[H]
 \centering
  \includegraphics[width=0.9\linewidth]{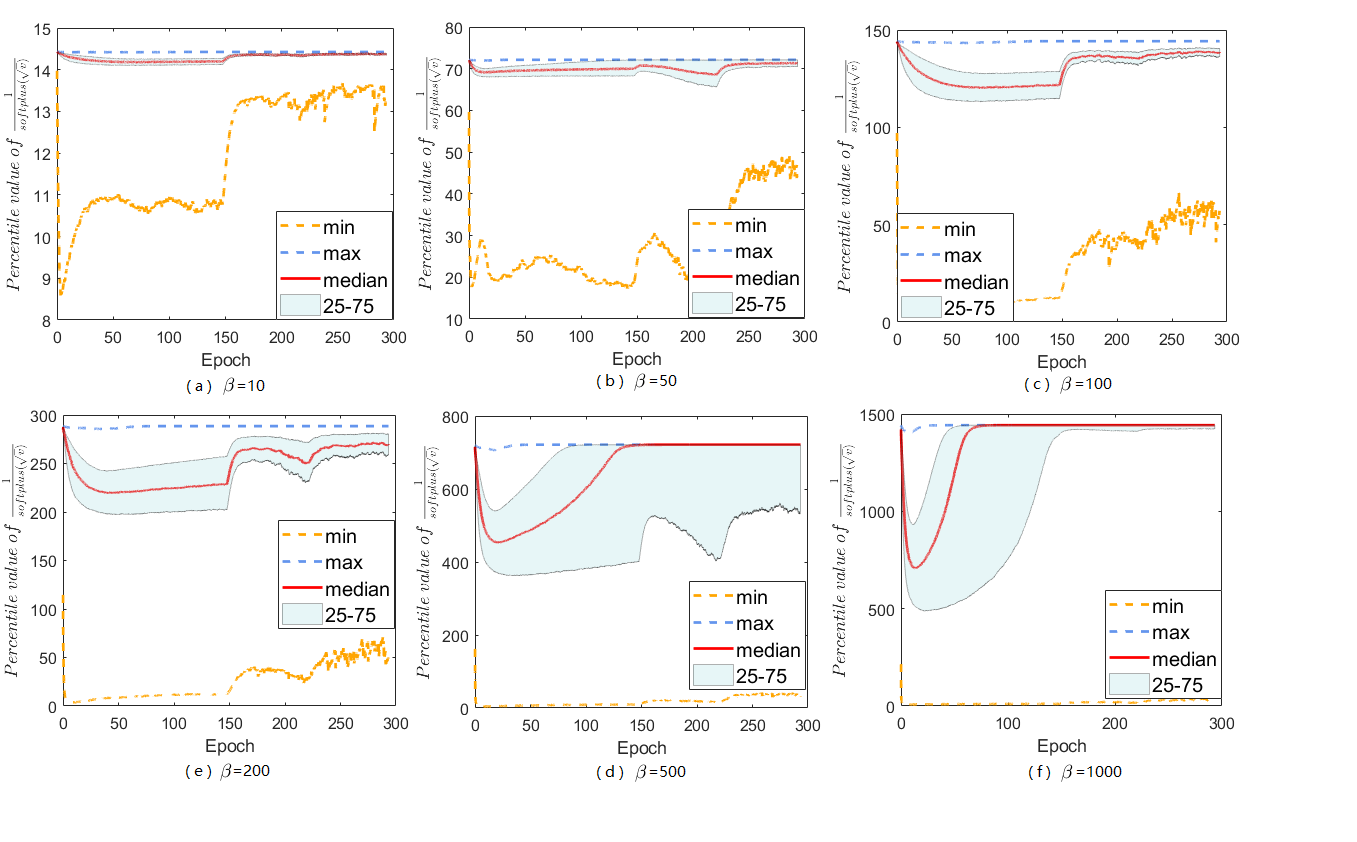}
  \caption{The range of A-LR: $1/softplus( \sqrt{v_t})$ over iterations for \textsc{Sadam} on DenseNets with different choice of $\beta$.}
  \label{fig:dens_Sadam}
\end{figure}

\begin{figure}[H]
 \centering
  \includegraphics[width=0.9\linewidth]{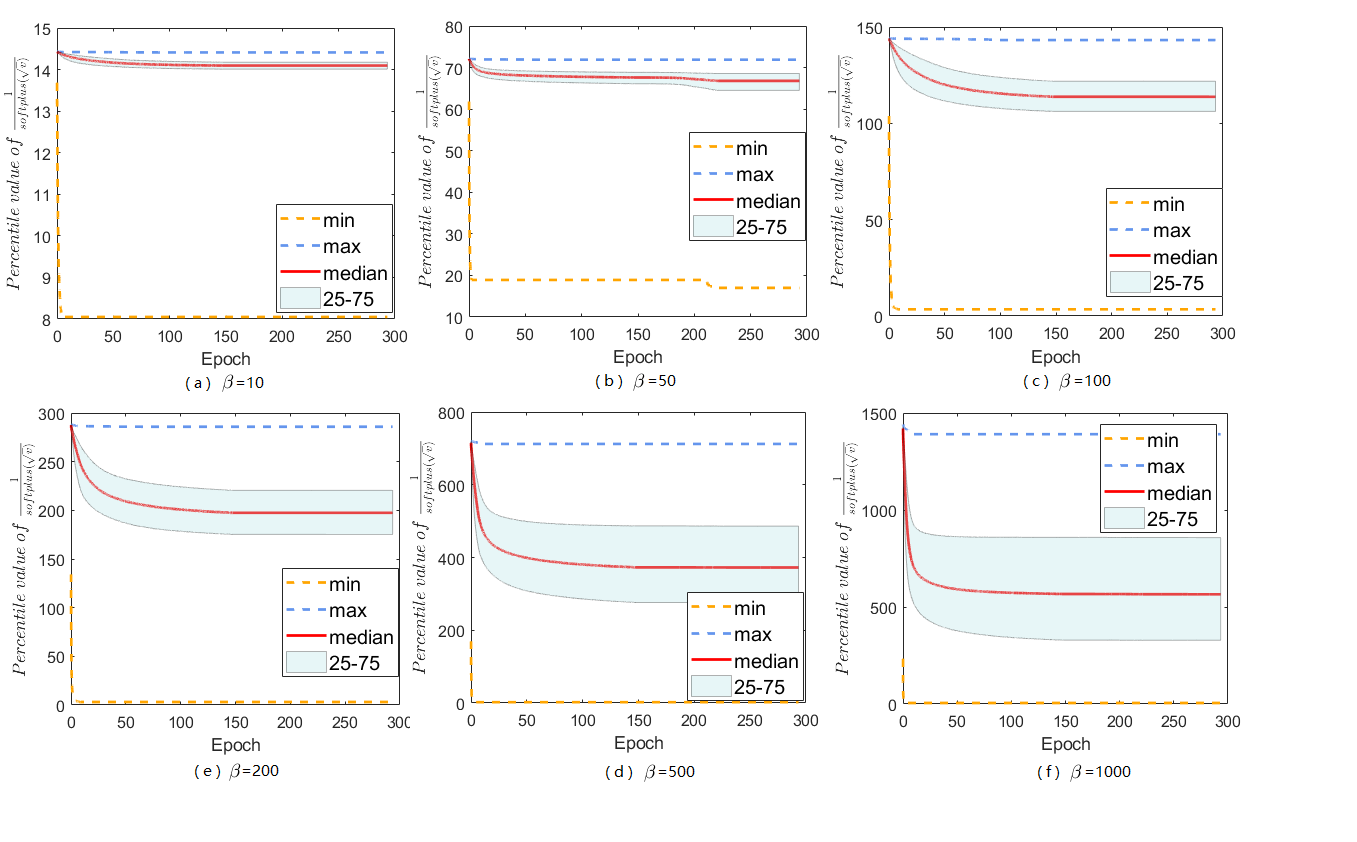}
  \caption{The range of A-LR: $1/softplus( \sqrt{ v_t})$, $v_t = max \{ v_{t-1}, \tilde{v}_{t}\}$ over iterations for \textsc{SAMSGrad} on DenseNets with different choices of $\beta$.}
  \label{fig:dens_Samsgrag}
\end{figure}

Here we do grid search to choose appropriate $\beta$ from $\{10, 50, 100, 200, 500, 1000\}$.
In summary, with {\em softplus} fuction, \textsc{Sadam} and \textsc{SAMSGrad} will narrow down the range of A-LR, make the A-LR vector more regular, avoiding "small learning rate dilemma" and finally achieve better performance.

\subsection{Parameter $\beta$ Matters in Both Training and Testing}

After studying existing \textsc{Adam}-type methods, and effect of different $\beta$ in adjusting A-LR, we focus on the training and testing accuracy of our {\em softplus} framework, especially \textsc{Sadam} and \textsc{SAMSGrad}, with different choices of $\beta$. 

\begin{figure}[H]
 \centering
  \includegraphics[width=0.8\linewidth]{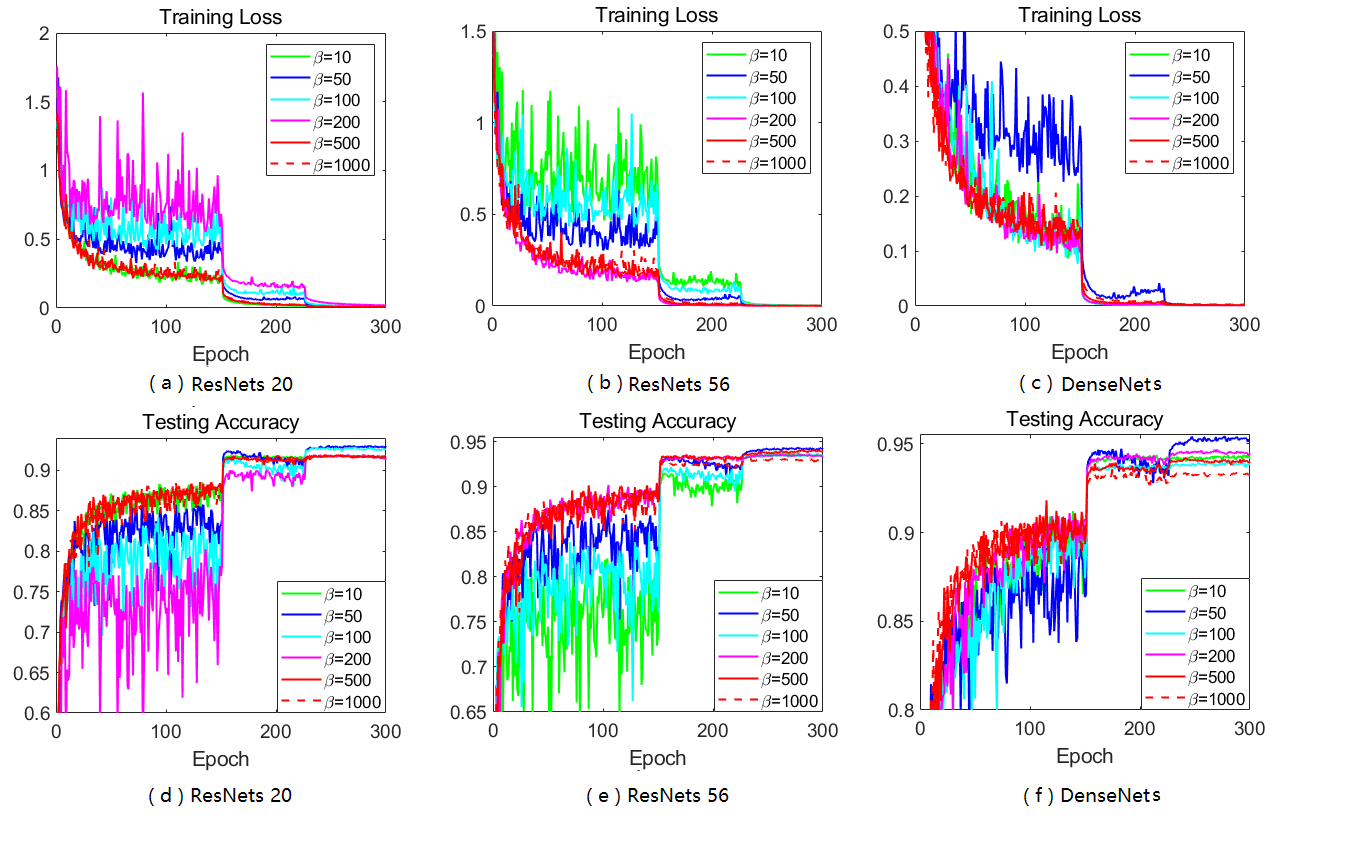}
  \caption{Performance of \textsc{Sadam} on CIFAR-10 with different choice of $\beta$.}
  \label{fig:cifar_Sadam}
\end{figure}

\begin{figure}[H]
 \centering
  \includegraphics[width=0.8\linewidth]{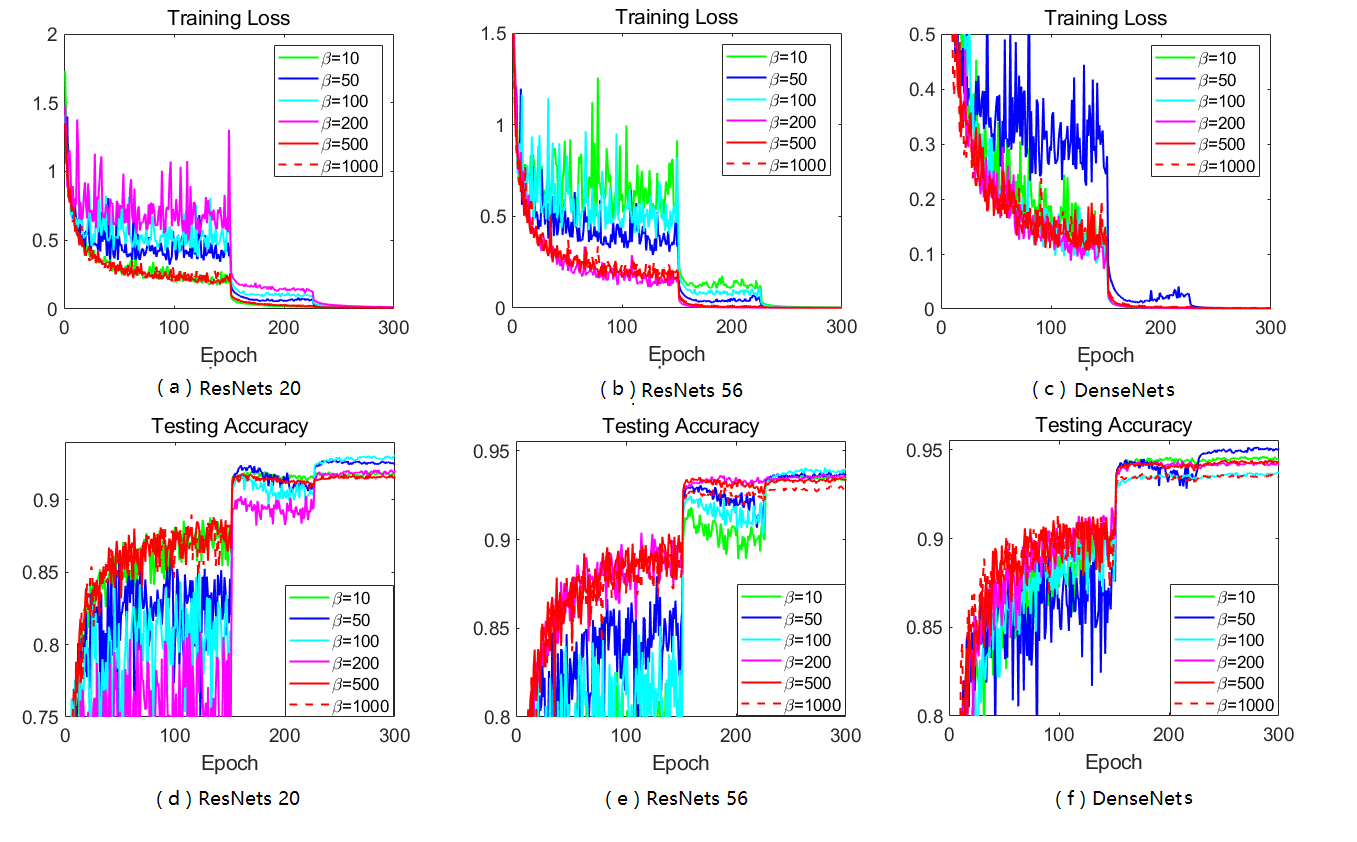}
  \caption{Performance of \textsc{SAMSGrad} on CIFAR-10 with different choice of $\beta$.}
  \label{fig:cifar_Samsgrag}
\end{figure}

\section{CIFAR100}

Two popular CNN architectures are tested on CIFAR-100 dataset to compare different algorithms: VGGNet \citep{simonyan2014very} and ResNets18  \citep{he2016deep}. Besides the figures in main text, we have repeated experiments and show results as follows. Our proposed methods again perform slightly better than S-Momentum in terms of test accuracy.

\begin{table}[H]
    \centering 
    \caption{Test Accuracy(\%) of CIFAR100 for VGGNet.} \label{tab:cifarvgg}
    \begin{tabular}{|c| c c c c|}
        \hline
         Method &50th epoch & 150th epoch & 250th epoch & best perfomance\\
         \hline
         S-Momentum  & $59.09 \pm 2.09$ & $61.25\pm 1.51$ &  $76.14 \pm 0.12$ & $76.43 \pm 0.15$\\
         \textsc{Adam} & $60.21 \pm 0.81$ &  $62.98 \pm 0.10$ &  $73.81 \pm 0.17$& $ 74.18 \pm 0.15$\\
         \textsc{AMSGrad} & $61.00 \pm 1.17$ &  $63.27 \pm 1.18$ &  $74.04 \pm 0.16$& $ 74.26 \pm 0.18$\\
         \textsc{PAdam}  & $53.62 \pm 1.70$ & $56.02 \pm 0.86$ & $75.85 \pm 0.20$& $ 76.36\pm 0.16$\\
         \textsc{PAMSGrad}  & $52.49 \pm 3.07$ &  $57.39 \pm 1.40$ & $75.82 \pm 0.31$& $76.26 \pm 0.30$ \\
        \textsc{AdaBound}  & $60.27 \pm 0.99$ &  $60.36 \pm 1.71$ &  $75.86 \pm 0.23$&  $76.10 \pm 0.22$ \\
        \textsc{AmsBound}  & $59.88 \pm 0.56$ &  $60.11 \pm 1.92$ & $75.74\pm 0.23$&  $75.99 \pm 0.20$ \\
        \hline
        \textsc{Sadam}  & ${58.59} \pm 1.60$ &  $61.27 \pm 1.67$& $ 76.35 \pm 0.18$ & $ \textbf{76.64} \pm 0.18$ \\
        \textsc{SAMSgrad}  & $59.16 \pm 1.20$ & $60.86 \pm 0.39$ & $76.27 \pm 0.23$& $ 76.47 \pm 0.26$ \\
        \hline
    \end{tabular}
\end{table}

\begin{table}[H]
    \centering 
    \caption{Test Accuracy(\%) of CIFAR100 for ResNets18.} \label{tab:cifarres}
    \begin{tabular}{|c| c c c c|}
        \hline
         Method &50th epoch & 150th epoch & 250th epoch & best perfomance\\
         \hline
         S-Momentum  & $59.98 \pm 1.31$ & $63.32\pm 1.61$ &  $77.19 \pm 0.36$ & $77.50 \pm 0.25$\\
         \textsc{Adam} & $63.40 \pm 1.42$ &  $66.18 \pm 1.02$ &  $75.68 \pm 0.49$& $ 76.14 \pm 0.24$\\
         \textsc{AMSGrad} & $63.16 \pm 0.47$ &  $66.59 \pm 1.42$ &  $75.92 \pm 0.26$& $ 76.32 \pm 0.11$\\
         \textsc{PAdam}  & $56.28 \pm 0.87$ & $58.71 \pm 1.66$ & $77.18 \pm 0.21$& $ 77.51\pm 0.19$\\
         \textsc{PAMSGrad}  & $54.34 \pm 2.21$ &  $58.81 \pm 1.95$ & $77.41 \pm 0.17$& $77.67 \pm 0.14$ \\
        \textsc{AdaBound}  & $61.13 \pm 0.84$ &  $64.30 \pm 1.84$ &  $77.18 \pm 0.38$&  $77.50\pm 0.29$ \\
        \textsc{AmsBound}  & $61.05 \pm1.59$ &  $62.04 \pm 2.10$ & $77.08\pm 0.19$&  $77.34 \pm 0.13$ \\
        \hline
        \textsc{Sadam}  & ${59.00} \pm 1.09$ &  $62.75 \pm 1.03$& $ 77.26 \pm 0.30$ & $ {77.61} \pm 0.19$ \\
        \textsc{SAMSgrad}  & $59.63 \pm 1.27$ & $63.44 \pm 1.84$ & $77.31 \pm 0.40$& $ \textbf{77.70} \pm 0.31$ \\
        \hline
    \end{tabular}
\end{table}

\section{Theoretical Analysis Details}
We analyze the convergence rate of \textsc{Adam} and \textsc{Sadam} under different cases, and derive competitive results of our methods. The following table gives an overview of stochastic gradient methods convergence rate under various conditions, in our work we provide a different way of proof compared with previous works and also associate the analysis with hyperparameters of \textsc{Adam} methods.

\subsection{Prepared Lemmas}
We have a series of prepared lemmas to help with optimization convergence rate analysis, and some of them maybe also used in generalization error bound analysis.

\begin{lemma}\label{lemma:elementwise}
For any vectors $a, b, c \in \mathbb{R}^d$, $<a, b\odot c> = <a \odot b, c> = <a\odot \sqrt{b}, c\odot \sqrt{b}>$, here $\odot$ is element-wise product,$\sqrt{b}$ is element-wise square root.
\end{lemma}

\begin{proof}
\begin{align*}
   <a, b\odot c>  &= <\begin{pmatrix}
a_1\\
\vdots\\
a_d\\
\end{pmatrix},\begin{pmatrix}
b_1 c_1\\
\vdots\\
b_d c_d\\
\end{pmatrix}>
=a_1b_1c_1+ \dots + a_d b_d c_d\\
 <a\odot b,  c>  &= <\begin{pmatrix}
a_1b_1\\
\vdots\\
a_d b_d\\
\end{pmatrix},\begin{pmatrix}
c_1\\
\vdots\\
c_d\\
\end{pmatrix}>
=a_1b_1c_1+ \dots + a_d b_d c_d\\
 <a\odot \sqrt{b},  c\odot \sqrt{b}>  &= <\begin{pmatrix}
a_1 \sqrt{b_1}\\
\vdots\\
a_d \sqrt{b_d}\\
\end{pmatrix},\begin{pmatrix}
\sqrt{b_1} c_1\\
\vdots\\
\sqrt{b_d} c_d\\
\end{pmatrix}>
=a_1b_1c_1+ \dots + a_d b_d c_d
\end{align*}
\end{proof}

\begin{lemma}\label{infi_norm}
	For any vector a, we have
	\begin{equation}
	\|a^2\|_{\infty} \leq \|a\|^2.
	\end{equation}
\end{lemma}

\begin{lemma}\label{g_t}
	For unbiased stochastic gradient, we have
	\begin{equation}
	E[\|g_t\|^2]\leq \sigma^2+G^2.
	\end{equation}
\end{lemma}

\begin{proof}
	From gradient bounded assumption and variance bounded assumption,
	\begin{align*}
	E[\|g_t\|^2]& =E[\|g_t - \nabla f(x_t) +\nabla f(x_t)\|^2]\\
	&= E[\|g_t - \nabla f(x_t) \|^2] +\|\nabla f(x_t)\|^2\\
	&\leq \sigma^2+G^2.
	\end{align*}
\end{proof}

\begin{lemma}\label{m_t}
All momentum-based optimizers using first momentum $m_t=\beta_1 m_{t-1}+(1-\beta_1)g_t$ will satisfy 
\begin{equation}
    E[\|m_t\|^2]\leq \sigma^2+G^2.
\end{equation}
\end{lemma}

\begin{proof}
From the updating rule of first momentum estimator, we can derive 
\begin{equation}
    m_t = \Sigma_{i=1}^t(1-\beta_1)\beta_1^{t-i}g_i.
\end{equation}

Let $\Gamma_t=\Sigma_{i=1}^t\beta_1^{t-i}= \frac{1-\beta_1^t}{1-\beta_1}$, by Jensen inequality and Lemma \ref{g_t} ,
\begin{align*}
    E[\|m_t\|^2] &= E[\|\Sigma_{i=1}^t(1-\beta_1)\beta_1^{t-i}g_i\|^2]
    = \Gamma_t^2 E[\|\Sigma_{i=1}^t\frac{(1-\beta_1)\beta_1^{t-i}}{\Gamma_t}g_i\|^2] \\
    &\leq \Gamma_t^2 \Sigma_{i=1}^t(1-\beta_1)^2\frac{\beta_1^{t-i}}{\Gamma_t}E[\|g_i\|^2] \leq \Gamma_t(1-\beta_1)^2\Sigma_{i=1}^t\beta_1^{t-i} ( \sigma^2+G^2)\\
    &\leq  \sigma^2+G^2.
\end{align*}
\end{proof}

\begin{lemma}\label{v_t}
Each coordinate of vector $v_t=\beta_2 v_{t-1}+(1-\beta_2)g_t^2$ will satisfy 
\begin{align*}
    E[v_{t,j}]\leq  \sigma^2+G^2,
\end{align*}
where $j \in [1,d]$ is the coordinate index.
\end{lemma}

\begin{proof}
	From the updating rule of second momentum estimator, we can derive 
	\begin{equation}
	v_{t,j} = \Sigma_{i=1}^t(1-\beta_2)\beta_2^{t-i}g_{i,j}^2 \geq 0.
	\end{equation}
	Since the decay parameter $\beta_2 \in [0,1)$, $\Sigma_{i=1}^t(1-\beta_2)\beta_2^{t-i}=1-\beta_2^t \leq 1$. From Lemma \ref{g_t},
	\begin{align*}
	E[v_{t,j}]=E[\Sigma_{i=1}^t(1-\beta_2)\beta_2^{t-i}g_{i,j}^2]
	\leq \Sigma_{i=1}^t(1-\beta_2)\beta_2^{t-i}( \sigma^2+G^2)
	\leq  \sigma^2+G^2.
	\end{align*}
\end{proof}


And we can derive the following important lemma:

\begin{lemma}\label{bound1}\textbf{[Bounded A-LR]}
For any $t\geq 1$, $j \in [1, d]$, $\beta_2\in [0,1]$, and fixed $\epsilon$ in \textsc{Adam} and $\beta$ defined in softplus function in \textsc{Sadam},  the following bounds always hold:

\textsc{Adam} has $(\mu_1,\mu_2)-$ bounded A-LR:
\begin{equation} \label{eq:sqrtv}
    \mu_1 \leq \frac{1}{\sqrt{v_{t,j}}+\epsilon} \leq \mu_2;
\end{equation}

\textsc{Sadam} has $(\mu_3,\mu_4)-$ bounded A-LR:
\begin{equation} \label{eq:softv}
    \mu_3 \leq \frac{1}{softplus(\sqrt{v_{t,j}} )}\leq \mu_4;
\end{equation}
where $ 0 <\mu_1\leq \mu_2$, $ 0 <\mu_3\leq \mu_4$. 
For brevity, we use $\mu_l, \mu_u$ denoting the lower bound and upper bound respectively, and both \textsc{Adam} and \textsc{Sadam} will be analysis with the help of $(\mu_l, \mu_u)$.
\end{lemma}

\begin{proof}
For \textsc{Adam}, let $\mu_1 =\frac{1}{\sqrt{\sigma^2+G^2}+\epsilon}$, $\mu_2 = \frac{1}{\epsilon}$, then we can get the result in (\ref{eq:sqrtv}).

For \textsc{Sadam}, notice that $softplus(\cdot)$ is a monotone increasing function, and $\sqrt{v_{t,j}}$ is both upper-bounded and lower-bounded, then we have (\ref{eq:softv}), where $\mu_3 =\frac{1}{ \frac{1}{\beta}\log(1+e^{\beta\cdot \sqrt{\sigma^2+G^2}})}$, $\mu_4 =\frac{1}{ \frac{1}{\beta}\log(1+e^{\beta\cdot 0}) }=\frac{\beta}{\log2}$.
\end{proof}

\begin{lemma}\label{z_t}
Define $z_t = x_t + \frac{\beta_1}{1-\beta_1}(x_t - x_{t-1}), \forall t \geq 1$ $\beta_1 \in [0,1)$. Let $\eta_t = \eta$, then the following updating formulas hold:

Gradient-based optimizer 
\begin{equation}
    z_t = x_t,  \quad z_{t+1}=z_t - \eta g_t;
\end{equation}

\textsc{Adam} optimizer 
\begin{equation}
    z_{t+1}=z_t + \frac{\eta \beta_1}{1-\beta_1}
    (\frac{1}{\sqrt{v_{t-1}}+\epsilon}-\frac{1}{\sqrt{v_{t}}+\epsilon})\odot m_{t-1}-\frac{\eta}{\sqrt{v_t}+\epsilon}\odot g_t;
\end{equation}

\textsc{Sadam} optimizer 
\begin{equation}
    z_{t+1}=z_t +\frac{\eta\beta_1}{1-\beta_1}(\frac{1}{softplus(\sqrt{v_{t-1}})}-\frac{1}{softplus(\sqrt{v_{t}})})\odot m_{t-1}-\frac{\eta}{softplus(\sqrt{v_t})}\odot g_t. 
\end{equation}
\end{lemma}

\begin{proof}
We consider the \textsc{Adam} optimizer and let $\beta_1= 0$, we can easily derive the gradient-based case.
\begin{align*}
    z_{t+1} &=x_{t+1} + \frac{\beta_1}{1-\beta_1}(x_{t+1} - x_{t})\\
    z_{t+1} &= z_t +\frac{1}{1-\beta_1}(x_{t+1} - x_{t})-\frac{\beta_1}{1-\beta_1}(x_{t} - x_{t-1})\\
    &=z_t -\frac{1}{1-\beta_1}\frac{\eta}{\sqrt{v_t}+\epsilon} \odot m_t +\frac{\beta_1}{1-\beta_1}\frac{\eta}{\sqrt{v_{t-1}}+\epsilon}\odot  m_{t-1}\\
    &=z_t + \frac{\eta \beta_1}{1-\beta_1}
    (\frac{1}{\sqrt{v_{t-1}}+\epsilon}-\frac{1}{\sqrt{v_{t}}+\epsilon})\odot m_{t-1}-\frac{\eta}{\sqrt{v_t}+\epsilon}\odot g_t.
\end{align*}

Similarly, consider the \textsc{Sadam} optimizer:
\begin{align*}
z_{t+1} &= z_t +\frac{1}{1-\beta_1}(x_{t+1} - x_{t})-\frac{\beta_1}{1-\beta_1}(x_{t} - x_{t-1})\\
&=z_t -\frac{1}{1-\beta_1}\frac{\eta}{softplus(\sqrt{v_t})}\odot m_t +\frac{\beta_1}{1-\beta_1}\frac{\eta}{softplus(\sqrt{v_{t-1})}}\odot m_{t-1}\\
&=z_t +\frac{\eta\beta_1}{1-\beta_1}(\frac{1}{softplus(\sqrt{v_{t-1}})}-\frac{1}{softplus(\sqrt{v_{t}})})\odot m_{t-1}-\frac{\eta}{softplus(\sqrt{v_t})}\odot g_t.
\end{align*}
\end{proof}

\begin{lemma}\label{square}
As defined in Lemma \ref{z_t}, with the condition that $v_t\geq v_{t-1}$, i.e., \textsc{AMSGrad} and \textsc{SAMSGrad}, we can derive the bound of distance of $\|z_{t+1}-z_t\|^2$ as follows:

\textsc{Adam} optimizer 
\begin{align}\label{eq:distance}
    E[\|z_{t+1}-z_t\|^2]&\leq \frac{2\eta^2\beta_1^2(\sigma^2+G^2)}{(1-\beta_1)^2}E[\sum_{j=1}^d (\frac{1}{\sqrt{v_{t-1,j}}+\epsilon})^2-(\frac{1}{\sqrt{v_{t,j}}+\epsilon})^2] \nonumber\\
    &+2\eta^2 \mu_2^2 (\sigma^2+G^2)
\end{align}

\textsc{Sadam} optimizer 
\begin{align}\label{eq:distance2}
    E[\|z_{t+1}-z_t\|^2] &\leq \frac{2\eta^2\beta_1^2(\sigma^2+G^2)}{(1-\beta_1)^2}E[\sum_{j=1}^d (\frac{1}{softplus(\sqrt{v_{t-1,j}})})^2-(\frac{1}{softplus(\sqrt{v_{t,j}})})^2] \nonumber\\
   &+2\eta^2 \mu_4^2 (\sigma^2+G^2)
\end{align}

\end{lemma}

\begin{proof}
\textsc{Adam} case:
\begin{align*}
   E[\|z_{t+1}-z_t\|^2]&= E[\|\frac{\eta \beta_1}{1-\beta_1}
    (\frac{1}{\sqrt{v_{t-1}}+\epsilon}-\frac{1}{\sqrt{v_{t}}+\epsilon})\odot m_{t-1}-\frac{\eta}{\sqrt{v_t}+\epsilon}\odot g_t\|^2]\\
    &\leq 2E[\|\frac{\eta \beta_1}{1-\beta_1}
    (\frac{1}{\sqrt{v_{t-1}}+\epsilon}-\frac{1}{\sqrt{v_{t}}+\epsilon})\odot m_{t-1}\|]^2 +2E[\|\frac{\eta}{\sqrt{v_t}+\epsilon}\odot g_t\|]^2\\
    &\leq \frac{2\eta^2\beta_1^2(\sigma^2+G^2)}{(1-\beta_1)^2}E[\sum_{j=1}^d (\frac{1}{\sqrt{v_{t-1,j}}+\epsilon}-\frac{1}{\sqrt{v_{t,j}}+\epsilon})^2]
    + 2\eta^2 \mu_2^2 (\sigma^2+G^2) \\
    &\leq \frac{2\eta^2\beta_1^2(\sigma^2+G^2)}{(1-\beta_1)^2}E[\sum_{j=1}^d (\frac{1}{\sqrt{v_{t-1,j}}+\epsilon})^2-(\frac{1}{\sqrt{v_{t,j}}+\epsilon})^2]
    +2\eta^2 \mu_2^2 (\sigma^2+G^2)\\
\end{align*}
The first inequality holds because $\|a-b\|^2\leq 2\|a\|^2 +2 \|b\|^2$, the second inequality holds because Lemma \ref{g_t} and \ref{m_t} and  Lemma \ref{bound1}, the third inequality holds because $(a-b)^2 \leq a^2- b^2$ when $a\geq b$, and in our assumption, we have $v_t \geq v_{t-1}$ holds.

\textsc{Sadam} case:
\begin{align*}
   E[\|z_{t+1}-z_t\|^2]&= E[\|\frac{\eta \beta_1}{1-\beta_1}
    (\frac{1}{softplus(\sqrt{v_{t-1}})}-\frac{1}{softplus(\sqrt{v_{t}})})\odot m_{t-1}-\frac{\eta}{softplus(\sqrt{v_t})} \odot g_t\|^2]\\
    &\leq 2E[\|\frac{\eta \beta_1}{1-\beta_1}
    (\frac{1}{softplus(\sqrt{v_{t-1}})}-\frac{1}{softplus(\sqrt{v_{t}})})\odot m_{t-1}\|]^2\\ &+2E[\|\frac{\eta}{softplus(\sqrt{v_t})}\odot g_t\|]^2\\
    &\leq \frac{2\eta^2\beta_1^2(\sigma^2+G^2)}{(1-\beta_1)^2}E[\sum_{j=1}^d (\frac{1}{softplus(\sqrt{v_{t-1,j}})}- \frac{1}{softplus(\sqrt{v_{t,j}})})^2]\\
    &+2\eta^2 \mu_4^2(\sigma^2+G^2)\\
    &\leq  \frac{2\eta^2\beta_1^2(\sigma^2+G^2)}{(1-\beta_1)^2}E[\sum_{j=1}^d (\frac{1}{softplus(\sqrt{v_{t-1,j}})})^2-(\frac{1}{softplus(\sqrt{v_{t,j}})})^2]\\
    &+2\eta^2 \mu_4^2(\sigma^2+G^2)\\
\end{align*}

Because the {\em softplus} function is monotone increasing function,
therefore, the third inequality holds as well. 
\end{proof}

\begin{lemma}\label{product}
As defined in Lemma \ref{z_t}, with the condition that $v_t\geq v_{t-1}$, we can derive the bound of the inner product as follows:

\textsc{Adam} optimizer 
\begin{equation} \label{equ:product3}
    -E[\langle \nabla f(z_t)-\nabla f(x_t), \frac{\eta}{\sqrt{v_t}+\epsilon}\odot g_t\rangle] \leq \frac{1}{2}L^2\eta^2\mu_2^2 (\frac{\beta_1}{1-\beta_1})^2 (\sigma^2+G^2) +  \frac{1}{2}\eta^2\mu_2^2( \sigma^2 + G^2);
\end{equation}

\textsc{Sadam} optimizer
\begin{equation} \label{equ:product4}
    	-E[\langle\nabla f(z_t)-\nabla f(x_t), \frac{\eta}{softplus(\sqrt{v_t})}\odot g_t\rangle]\leq \frac{1}{2}L^2\eta^2\mu_4^2 (\frac{\beta_1}{1-\beta_1})^2 (\sigma^2+G^2) +  \frac{1}{2}\eta^2\mu_4^2( \sigma^2 + G^2).
\end{equation}
\end{lemma}

\begin{proof}
Since the stochastic gradient is unbiased, then we have $E[g_t]=\nabla f(x_t)$.

\textsc{Adam} case:

\begin{align*}
    -E[\langle\nabla f(z_t)&-\nabla f(x_t), \frac{\eta}{\sqrt{v_t}+\epsilon}\odot g_t\rangle]\\
    &\leq \frac{1}{2}  E[\|\nabla f(z_t)-\nabla f(x_t)\|^2]+ \frac{1}{2}E[\|\frac{\eta}{\sqrt{v_t}+\epsilon}\odot g_t\|^2]\\
    &\leq \frac{L^2}{2}E[\|z_t- x_t\|^2]+
    \frac{1}{2}E[\|\frac{\eta}{\sqrt{v_t}+\epsilon}\odot g_t\|^2]\\
    &=\frac{L^2}{2}(\frac{\beta_1}{1-\beta_1})^2E[\|x_t- x_{t-1}\|^2]+ 
   \frac{1}{2}E[\|\frac{\eta}{\sqrt{v_t}+\epsilon}\odot g_t\|^2]\\
    &=\frac{L^2 }{2}(\frac{\beta_1}{1-\beta_1})^2E[\|\frac{\eta}{\sqrt{v_{t-1}}+\epsilon}\odot m_{t-1}\|^2]+ \frac{1}{2}E[\|\frac{\eta}{\sqrt{v_t}+\epsilon}\odot g_t\|^2]\\
   &\leq \frac{1}{2}L^2\eta^2\mu_2^2 (\frac{\beta_1}{1-\beta_1})^2 (\sigma^2+G^2) +  \frac{1}{2}\eta^2\mu_2^2( \sigma^2 + G^2)
\end{align*}

The first inequality holds because $\frac{1}{2}a^2+\frac{1}{2}b^2\geq -<a,b>$, the second inequality holds for L-smoothness, the last inequalities hold due to Lemma \ref{m_t} and \ref{bound1}.

Similarly, for \textsc{Sadam}, we also have the following result:
\begin{align*}
	-E[\langle\nabla f(z_t)&-\nabla f(x_t), \frac{\eta}{softplus(\sqrt{v_t})}\odot g_t\rangle]\\
	&\leq \frac{1}{2}  E[\|\nabla f(z_t)-\nabla f(x_t)\|^2]+ \frac{1}{2}E[\|\frac{\eta}{softplus(\sqrt{v_t})}\odot g_t\|^2]\\
	&\leq \frac{L^2}{2}E[\|z_t- x_t\|^2]+
	\frac{1}{2}E[\|\frac{\eta}{softplus(\sqrt{v_t})}\odot g_t\|^2]\\
	&=\frac{L^2}{2}(\frac{\beta_1}{1-\beta_1})^2E[\|x_t- x_{t-1}\|^2]+ 
	\frac{1}{2}E[\|\frac{\eta}{softplus(\sqrt{v_t})}\odot g_t\|^2]\\
	&=\frac{L^2 }{2}(\frac{\beta_1}{1-\beta_1})^2E[\|\frac{\eta}{softplus(\sqrt{v_{t-1}})}\odot m_{t-1}\|^2]+ \frac{1}{2}E[\|\frac{\eta}{softplus(\sqrt{v_t})}\odot g_t\|^2]\\
	&\leq \frac{1}{2}L^2\eta^2\mu_4^2 (\frac{\beta_1}{1-\beta_1})^2 (\sigma^2+G^2) +  \frac{1}{2}\eta^2\mu_4^2( \sigma^2 + G^2).
\end{align*}
\end{proof}

\subsection{\textsc{Adam} Convergence in Nonconvex Setting}

\begin{proof}
All the analyses hold true under the condition: $v_t\geq v_{t-1}$. From L-smoothness and Lemma \ref{z_t}, we have

\begin{align*}
    f(z_{t+1})&\leq f(z_t) + \langle\nabla f(z_t), z_{t+1}-z_t\rangle+\frac{L}{2}\|z_{t+1}-z_t\|^2\\
    &= f(z_t)+ \frac{\eta \beta_1}{1-\beta_1}
    \langle\nabla f(z_t), (\frac{1}{\sqrt{v_{t-1}}+\epsilon}-\frac{1}{\sqrt{v_{t}}+\epsilon}) \odot m_{t-1}\rangle\\
    &- \langle\nabla f(z_t), \frac{\eta}{\sqrt{v_t}+\epsilon} \odot g_t\rangle +\frac{L}{2}\|z_{t+1}-z_t\|^2\\
\end{align*}

Take expectation on both sides,

\begin{align*}
    E[f(z_{t+1})-f(z_t)] &\leq  \frac{\eta \beta_1}{1-\beta_1}
    E[\langle\nabla f(z_t), (\frac{1}{\sqrt{v_{t-1}}+\epsilon}-\frac{1}{\sqrt{v_{t}}+\epsilon}) \odot m_{t-1}\rangle]\\
     &- E[\langle\nabla f(z_t), \frac{\eta}{\sqrt{v_t}+\epsilon} \odot g_t\rangle] +\frac{L}{2}E[\|z_{t+1}-z_t\|^2]\\ 
     &= \frac{\eta \beta_1}{1-\beta_1}
    E[\langle\nabla f(z_t), (\frac{1}{\sqrt{v_{t-1}}+\epsilon}-\frac{1}{\sqrt{v_{t}}+\epsilon}) \odot m_{t-1}\rangle]\\
     &- E[\langle\nabla f(z_t)-\nabla f(x_t), \frac{\eta}{\sqrt{v_t}+\epsilon} \odot g_t\rangle] - E[\langle\nabla f(x_t), \frac{\eta}{\sqrt{v_t}+\epsilon} \odot g_t\rangle]\\
     &+\frac{L}{2}E[\|z_{t+1}-z_t\|^2]\\
\end{align*}

Plug in the results from prepared lemmas, then we have,

\begin{align*}
    E[f(z_{t+1})-f(z_t)]
    &\leq  \frac{\eta \beta_1}{1-\beta_1}
    E[\langle\nabla f(z_t), (\frac{1}{\sqrt{v_{t-1}}+\epsilon}-\frac{1}{\sqrt{v_{t}}+\epsilon}) \odot m_{t-1}\rangle]\\
     &+ \frac{1}{2}L^2\eta^2\mu_2^2 (\frac{\beta_1}{1-\beta_1})^2 (\sigma^2+G^2) +  \frac{1}{2}\eta^2\mu_2^2( \sigma^2 + G^2)
     - E[\langle\nabla f(x_t), \frac{\eta}{\sqrt{v_t}+\epsilon} \odot g_t\rangle]\\
    &+\frac{L\eta^2\beta_1^2(\sigma^2+G^2)}{(1-\beta_1)^2}E[\sum_{j=1}^d (\frac{1}{\sqrt{v_{t-1,j}}+\epsilon})^2-(\frac{1}{\sqrt{v_{t,j}}+\epsilon})^2]
    +L\eta^2\mu_2^2(\sigma^2+G^2)\\
\end{align*}    

Applying the bound of $m_t$ and $\nabla f(z_t)$,
\begin{align*}  
 E[f(z_{t+1})-f(z_t)]
    & \leq \frac{\eta \beta_1}{1-\beta_1} G\sqrt{\sigma^2+G^2} E[\sum_{j=1}^d \frac{1}{\sqrt{v_{t-1,j}}+\epsilon}-\frac{1}{\sqrt{v_{t,j}}+\epsilon}] \\
   &+ \frac{1}{2}L^2\eta^2\mu_2^2 (\frac{\beta_1}{1-\beta_1})^2 (\sigma^2+G^2) +  \frac{1}{2}\eta^2\mu_2^2( \sigma^2 + G^2)
     - E[\langle\nabla f(x_t), \frac{\eta}{\sqrt{v_t}+\epsilon} \odot g_t\rangle]\\
    &+\frac{L\eta^2\beta_1^2(\sigma^2+G^2)}{(1-\beta_1)^2}E[\sum_{j=1}^d (\frac{1}{\sqrt{v_{t-1,j}}+\epsilon})^2-(\frac{1}{\sqrt{v_{t,j}}+\epsilon})^2]
    +L\eta^2\mu_2^2(\sigma^2+G^2)\\
\end{align*}

By rearranging,
\begin{align*}
E[\langle\nabla f(x_t),\frac{\eta}{\sqrt{v_t}+\epsilon} \odot g_t\rangle ]
&\leq E[ f( z_t) - f( z_{t+1})] + \frac{\eta \beta_1}{1-\beta_1} G\sqrt{\sigma^2+G^2} E[\sum_{j=1}^d \frac{1}{\sqrt{v_{t-1,j}}+\epsilon}-\frac{1}{\sqrt{v_{t,j}}+\epsilon}] \\
 &+ \frac{1}{2}L^2\eta^2\mu_2^2 (\frac{\beta_1}{1-\beta_1})^2 (\sigma^2+G^2) +  \frac{1}{2}\eta^2\mu_2^2( \sigma^2 + G^2)\\
 &+\frac{L\eta^2\beta_1^2(\sigma^2+G^2)}{(1-\beta_1)^2}E[\sum_{j=1}^d (\frac{1}{\sqrt{v_{t-1,j}}+\epsilon})^2-(\frac{1}{\sqrt{v_{t,j}}+\epsilon})^2] +L\eta^2\mu_2^2(\sigma^2+G^2)\\
\end{align*}

For the LHS above: 
\begin{align*}
	E[\langle\nabla f(x_t), \frac{1}{\sqrt{v_t}+\epsilon} \odot g_t\rangle] 
	&\geq E[ \sum_{\{j| \nabla f( x_{t,j}) g_{t,j}\geq 0\}} \mu_1  \nabla f( x_{t,j}) g_{t,j}  + \sum_{\{j|  \nabla f( x_{t,j}) g_{t,j}< 0\}} \mu_2  \nabla f( x_{t,j}) g_{t,j}]\\
	&\geq E[ \sum_{\{j| \nabla f( x_{t,j}) g_{t,j}\geq 0\}} \mu_1  \nabla f( x_{t,j})^2  + \sum_{\{j|  \nabla f( x_{t,j}) g_{t,j}< 0\}} \mu_2  \nabla f( x_{t,j})^2]\\
	&\geq \mu_1 \| \nabla f( x_t)\|^2
\end{align*}

Then we obtain: 
\begin{align*}
\eta\mu_1 \| \nabla f( x_t)\|^2&
\leq  E[ f( z_t) - f( z_{t+1})] +\frac{\eta \beta_1}{1-\beta_1} G\sqrt{\sigma^2+G^2} E[\sum_{j=1}^d \frac{1}{\sqrt{v_{t-1,j}}+\epsilon}-\frac{1}{\sqrt{v_{t,j}}+\epsilon}] \\
&+ \frac{1}{2}L^2\eta^2\mu_2^2 (\frac{\beta_1}{1-\beta_1})^2 (\sigma^2+G^2) +  \frac{1}{2}\eta^2\mu_2^2( \sigma^2 + G^2)\\
 &+\frac{L\eta^2\beta_1^2(\sigma^2+G^2)}{(1-\beta_1)^2}E[\sum_{j=1}^d (\frac{1}{\sqrt{v_{t-1,j}}+\epsilon})^2-(\frac{1}{\sqrt{v_{t,j}}+\epsilon})^2] +L\eta^2\mu_2^2(\sigma^2+G^2)\\
\end{align*}

Divide $\eta \mu_1$ on both sides:
\begin{align*}
 \| \nabla f( x_t)\|^2
 &\leq \frac{1}{\eta \mu_1}  E[ f( z_t) - f( z_{t+1})] + \frac{\beta_1 }{(1-\beta_1)\mu_1 } G\sqrt{\sigma^2+G^2} E[\sum_{j=1}^d \frac{1}{\sqrt{v_{t-1,j}}+\epsilon}-\frac{1}{\sqrt{v_{t,j}}+\epsilon}] \\
&+ \frac{1}{2\mu_1}L^2\eta\mu_2^2 (\frac{\beta_1}{1-\beta_1})^2 (\sigma^2+G^2) +  \frac{1}{2\mu_1}\eta\mu_2^2( \sigma^2 + G^2)\\
 &+\frac{L\eta\beta_1^2(\sigma^2+G^2)}{(1-\beta_1)^2\mu_1}E[\sum_{j=1}^d (\frac{1}{\sqrt{v_{t-1,j}}+\epsilon})^2-(\frac{1}{\sqrt{v_{t,j}}+\epsilon})^2] +\frac{L\eta\mu_2^2}{\mu_1}(\sigma^2+G^2)\\
\end{align*}

Summing from $t=1$ to $T$, where $T$ is the maximum number of iteration,
\begin{align*}
 \sum_{t=1}^{T}[ \| \nabla f( x_t)\|^2]
 &\leq \frac{1}{\eta \mu_1}  E[ f( z_1) - f^* ] + \frac{\beta_1 }{(1-\beta_1)\mu_1 } G\sqrt{\sigma^2+G^2} E[\sum_{j=1}^d \frac{1}{\sqrt{v_{0,j}}+\epsilon}-\frac{1}{\sqrt{v_{T,j}}+\epsilon}] \\
&+ \frac{T}{2\mu_1}L^2\eta\mu_2^2 (\frac{\beta_1}{1-\beta_1})^2 (\sigma^2+G^2) +  \frac{T}{2\mu_1}\eta\mu_2^2( \sigma^2 + G^2)\\
 &+\frac{L\eta\beta_1^2(\sigma^2+G^2)}{(1-\beta_1)^2\mu_1}E[\sum_{j=1}^d (\frac{1}{\sqrt{v_{0,j}}+\epsilon})^2-(\frac{1}{\sqrt{v_{T,j}}+\epsilon})^2] +\frac{T L\eta\mu_2^2}{\mu_1}(\sigma^2+G^2)\\
\end{align*}

Since $v_0 = 0$, $\mu_2 = \frac{1}{\epsilon}$, we have
\begin{align*}
 \sum_{t=1}^{T}[ \| \nabla f( x_t)\|^2]
 &\leq \frac{1}{\eta \mu_1}  E[ f( z_1) - f^* ] + \frac{\beta_1 d }{(1-\beta_1)\mu_1 } G\sqrt{\sigma^2+G^2} (\mu_2 - \mu_1)\\
&+ \frac{T}{2\mu_1}L^2\eta\mu_2^2 (\frac{\beta_1}{1-\beta_1})^2 (\sigma^2+G^2) +  \frac{T}{2\mu_1}\eta\mu_2^2( \sigma^2 + G^2)\\
 &+\frac{L\eta\beta_1^2 d (\sigma^2+G^2)}{(1-\beta_1)^2\mu_1} (\mu_2^2 - \mu_1^2)
 +\frac{T L\eta\mu_2^2}{\mu_1}(\sigma^2+G^2)\\
\end{align*}

Divided by $\frac{1}{T}$,
\begin{align*}
\frac{1}{T}\sum_{t=1}^{T}[ \| \nabla f( x_t)\|^2]
 &\leq \frac{1}{\eta \mu_1 T}  E[ f( z_1) - f^* ] 
 +\frac{\beta_1 d }{(1-\beta_1)\mu_1 T} G\sqrt{\sigma^2+G^2} (\mu_2 - \mu_1)\\
&+ \frac{1}{2\mu_1}L^2\eta\mu_2^2 (\frac{\beta_1}{1-\beta_1})^2 (\sigma^2+G^2) +  \frac{1}{2\mu_1}\eta\mu_2^2( \sigma^2 + G^2)\\
 &+\frac{L\eta\beta_1^2 d (\sigma^2+G^2)}{(1-\beta_1)^2\mu_1 T} (\mu_2^2 - \mu_1^2)
 +\frac{ L\eta\mu_2^2}{\mu_1}(\sigma^2+G^2)\\
 &\leq \frac{1}{\eta \mu_1 T}  E[ f( z_1) - f^* ] + (\frac{\beta_1 d }{(1-\beta_1)\mu_1 T}  (\mu_2 - \mu_1)\\
 &+ \frac{1}{2\mu_1}L^2\eta\mu_2^2 (\frac{\beta_1}{1-\beta_1})^2  +  \frac{\eta\mu_2^2}{2\mu_1}
 +\frac{L\eta\beta_1^2 d (\mu_2^2 - \mu_1^2)}{(1-\beta_1)^2\mu_1 T} 
 +\frac{ L\eta\mu_2^2}{\mu_1})(\sigma^2+G^2)
\end{align*}
The second inequality holds because $G\sqrt{\sigma^2+G^2} \leq \sigma^2+G^2$.

Setting $\eta = \frac{1}{\sqrt{T}}$, let $x_0=x_1$, then $z_1 = x_1$, $f(z_1)=f(x_1)$ we derive the final result:

\begin{align*}
    \min_{t=1,\dots,T}E[\|\nabla f(x_t)\|^2] &\leq \frac{1}{\mu_1\sqrt{T}}  E[ f( x_1) - f^*] + (\frac{\beta_1 d}{(1-\beta_1)\mu_1 T } ( \mu_2 - \mu_1) \\
    &+ \frac{L^2\mu_2^2}{2\mu_1\sqrt{T}} (\frac{\beta_1}{1-\beta_1})^2  +  \frac{\mu_2^2}{2\mu_1\sqrt{T}}
 +\frac{L\beta_1^2 d (\mu_2^2 - \mu_1^2)}{(1-\beta_1)^2\mu_1 T\sqrt{T}} 
 +\frac{ L\mu_2^2}{\mu_1\sqrt{T}})(\sigma^2+G^2)\\
    &= \frac{C_1}{\sqrt{T}}+ \frac{C_2}{T} + \frac{C_3}{T\sqrt{T}}
\end{align*}
where
\begin{align*}
C_1 & = \frac{1}{\mu_1}[f(x_{1})-f^*]+ (\frac{L^2\mu_2^2}{2\mu_1} (\frac{\beta_1}{1-\beta_1})^2 +  \frac{\mu_2^2}{2\mu_1}  +\frac{ L\mu_2^2}{\mu_1})(\sigma^2+G^2) \\
C_2 & = \frac{\beta_1 ( \mu_2 - \mu_1)d}{(1-\beta_1)\mu_1 } , \\
C_3 & = \frac{L\beta_1^2 d (\mu_2^2 - \mu_1^2)}{(1-\beta_1)^2\mu_1}.
\end{align*}

With fixed $ L, \sigma, G, \beta_1$, we have $C_1 = O(\frac{1}{\epsilon^2})$, $C_2 = O(\frac{d}{\epsilon})$, $C_3 = O(\frac{d}{\epsilon^2})$. Therefore,
\begin{align*}
    \min_{t=1,\dots,T}E[\|\nabla f(x_t)\|^2] &\leq O( \frac{1}{\epsilon^2 \sqrt{T}}+ \frac{d}{\epsilon T} + \frac{d}{\epsilon^2 T\sqrt{T}})
\end{align*}

\end{proof}

Thus, we get the sublinear convergence rate of \textsc{Adam} in nonconvex setting, which recovers the well-known result of SGD (\citep{ghadimi2013stochastic}) in nonconvex optimization in terms of $T$. 

\begin{remark}
The leading item from the above convergence is $C_1/\sqrt{T}$, $\epsilon$ plays an essential role in the complexity, and we derive a more accurate order $O(\frac{1}{\epsilon^2\sqrt{T}})$. At present, $\epsilon $ is always underestimated and considered to be not associated with accuracy of the solution (\citep{zaheer2018adaptive}). However, it is closely related with complexity, and with bigger $\epsilon$, the computational complexity should be better. This also supports the analysis of A-LR: $\frac{1}{\sqrt{v_t}+\epsilon }$ of \textsc{Adam} in our main paper.
\end{remark}

In some other works, people use $\sigma_i$ or $G_i$ to show all the element-wise bound, and then by applying $\sum_{j=1}^d \sigma_i := \sigma$, $\sum_{j=1}^d G_i := G$ to hide $d$ in the complexity. Here in our work, we didn't specify write out  $\sigma_i$ or $G_i$, instead we use $\sigma, G$ through all the procedure. 

\subsection{\textsc{Sadam} Convergence in Nonconvex Setting}

As \textsc{Sadam} also has constrained bound pair $(\mu_3, \mu_4)$, we can learn from the proof of \textsc{Adam} method, which provides us a general framework of such kind of adaptive methods.

Similar to the \textsc{Adam} proof, from L-smoothness and Lemma \ref{z_t} , we have
\begin{proof}
	All the analyses hold true under the condition: $v_t\geq v_{t-1}$. From L-smoothness and Lemma \ref{z_t}, we have
	\begin{align*}
	f(z_{t+1})&\leq f(z_t) + \langle\nabla f(z_t), z_{t+1}-z_t\rangle+\frac{L}{2}\|z_{t+1}-z_t\|^2\\
	&= f(z_t)+ \frac{\eta \beta_1}{1-\beta_1}
	\langle\nabla f(z_t), (\frac{1}{softplus(\sqrt{v_{t-1}})}-\frac{1}{softplus(\sqrt{v_{t}})}) \odot m_{t-1}\rangle\\
	&- \langle\nabla f(z_t), \frac{\eta}{softplus(\sqrt{v_{t}})} \odot g_t\rangle +\frac{L}{2}\|z_{t+1}-z_t\|^2\\
\end{align*}
	
Taking expectation on both sides, and plug in the results from prepared lemmas, then we have,

\begin{align*}
&E[f(z_{t+1})-f(z_t)] \\
&\leq\frac{\eta \beta_1}{1-\beta_1}E[\langle\nabla f(z_t), (\frac{1}{softplus(\sqrt{v_{t-1}})}-\frac{1}{softplus(\sqrt{v_{t}})}) \odot m_{t-1}\rangle]\\
&- E[\langle\nabla f(z_t), \frac{\eta}{softplus(\sqrt{v_{t}})} \odot g_t\rangle] +\frac{L}{2}E[\|z_{t+1}-z_t\|^2]\\ 
&\leq  \frac{\eta \beta_1}{1-\beta_1}E[\langle\nabla f(z_t), (\frac{1}{softplus(\sqrt{v_{t-1}})}-\frac{1}{softplus(\sqrt{v_{t}})}) \odot m_{t-1}\rangle]\\
&- E[\langle\nabla f(z_t), \frac{\eta}{softplus(\sqrt{v_{t}})} \odot g_t\rangle]\\
&+\frac{L\eta^2\beta_1^2(\sigma^2+G^2)}{(1-\beta_1)^2}E[\sum_{j=1}^d (\frac{1}{softplus(\sqrt{v_{t-1,j}})})^2-(\frac{1}{softplus(\sqrt{v_{t,j}})})^2] +L\eta^2\mu_4^2(\sigma^2+G^2)\\
& = \frac{\eta \beta_1}{1-\beta_1} G\sqrt{\sigma^2+G^2} E[\sum_{j=1}^d \frac{1}{softplus(\sqrt{v_{t-1,j}})}-\frac{1}{softplus(\sqrt{v_{t,j}})}] \\
&- E[\langle\nabla f(z_t) - \nabla f(x_t), \frac{\eta}{softplus(\sqrt{v_{t}})} \odot g_t\rangle]- E[\langle\nabla f(x_t), \frac{\eta}{softplus(\sqrt{v_{t}})} \odot g_t\rangle]\\
&+\frac{L\eta^2\beta_1^2(\sigma^2+G^2)}{(1-\beta_1)^2}E[\sum_{j=1}^d (\frac{1}{softplus(\sqrt{v_{t-1,j}})})^2-(\frac{1}{softplus(\sqrt{v_{t,j}})})^2] +L\eta^2\mu_4^2(\sigma^2+G^2)\\
& \leq \frac{\eta \beta_1}{1-\beta_1} G\sqrt{\sigma^2+G^2} E[\sum_{j=1}^d \frac{1}{softplus(\sqrt{v_{t-1,j}})}-\frac{1}{softplus(\sqrt{v_{t,j}})}] \\
&+ \frac{L^2\eta^2 \mu_4^2}{2}(\frac{\beta_1}{1-\beta_1})^2 (\sigma^2+G^2) +  \frac{\eta^2\mu_4^2}{2}( \sigma^2 + G^2)
- E[\langle\nabla f(x_t), \frac{\eta}{softplus(\sqrt{v_{t}})} \odot g_t\rangle]\\
&+\frac{L\eta^2\beta_1^2(\sigma^2+G^2)}{(1-\beta_1)^2}E[\sum_{j=1}^d (\frac{1}{softplus(\sqrt{v_{t-1,j}})})^2-(\frac{1}{softplus(\sqrt{v_{t,j}})})^2]+L\eta^2\mu_4^2(\sigma^2+G^2)
\end{align*}
	
By rearranging,
\begin{align*}
	&E[\langle\nabla f(x_t), \frac{\eta}{softplus(\sqrt{v_{t}})} \odot g_t\rangle] \\
	&\leq E[ f( z_t) - f( z_{t+1})] + \frac{\eta \beta_1}{1-\beta_1} G\sqrt{\sigma^2+G^2} E[\sum_{j=1}^d \frac{1}{softplus(\sqrt{v_{t-1,j}})}-\frac{1}{softplus(\sqrt{v_{t,j}})}] \\
	&+ \frac{L^2\eta^2 \mu_4^2}{2}(\frac{\beta_1}{1-\beta_1})^2 (\sigma^2+G^2) +  \frac{\eta^2 \mu_4^2}{2}( \sigma^2 + G^2)\\
	&+\frac{L\eta^2\beta_1^2(\sigma^2+G^2)}{(1-\beta_1)^2}E[\sum_{j=1}^d (\frac{1}{softplus(\sqrt{v_{t-1,j}})})^2-(\frac{1}{softplus(\sqrt{v_{t,j}})})^2] +L\eta^2\mu_4^2(\sigma^2+G^2)
\end{align*}
	
For the LHS above: 
	\begin{align*}
	E[\langle\nabla f(x_t), \frac{1}{softplus(\sqrt{v_{t}})} \odot g_t\rangle]
	&\geq E[ \sum_{\{j| \nabla f( x_{t,j}) g_{t,j}\geq 0\}} \mu_3 \nabla f( x_{t,j}) g_{t,j}  + \sum_{\{j|  \nabla f( x_{t,j}) g_{t,j}< 0\}} \mu_4  \nabla f( x_{t,j}) g_{t,j}\\
	&\geq E[ \sum_{\{j| \nabla f( x_{t,j}) g_{t,j}\geq 0\}} \mu_3  \nabla f( x_{t,j})^2  + \sum_{\{j|  \nabla f( x_{t,j}) g_{t,j}< 0\}} \mu_4  \nabla f( x_{t,j})^2\\
	&\geq \mu_3 \| \nabla f( x_t)\|^2
	\end{align*}
	
Then we obtain: 
\begin{align*}
\eta\mu_3 \| \nabla f( x_t)\|^2&
	\leq  E[ f( z_t) - f( z_{t+1})] +\frac{\eta \beta_1}{1-\beta_1} G\sqrt{\sigma^2+G^2} E[\sum_{j=1}^d \frac{1}{softplus(\sqrt{v_{t-1,j}})}-\frac{1}{softplus(\sqrt{v_{t,j}})}] \\
	&+ \frac{L^2\eta^2 \mu_4^2}{2}(\frac{\beta_1}{1-\beta_1})^2 (\sigma^2+G^2) +  \frac{\eta^2 \mu_4^2}{2}( \sigma^2 + G^2)\\
	&+\frac{L\eta^2\beta_1^2(\sigma^2+G^2)}{(1-\beta_1)^2}E[\sum_{j=1}^d (\frac{1}{softplus(\sqrt{v_{t-1,j}})})^2-(\frac{1}{softplus(\sqrt{v_{t,j}})})^2] +L\eta^2\mu_4^2(\sigma^2+G^2)
	\end{align*}
	
Divide $\eta\mu_3$ on both sides and then sum from $t=1$ to $T$, where $T$ is the maximum number of iteration,

	\begin{align*}
	\sum_{t=1}^{T}[ \| \nabla f( x_t)\|^2]&
	\leq \frac{1}{\eta \mu_3}  E[ f( z_1) - f^*] + \frac{\beta_1}{(1-\beta_1)\mu_3 } G\sqrt{\sigma^2+G^2} E[\sum_{j=1}^d \frac{1}{softplus(\sqrt{v_{0,j}})}-\frac{1}{softplus(\sqrt{v_{T,j}})}] \\
	&+ \frac{L^2\eta T\mu_4^2 }{2\mu_3}(\frac{\beta_1}{1-\beta_1})^2 (\sigma^2+G^2) +  \frac{\eta\mu_4^2 T}{2\mu_3}( \sigma^2 + G^2)\\
	&+\frac{L\eta\beta_1^2 (\sigma^2+G^2)}{(1-\beta_1)^2\mu_3}E[\sum_{j=1}^d (\frac{1}{softplus(\sqrt{v_{0,j}})})^2-(\frac{1}{softplus(\sqrt{v_{T,j}})})^2]
	+\frac{L\eta\mu_4^2 T(\sigma^2+G^2)}{\mu_3}
	\end{align*}
	
	Since $v_0 = 0$, $\frac{1}{softplus(0)}=\mu_4 $, we have
	
	\begin{align*}
	\sum_{t=1}^{T}[ \| \nabla f( x_t)\|^2]&
\leq \frac{1}{\eta \mu_3}  E[ f( z_1) - f^*] + \frac{\beta_1 d}{(1-\beta_1)\mu_3 } G\sqrt{\sigma^2+G^2} (\mu_4 -\mu_3)\\
	&+ \frac{L^2\eta T\mu_4^2 }{2\mu_3}(\frac{\beta_1}{1-\beta_1})^2 (\sigma^2+G^2) +  \frac{\eta\mu_4^2 T}{2\mu_3}( \sigma^2 + G^2)\\
	&+\frac{L\eta\beta_1^2 d(\sigma^2+G^2)}{(1-\beta_1)^2\mu_3}(\mu_4^2 - \mu_3^2)
	+\frac{L\eta\mu_4^2 T(\sigma^2+G^2)}{\mu_3}
	\end{align*}
	
	Divided by $\frac{1}{T}$,
	\begin{align*}
	\frac{1}{T}\sum_{t=1}^{T}[ \| \nabla f( x_t)\|^2]
	&\leq \frac{1}{\eta \mu_3 T}  E[ f( z_1) - f^*] + \frac{\beta_1 d}{(1-\beta_1)\mu_3 T} G\sqrt{\sigma^2+G^2} (\mu_4 -\mu_3)\\
	&+ \frac{L^2\eta \mu_4^2 }{2\mu_3}(\frac{\beta_1}{1-\beta_1})^2 (\sigma^2+G^2) +  \frac{\eta\mu_4^2 }{2\mu_3}( \sigma^2 + G^2)\\
	&+\frac{L\eta\beta_1^2 d(\sigma^2+G^2)}{(1-\beta_1)^2\mu_3 T}(\mu_4^2 - \mu_3^2)
	+\frac{L\eta\mu_4^2 (\sigma^2+G^2)}{\mu_3 }\\
	& \leq  \frac{1}{\eta \mu_3 T} E[ f( z_1) - f^*] + (\frac{\beta_1 d}{(1-\beta_1)\mu_3 T}(\mu_4 -\mu_3)\\
	&+ \frac{L^2\eta \mu_4^2 }{2\mu_3}(\frac{\beta_1}{1-\beta_1})^2  +  \frac{\eta\mu_4^2 }{2\mu_3}
	+\frac{L\eta\beta_1^2 d}{(1-\beta_1)^2\mu_3 T}(\mu_4^2 - \mu_3^2)
	+\frac{L\eta\mu_4^2}{\mu_3 }
	) (\sigma^2+G^2)
	\end{align*}
	
	Setting $\eta = \frac{1}{\sqrt{T}}$, let $x_0=x_1$, then $z_1 = x_1$, $f(z_1)=f(x_1)$ we derive the final result for \textsc{Sadam} method:
	
	\begin{align*}
	\min_{t=1,\dots,T}E[\|\nabla f(x_t)\|^2] &\leq \frac{1}{\mu_3 \sqrt{T}}  E[ f( x_1) - f^*] + (\frac{\beta_1 d}{(1-\beta_1)\mu_3 T}(\mu_4 -\mu_3)\\
	&+ \frac{L^2 \mu_4^2 }{2\mu_3\sqrt{T}}(\frac{\beta_1}{1-\beta_1})^2  +  \frac{\mu_4^2}{2\mu_3\sqrt{T}}
	+\frac{L\beta_1^2 d(\mu_4^2 - \mu_3^2)}{(1-\beta_1)^2\mu_3 T\sqrt{T}}
	+\frac{L\mu_4^2}{\mu_3\sqrt{T}}
	) (\sigma^2+G^2)\\
 &= \frac{C_1}{\sqrt{T}}+ \frac{C_2}{T} + \frac{C_3}{T\sqrt{T}}
\end{align*}
where
\begin{align*}
C_1 & = \frac{1}{\mu_3}[f(x_{1})-f^*]+ (\frac{L^2\mu_4^2}{2\mu_3} (\frac{\beta_1}{1-\beta_1})^2 +  \frac{\mu_4^2}{2\mu_3}  +\frac{ L\mu_4^2}{\mu_3})(\sigma^2+G^2) \\
C_2 & = \frac{\beta_1 ( \mu_4 - \mu_3)d}{(1-\beta_1)\mu_3 } , \\
C_3 & = \frac{L\beta_1^2 d (\mu_4^2 - \mu_3^2)}{(1-\beta_1)^2\mu_3}.
\end{align*}

With fixed $ L, \sigma, G, \beta_1$, we have $C_1 = O(\beta^2)$, $C_2 = O(d \beta)$, $C_3 = O(d \beta^2)$. Therefore,
\begin{align*}
    \min_{t=1,\dots,T}E[\|\nabla f(x_t)\|^2] &\leq O( \frac{\beta^2}{\sqrt{T}}+ \frac{d \beta}{T} + \frac{d \beta^2}{ T\sqrt{T}})
\end{align*}

\end{proof}

Thus, we get the sublinear convergence rate of \textsc{Sadam} in nonconvex setting, which is the same order of \textsc{Adam} and recovers the well-known result of SGD \citep{ghadimi2013stochastic} in nonconvex optimization in terms of $T$. 

\begin{remark}
   The leading item from the above convergence is $C_1 /\sqrt{T}$,  $\beta$ plays an essential role in the complexity, and a more accurate convergence should be $O(\frac{\beta log(1+e^\beta)}{ \sqrt{T}})$. When $\beta$ is chosen big, this will become $O(\frac{\beta^2}{\sqrt{T}})$, somehow behave like \textsc{Adam}'s case as $O(\frac{1}{\epsilon^2\sqrt{T}})$, which also guides us to have a range of $\beta$; when $\beta$ is chosen small, this will become $O(\frac{1}{\sqrt{T}})$, the computational complexity will get close to SGD case, and $\beta$ is a much smaller number compared with $1/\epsilon$, proving that \textsc{Sadam} converges faster.
    This also supports the analysis of range of A-LR: $1/softplus(\sqrt{v_t})$ in our main paper.
\end{remark}

\subsection{Non-strongly Convex}
In previous works, convex case has been well-studied in adaptive gradient methods. \textsc{AMSGrad} and later methods \textsc{PAMSGrad} both use a projection on minimizing objective function, here we want to show a different way of proof in non-strongly convex case. For consistency, we still follow the construction of sequence $\{z_t\}$.

Starting from convexity: $$f(y)\geq f(x) + \nabla f(x)^T (y-x).$$

Then, for any $x\in \mathbb{R}^d$, $\forall t \in [1,T]$,

\begin{equation}\label{equ:cov}
\langle \nabla f(x), x_t - x^*\rangle \geq f(x_t)-f^*,
\end{equation}
where $f^* = f(x^*)$, $x^*$ is the optimal solution. 

\begin{proof}
\textsc{Adam} case: 

In the updating rule of \textsc{Adam} optimizer, $x_{t+1}=x_t - \frac{\eta_t}{\sqrt{v_t}+\epsilon}\odot m_t$, setting stepsize to be fixed, $\eta_t = \eta $, and assume $v_t \geq v_{t-1}$ holds. Using previous results,

\begin{align*}
    &E[\|z_{t+1}-x^* \|^2] \\
    &=  E[\|z_t + \frac{\eta \beta_1}{1-\beta_1}
    (\frac{1}{\sqrt{v_{t-1}}+\epsilon}-\frac{1}{\sqrt{v_{t}}+\epsilon})\odot m_{t-1}-\frac{\eta}{\sqrt{v_t}+\epsilon}\odot g_t-x^* \|^2]\\
    &= E[\|z_{t} -x^* \|^2] + E[\|\frac{\eta \beta_1}{1-\beta_1}
    (\frac{1}{\sqrt{v_{t-1}}+\epsilon}-\frac{1}{\sqrt{v_{t}}+\epsilon})\odot m_{t-1}-\frac{\eta}{\sqrt{v_t}+\epsilon}\odot g_t\|^2] \\
    &+ 2E[\langle\frac{\eta\beta_1}{1-\beta_1}(\frac{1}{\sqrt{v_{t-1}}+\epsilon}-\frac{1}{\sqrt{v_{t}}+\epsilon})\odot m_{t-1}, z_t-x^*\rangle]-2E[\langle\frac{\eta}{\sqrt{v_t}+\epsilon}\odot g_t,z_t-x^*\rangle]\\
    &\leq E[\|z_{t} -x^* \|^2] +2\frac{\eta^2\beta_1^2}{(1-\beta_1)^2}E[\|(\frac{1}{\sqrt{v_{t-1}}+\epsilon}-\frac{1}{\sqrt{v_{t}}+\epsilon})\odot m_{t-1}\|^2] + 2\eta^2 E[\|\frac{1}{\sqrt{v_t}+\epsilon}\odot g_t\|^2] \\
     &+ 2\frac{\eta\beta_1}{1-\beta_1}E[\langle (\frac{1}{\sqrt{v_{t-1}}+\epsilon}-\frac{1}{\sqrt{v_{t}}+\epsilon}) \odot m_{t-1}, z_t-x^*\rangle]
     -2\eta E[\langle\frac{1}{\sqrt{v_t}+\epsilon}\odot g_t,z_t-x^* \rangle]\\
     &\leq E[\|z_{t} -x^* \|^2] +2\frac{\eta^2\beta_1^2(\sigma^2+G^2)}{(1-\beta_1)^2}E[\sum_{j=1}^d (\frac{1}{\sqrt{v_{t-1}}+\epsilon})^2-(\frac{1}{\sqrt{v_{t}}+\epsilon})^2]
     +2\eta^2\mu_2^2(\sigma^2+G^2)\\
     &+ 2\frac{\eta\beta_1}{1-\beta_1}E[\langle (\frac{1}{\sqrt{v_{t-1}}+\epsilon}-\frac{1}{\sqrt{v_{t}}+\epsilon}) \odot m_{t-1}, z_t-x^*\rangle]
     -2\eta E[\langle\frac{1}{\sqrt{v_t}+\epsilon}\odot g_t,z_t-x^*\rangle]\\
 \end{align*}  
 The first inequality holds due to $\|a-b\|^2\leq 2\|a\|^2 +2\|b\|^2$, the second inequality holds due to Lemma \ref{g_t}, \ref{m_t}, \ref{bound1}.

Since $<a,b>\leq \frac{1}{2\eta}a^2 +\frac{\eta}{2}b^2$,
\begin{align*}
    &2E[\langle (\frac{1}{\sqrt{v_{t-1}}+\epsilon}-\frac{1}{\sqrt{v_{t}}+\epsilon}) \odot m_{t-1}, z_t-x^* \rangle]\\
    &\leq \frac{1}{\eta} E[\|(\frac{1}{\sqrt{v_{t-1}}+\epsilon}-\frac{1}{\sqrt{v_{t}}+\epsilon}) \odot m_{t-1}\|^2] +\eta E[\|z_t - x^*\|^2]\\
    &\leq\frac{1}{\eta} (\sigma^2+G^2)E[\sum_{j=1}^d (\frac{1}{\sqrt{v_{t-1,j}}+\epsilon})^2-(\frac{1}{\sqrt{v_{t,j}}+\epsilon})^2]+ \eta E[\|z_t - x^*\|^2]\\
\end{align*}
From the definition of $z_t$ and convexity,
\begin{align*}
	\langle\nabla f(x_t),x_t-x^*\rangle \geq f(x_t)-f^*\geq 0
\end{align*}

\begin{align*}
&-2\eta E[\langle\frac{1}{\sqrt{v_t}+\epsilon}\odot g_t ,z_t-x^* \rangle ] \\
&=-2\eta E[ \langle \frac{1}{\sqrt{v_t}+\epsilon}\odot  g_t,x_t-x^* +\frac{\beta_1}{1-\beta_1}(x_t-x_{t-1})\rangle ]\\
&=-2\eta E[\langle \frac{1}{\sqrt{v_t}+\epsilon}\odot g_t,x_t-x^*\rangle ]- \frac{2\eta\beta_1}{1-\beta_1} E[\langle \frac{1}{\sqrt{v_t}+\epsilon}\odot g_t,x_t-x_{t-1}\rangle ]\\
&=-2\eta E[ \langle \frac{1}{\sqrt{v_t}+\epsilon}\odot g_t,x_t-x^*\rangle ]- \frac{2\eta^2\beta_1}{1-\beta_1} E[\langle \frac{1}{\sqrt{v_t}+\epsilon}\odot g_t,\frac{1}{\sqrt{v_{t-1}}+\epsilon}\odot m_{t-1}\rangle ]\\
&\leq -2\eta\mu_1 \langle \nabla f( x_t), x_t - x^*\rangle + \frac{2\eta^2\beta_1\mu_2^2}{(1-\beta_1)}( \sigma^2 + G^2)\\
&\leq -2\eta\mu_1  (f( x_t) - f^*) + \frac{2\eta^2\beta_1\mu_2^2}{(1-\beta_1)}( \sigma^2 + G^2)
\end{align*}

Plugging in previous two inequalities:

\begin{align*}
&E[\|z_{t+1}-x^* \|^2]\\
&\leq E[\|z_{t} -x^* \|^2] +2\frac{\eta^2\beta_1^2(\sigma^2+G^2)}{(1-\beta_1)^2}E[\sum_{j=1}^d (\frac{1}{\sqrt{v_{t-1}}+\epsilon})^2-(\frac{1}{\sqrt{v_{t}}+\epsilon})^2]
     +2\eta^2\mu_2^2(\sigma^2+G^2)\\
&+ \frac{\beta_1(\sigma^2+G^2)}{1-\beta_1}E[\sum_{j=1}^d (\frac{1}{\sqrt{v_{t-1,j}}+\epsilon})^2-(\frac{1}{\sqrt{v_{t,j}}+\epsilon})^2]+ \frac{\eta^2\beta_1}{1-\beta_1} E[\|z_t - x^*\|^2]\\
&-2\eta\mu_1  (f( x_t) - f^*) + \frac{2\eta^2\beta_1\mu_2^2}{(1-\beta_1)}( \sigma^2 + G^2)\\
\end{align*}  

By rearranging:
\begin{align*}
 &2\eta\mu_1  (f( x_t) - f^*) \\
 &\leq E[\|z_{t} -x^* \|^2] - E[\|z_{t+1}-x^* \|^2] 
 +2\frac{\eta^2\beta_1^2(\sigma^2+G^2)}{(1-\beta_1)^2}E[\sum_{j=1}^d (\frac{1}{\sqrt{v_{t-1}}+\epsilon})^2-(\frac{1}{\sqrt{v_{t}}+\epsilon})^2]\\
 &+2\eta^2\mu_2^2(\sigma^2+G^2)
+ \frac{\beta_1(\sigma^2+G^2)}{1-\beta_1}E[\sum_{j=1}^d (\frac{1}{\sqrt{v_{t-1,j}}+\epsilon})^2-(\frac{1}{\sqrt{v_{t,j}}+\epsilon})^2]+ \frac{\eta^2\beta_1}{1-\beta_1} E[\|z_t - x^*\|^2]\\
&+ \frac{2\eta^2\beta_1\mu_2^2}{(1-\beta_1)}( \sigma^2 + G^2)\\
\end{align*}

Divide $2\eta\mu_1 $ on both sides,
\begin{align*}
 f( x_t) - f^* &
\leq \frac{1}{2\eta\mu_1 }(E[\|z_{t} -x^* \|^2] - E[\|z_{t+1}-x^* \|^2])  +\frac{\eta\beta_1^2(\sigma^2+G^2)}{(1-\beta_1)^2\mu_1}E[\sum_{j=1}^d (\frac{1}{\sqrt{v_{t-1}}+\epsilon})^2-(\frac{1}{\sqrt{v_{t}}+\epsilon})^2]\\
&+\frac{\eta\mu_2^2}{\mu_1}(\sigma^2+G^2)
+ \frac{\beta_1(\sigma^2+G^2)}{2\eta\mu_1(1-\beta_1)}E[\sum_{j=1}^d (\frac{1}{\sqrt{v_{t-1,j}}+\epsilon})^2-(\frac{1}{\sqrt{v_{t,j}}+\epsilon})^2]\\
&+ \frac{\eta\beta_1}{2\mu_1(1-\beta_1)} E[\|z_t - x^*\|^2]
+ \frac{\eta\beta_1\mu_2^2}{(1-\beta_1)\mu_1}( \sigma^2 + G^2)\\
\end{align*}

Assume that $\forall t$, $E[\|x_{t}-x^* \| \leq D$, for any $m \ne n$, $E[\|x_{m}-x_{n} \|] \leq D_{\infty}$ hold, then $ E[\|z_{t} -x^* \|^2]$ can be bounded.

\begin{align}
E[\|z_{1} -x^* \|^2] &=E[\|x_{1} -x^* \|^2]\leq D^2\\
E[\|z_{t} -x^* \|^2] &= E[\|x_t - x^* + \frac{\beta_1}{1-\beta_1}(x_{t}-x_{t-1})\|^2] \nonumber\\
&\leq 2E[\|x_t - x^*\|^2] + \frac{2\beta_1^2}{(1-\beta_1)^2}E[\|(x_{t}-x_{t-1})\|^2]\nonumber\\
&\leq 2D^2 + \frac{2\beta_1^2}{(1-\beta_1)^2}D_{\infty}^2.
\end{align}

Thus:
\begin{align*}
f( x_t) - f^* &
\leq \frac{1}{2\eta\mu_1 }(E[\|z_{t} -x^* \|^2] - E[\|z_{t+1}-x^* \|^2])  +\frac{\eta\beta_1^2(\sigma^2+G^2)}{(1-\beta_1)^2\mu_1}E[\sum_{j=1}^d (\frac{1}{\sqrt{v_{t-1}}+\epsilon})^2-(\frac{1}{\sqrt{v_{t}}+\epsilon})^2]\\
&+\frac{\eta\mu_2^2}{\mu_1}(\sigma^2+G^2)
+ \frac{\beta_1(\sigma^2+G^2)}{2\eta\mu_1(1-\beta_1)}E[\sum_{j=1}^d (\frac{1}{\sqrt{v_{t-1,j}}+\epsilon})^2-(\frac{1}{\sqrt{v_{t,j}}+\epsilon})^2]\\
&+ \frac{\eta\beta_1D^2 }{\mu_1(1-\beta_1)}+ \frac{\eta\beta_1^3 D_{\infty}^2}{\mu_1(1-\beta_1)^3}
+ \frac{\eta\beta_1\mu_2^2}{(1-\beta_1)\mu_1}( \sigma^2 + G^2)\\
\end{align*} 

Summing from $t=1$ to $T$,

\begin{align*}
\sum_{t=1}^T (f( x_t) - f^*) &
\leq  \frac{1}{2\eta\mu_1 }(E[\|z_{1} -x^* \|^2] - E[\|z_{T}-x^* \|^2] ) +\frac{\eta\beta_1^2(\sigma^2+G^2)}{(1-\beta_1)^2\mu_1}E[\sum_{j=1}^d (\frac{1}{\sqrt{v_{0}}+\epsilon})^2-(\frac{1}{\sqrt{v_{T}}+\epsilon})^2]\\
&+\frac{\eta\mu_2^2 T}{\mu_1}(\sigma^2+G^2)
+ \frac{\beta_1(\sigma^2+G^2)}{2\eta\mu_1(1-\beta_1)}E[\sum_{j=1}^d (\frac{1}{\sqrt{v_{0,j}}+\epsilon})^2-(\frac{1}{\sqrt{v_{T,j}}+\epsilon})^2]\\
&+ \frac{\eta\beta_1D^2 T}{\mu_1(1-\beta_1)}+ \frac{\eta\beta_1^3 D_{\infty}^2 T}{\mu_1(1-\beta_1)^3}
+ \frac{\eta\beta_1\mu_2^2 T}{(1-\beta_1)\mu_1}( \sigma^2 + G^2)\\
&\leq  \frac{1}{2\eta\mu_1 } D^2 +\frac{\eta\beta_1^2 d (\sigma^2+G^2)}{(1-\beta_1)^2\mu_1}(\mu_2^2-\mu_1^2)
+\frac{\eta\mu_2^2 T}{\mu_1}(\sigma^2+G^2)
+ \frac{\beta_1 d(\sigma^2+G^2)}{2\eta\mu_1(1-\beta_1)}(\mu_2^2-\mu_1^2)\\
&+ \frac{\eta\beta_1D^2 T}{\mu_1(1-\beta_1)}+ \frac{\eta\beta_1^3 D_{\infty}^2 T}{\mu_1(1-\beta_1)^3}
+ \frac{\eta\beta_1\mu_2^2 T}{(1-\beta_1)\mu_1}( \sigma^2 + G^2)\\
\end{align*}

The second inequality is based on the fact that, when iteration $t$ reaches the maximum number $T$, $x_t$ is  the optimal solution, $z_T = x^*$.

By Jensen's inequality,
\begin{align*}
\frac{1}{T}\sum_{t=1}^T (f(x_t)-f^*)&\geq f(\Bar{x}_t)-f^* ,
\end{align*}
where $\Bar{x}_t = \frac{1}{T}\sum_{t=1}^T x_t$.

Then,
\begin{align*}
f(\Bar{x}_t)-f^* &\leq \frac{D^2}{2\eta\mu_1 T}  +\frac{\eta\beta_1^2 d (\sigma^2+G^2)}{(1-\beta_1)^2\mu_1 T}(\mu_2^2-\mu_1^2)
+\frac{\eta\mu_2^2}{\mu_1}(\sigma^2+G^2)
+ \frac{\beta_1 d(\sigma^2+G^2)}{2\eta\mu_1(1-\beta_1) T}(\mu_2^2-\mu_1^2)\\
&+ \frac{\eta\beta_1D^2 }{\mu_1(1-\beta_1)}+ \frac{\eta\beta_1^3 D_{\infty}^2 }{\mu_1(1-\beta_1)^3}
+ \frac{\eta\beta_1\mu_2^2 }{(1-\beta_1)\mu_1}( \sigma^2 + G^2)\\
\end{align*}

By plugging the stepsize $\eta = O(\frac{1}{\sqrt{T}})$, we complete the proof of \textsc{Adam} in non-strongly convex case.

\begin{align*}
f(\Bar{x}_t)-f^* &\leq \frac{D^2}{2\mu_1\sqrt{T}}  +\frac{\beta_1^2 d (\sigma^2+G^2)}{(1-\beta_1)^2\mu_1 T\sqrt{T}}(\mu_2^2-\mu_1^2)
+\frac{\mu_2^2 }{\mu_1\sqrt{T}}(\sigma^2+G^2)
+ \frac{\beta_1 d(\sigma^2+G^2)}{2\mu_1(1-\beta_1) \sqrt{T}}(\mu_2^2-\mu_1^2)\\
&+ \frac{\beta_1D^2 }{\mu_1(1-\beta_1)\sqrt{T}}+ \frac{\beta_1^3 D_{\infty}^2 }{\mu_1(1-\beta_1)^3\sqrt{T}}
+ \frac{\beta_1\mu_2^2 }{(1-\beta_1)\mu_1\sqrt{T}}( \sigma^2 + G^2)\\
 &= O(\frac{1}{\sqrt{T}})+O(\frac{1}{T\sqrt{T}}) = O(\frac{1}{\sqrt{T}}).
\end{align*}
\end{proof}

\begin{remark}
The leading item of convergence order of \textsc{Adam} should be $O(\frac{\Tilde{C}}{\sqrt{T}})$, where $\Tilde{C} = \frac{D^2}{2\mu_1}+\frac{\mu_2^2 }{\mu_1}(\sigma^2+G^2)
+ \frac{\beta_1 d(\sigma^2+G^2)}{2\mu_1(1-\beta_1) }(\mu_2^2-\mu_1^2) + \frac{\beta_1D^2 }{\mu_1(1-\beta_1)}+ \frac{\beta_1^3 D_{\infty}^2 }{\mu_1(1-\beta_1)^3}
+ \frac{\beta_1\mu_2^2 }{(1-\beta_1)\mu_1}( \sigma^2 + G^2)$. With fixed $L, \sigma, G, \beta_1, D, D_{\infty}$, $\Tilde{C} = O(\frac{d}{\epsilon^2})$, which also contains $\epsilon$ as well as dimension $d$, here with bigger $\epsilon$, the order should be better, this also supports the discussion in our main paper.
\end{remark}

The analysis of \textsc{Sadam} is similar to \textsc{Adam}, by replacing the bounded pairs $(\mu_1, \mu_2)$ with $(\mu_3, \mu_4)$, we briefly give convergence result below.

\begin{proof} 
\textsc{Sadam} case:

\begin{align*}
f(\Bar{x}_t)-f^*&\leq \frac{D^2}{2\eta\mu_3 T}  +\frac{\eta\beta_1^2 d (\sigma^2+G^2)}{(1-\beta_1)^2\mu_3 T}(\mu_4^2-\mu_3^2)
+\frac{\eta\mu_4^2 }{\mu_3}(\sigma^2+G^2)
+ \frac{\beta_1 d(\sigma^2+G^2)}{2\eta\mu_3(1-\beta_1) T}(\mu_4^2-\mu_3^2)\\
&+ \frac{\eta\beta_1D^2 }{\mu_3(1-\beta_1)}+ \frac{\eta\beta_1^3 D_{\infty}^2 }{\mu_3(1-\beta_1)^3}
+ \frac{\eta\beta_1\mu_4^2 }{(1-\beta_1)\mu_3}( \sigma^2 + G^2)\\
\end{align*}

By plugging the stepsize $\eta = O(\frac{1}{\sqrt{T}})$, we get the convergence rate of \textsc{Sadam} in non-strongly convex case.
\begin{align*}
f(\Bar{x}_t)-f^* &\leq \frac{D^2}{2\mu_3\sqrt{T}}  +\frac{\beta_1^2 d (\sigma^2+G^2)}{(1-\beta_1)^2\mu_3 T\sqrt{T}}(\mu_4^2-\mu_3^2)
+\frac{\mu_4^2 }{\mu_3\sqrt{T}}(\sigma^2+G^2)
+ \frac{\beta_1 d(\sigma^2+G^2)}{2\mu_3(1-\beta_1) \sqrt{T}}(\mu_4^2-\mu_3^2)\\
&+ \frac{\beta_1D^2 }{\mu_3(1-\beta_1)\sqrt{T}}+ \frac{\beta_1^3 D_{\infty}^2 }{\mu_3(1-\beta_1)^3\sqrt{T}}
+ \frac{\beta_1\mu_4^2 }{(1-\beta_1)\mu_3\sqrt{T}}( \sigma^2 + G^2)\\
 &= O(\frac{1}{\sqrt{T}})+O(\frac{1}{T\sqrt{T}}) = O(\frac{1}{\sqrt{T}}).
\end{align*}

For brevity, 
\begin{align*}
f(\Bar{x}_t)-f^* 
&= O(\frac{1}{\sqrt{T}}).
\end{align*}
\end{proof}

\begin{remark}
The leading item of convergence order of \textsc{Sadam} should be $O(\frac{\Tilde{C}}{\sqrt{T}})$, where $\Tilde{C} = \frac{D^2}{2\mu_3}+\frac{\mu_4^2d }{\mu_3}(\sigma^2+G^2)
+ \frac{\beta_1 d(\sigma^2+G^2)}{2\mu_3(1-\beta_1) }(\mu_4^2-\mu_3^2) + \frac{\beta_1D^2 }{\mu_3(1-\beta_1)}+ \frac{\beta_1^3 D_{\infty}^2 }{\mu_3(1-\beta_1)^3}
+ \frac{\beta_1\mu_4^2 }{(1-\beta_1)\mu_3}( \sigma^2 + G^2)$. With fixed $L, \sigma, G, \beta_1, D, D_{\infty}$, $\Tilde{C} = O(d\beta log(1+e^\beta))=O(d\beta^2)$, with small $\beta$, the \textsc{Sadam} will be similar to SGD convergence rate, and $\beta$ is a much smaller number compared with $1/\epsilon$, proving that \textsc{Sadam} method perfoms better than \textsc{Adam} in terms of convergence rate.
\end{remark}

\subsection{P-L Condition}

Suppose that strongly convex assumption holds, we can easily deduce the P-L condition (see Lemma \ref{SC}), which shows that P-L condition is much weaker than strongly convex condition. And we further prove the convergence of \textsc{Adam}-type optimizer (\textsc{Adam} and \textsc{Sadam}) under the P-L condition in non-strongly convex case, which can be extended to the strongly convex case as well.

\begin{lemma}\label{SC}
Suppose that $f$ is continuously diffentiable and strongly convex with parameter $\gamma$. Then $f$ has the unique minimizer, denoted as $f^* = f(x^*)$. Then for any $x\in \mathbb{R}^d$, we have
\begin{equation*}
\|\nabla f (x)\|^2 \geq 2\gamma (f(x)-f^*).
\end{equation*}
\end{lemma}
\begin{proof}
	From  strongly convex assumption,
\begin{align*}
	f^* &\geq f(x)+\nabla f(x)^{T}(x^{*}-x)+\frac{\gamma}{2}\|x^{*}-x\|^{2}\\
	&\geq f(x)+\min_{\xi}(\nabla f(x)^{T}\xi+\frac{\gamma}{2}\|\xi\|^{2})\\
	&= f(x)-\frac{1}{2\gamma}\|\nabla f(x)\|^{2}
\end{align*}
Letting $\xi=x^{*}-x,$ when $\xi=-\frac{\nabla f(x)}{\gamma}$, the quadratic function can achieve its minimum.
\end{proof}

We restate our theorems under PL condition.
\begin{theorem}
Suppose $f(x)$ satisfies Assumption 1 and PL condition (with parameter $\lambda$) in non-strongly convex case and $v_t\geq v_{t-1}$. Let $\eta_t = \eta = O(\frac{1}{T})$, \textsc{Adam} and \textsc{Sadam} have convergence rate
\begin{align*}
E[f(x_t)-f^*]\leq O(\frac{1}{T}).
\end{align*}
\end{theorem}

\begin{proof}
\textsc{Adam} case: 

Starting from L-smoothness, and borrowing the previous results we already have 

\begin{align*}
E[f(z_{t+1})-f(z_t)] & \leq \frac{\eta \beta_1}{1-\beta_1} G\sqrt{\sigma^2+G^2} E[\sum_{j=1}^d \frac{1}{\sqrt{v_{t-1,j}}+\epsilon}-\frac{1}{\sqrt{v_{t,j}}+\epsilon}] \\
&+ \frac{L^2\eta^2\mu_2^2 }{2}(\frac{\beta_1}{1-\beta_1})^2 (\sigma^2+G^2) +  \frac{\eta^2\mu_2^2}{2}( \sigma^2 + G^2)
- E\langle\nabla f(x_t), \frac{\eta}{\sqrt{v_t}+\epsilon} \odot g_t\rangle\\
&+\frac{L\eta^2\beta_1^2(\sigma^2+G^2)}{(1-\beta_1)^2}E[\sum_{j=1}^d (\frac{1}{\sqrt{v_{t-1,j}}+\epsilon})^2-(\frac{1}{\sqrt{v_{t,j}}+\epsilon})^2]
+L\eta^2\mu_2^2(\sigma^2+G^2)\\
E\langle\nabla f(x_t), \frac{1}{\sqrt{v_t}+\epsilon}&\odot g_t\rangle 
\geq \mu_1 \| \nabla f( x_t)\|^2
\end{align*}

Therefore, we get: 
\begin{align*}
E[f(z_{t+1})-f(z_t)] & \leq \frac{\eta \beta_1}{1-\beta_1} G\sqrt{\sigma^2+G^2} E[\sum_{j=1}^d \frac{1}{\sqrt{v_{t-1,j}}+\epsilon}-\frac{1}{\sqrt{v_{t,j}}+\epsilon}] \\
&+ \frac{L^2\eta^2\mu_2^2 }{2}(\frac{\beta_1}{1-\beta_1})^2 (\sigma^2+G^2) +  \frac{\eta^2\mu_2^2}{2}( \sigma^2 + G^2)
- \eta\mu_1 \|\nabla f(x_t)\|^2\\
&+\frac{L\eta^2\beta_1^2(\sigma^2+G^2)}{(1-\beta_1)^2}E[\sum_{j=1}^d (\frac{1}{\sqrt{v_{t-1,j}}+\epsilon})^2-(\frac{1}{\sqrt{v_{t,j}}+\epsilon})^2]
+L\eta^2\mu_2^2(\sigma^2+G^2)\\
\end{align*}

From P-L condition assumption,

\begin{align*}
E[f(z_{t+1})] & \leq E[f(z_t)]+ \frac{\eta \beta_1}{1-\beta_1} G\sqrt{\sigma^2+G^2} E[\sum_{j=1}^d \frac{1}{\sqrt{v_{t-1,j}}+\epsilon}-\frac{1}{\sqrt{v_{t,j}}+\epsilon}] \\
&+ \frac{L^2\eta^2\mu_2^2 }{2}(\frac{\beta_1}{1-\beta_1})^2 (\sigma^2+G^2) +  \frac{\eta^2\mu_2^2}{2}( \sigma^2 + G^2)
- 2\lambda\eta\mu_1 E[f(x_t) - f^*] \\
&+\frac{L\eta^2\beta_1^2(\sigma^2+G^2)}{(1-\beta_1)^2}E[\sum_{j=1}^d (\frac{1}{\sqrt{v_{t-1,j}}+\epsilon})^2-(\frac{1}{\sqrt{v_{t,j}}+\epsilon})^2]
+L\eta^2\mu_2^2(\sigma^2+G^2)\\
\end{align*}

From convexity,
\begin{align*}
f(z_{t+1}) &\geq f(x_{t+1}) + \frac{\beta_1}{1-\beta_1}<\nabla f(x_{t+1}), x_{t+1} - x_{t}> \\
&=  f(x_{t+1}) + \frac{\beta_1}{1-\beta_1}<\nabla f(x_{t+1}),\frac{\eta}{\sqrt{v_t}+\epsilon}\odot m_{t}>
\end{align*}
From L-smoothness,
\begin{align*} 
f(z_t)\leq f(x_t)+ \frac{\beta_1}{1-\beta_1}<\nabla f(x_t), x_t - x_{t-1}> +\frac{L}{2}(\frac{\beta_1}{1-\beta_1})^2\|x_{t}-x_{t-1}\|^2.
\end{align*}

Then we can obtain 

\begin{align*}
 & E[f(x_{t+1})] + \frac{\beta_1}{1-\beta_1} E[<\nabla f(x_{t+1}),\frac{\eta}{\sqrt{v_t}+\epsilon}\odot m_{t}>]\\
& \leq E[f(x_t)]+ \frac{\beta_1}{1-\beta_1}E[<\nabla f(x_t), x_t - x_{t-1}>] +\frac{L}{2}(\frac{\beta_1}{1-\beta_1})^2E[\|x_{t}-x_{t-1}\|^2]\\
&+ \frac{\eta \beta_1}{1-\beta_1} G\sqrt{\sigma^2+G^2} E[\sum_{j=1}^d \frac{1}{\sqrt{v_{t-1,j}}+\epsilon}-\frac{1}{\sqrt{v_{t,j}}+\epsilon}] \\
&+ \frac{L^2\eta^2\mu_2^2 }{2}(\frac{\beta_1}{1-\beta_1})^2 (\sigma^2+G^2) +  \frac{\eta^2\mu_2^2}{2}( \sigma^2 + G^2)
- 2\lambda\eta\mu_1 E[f(x_t) - f^*] \\
&+\frac{L\eta^2\beta_1^2(\sigma^2+G^2)}{(1-\beta_1)^2}E[\sum_{j=1}^d (\frac{1}{\sqrt{v_{t-1,j}}+\epsilon})^2-(\frac{1}{\sqrt{v_{t,j}}+\epsilon})^2]
+L\eta^2\mu_2^2(\sigma^2+G^2)\\
&=  E[f(x_t)]+\frac{\beta_1}{1-\beta_1}E [<\nabla f(x_t), \frac{\eta}{\sqrt{v_{t-1}}+\epsilon} \odot m_{t-1}>] +\frac{L\eta^2}{2}(\frac{\beta_1}{1-\beta_1})^2
E[\|\frac{1}{\sqrt{v_{t-1}}+\epsilon} \odot m_{t-1}\|^2 ]\\
&+ \frac{\eta \beta_1}{1-\beta_1} G\sqrt{\sigma^2+G^2} E[\sum_{j=1}^d \frac{1}{\sqrt{v_{t-1,j}}+\epsilon}-\frac{1}{\sqrt{v_{t,j}}+\epsilon}] \\
&+ \frac{L^2\eta^2\mu_2^2 }{2}(\frac{\beta_1}{1-\beta_1})^2 (\sigma^2+G^2) +  \frac{\eta^2\mu_2^2}{2}( \sigma^2 + G^2)
- 2\lambda\eta\mu_1 E[f(x_t) - f^*] \\
&+\frac{L\eta^2\beta_1^2(\sigma^2+G^2)}{(1-\beta_1)^2}E[\sum_{j=1}^d (\frac{1}{\sqrt{v_{t-1,j}}+\epsilon})^2-(\frac{1}{\sqrt{v_{t,j}}+\epsilon})^2]
+L\eta^2\mu_2^2(\sigma^2+G^2)\\
\end{align*}

By rearranging, 

\begin{align*} 
E[f(x_{t+1})] 
&\leq E[f(x_t)]+\frac{\beta_1\eta}{1-\beta_1} (E [<\nabla f(x_t), \frac{1}{\sqrt{v_{t-1}}+\epsilon} \odot m_{t-1}>] -E [<\nabla f(x_{t+1}), \frac{1}{\sqrt{v_{t}}+\epsilon} \odot m_{t}>]) \\
&+\frac{L\eta^2}{2}(\frac{\beta_1}{1-\beta_1})^2 E[\|\frac{1}{\sqrt{v_{t-1}}+\epsilon} \odot m_{t-1}\|^2 ]
+\frac{\eta \beta_1}{1-\beta_1} G\sqrt{\sigma^2+G^2} E[\sum_{j=1}^d \frac{1}{\sqrt{v_{t-1,j}}+\epsilon}-\frac{1}{\sqrt{v_{t,j}}+\epsilon}] \\
&+ \frac{L^2\eta^2\mu_2^2 }{2}(\frac{\beta_1}{1-\beta_1})^2 (\sigma^2+G^2) +  \frac{\eta^2\mu_2^2}{2}( \sigma^2 + G^2)
- 2\lambda\eta\mu_1 E[f(x_t) - f^*] \\
&+\frac{L\eta^2\beta_1^2(\sigma^2+G^2)}{(1-\beta_1)^2}E[\sum_{j=1}^d (\frac{1}{\sqrt{v_{t-1,j}}+\epsilon})^2-(\frac{1}{\sqrt{v_{t,j}}+\epsilon})^2]
+L\eta^2\mu_2^2(\sigma^2+G^2)\\
\end{align*}

From the fact $\pm <a,b>\leq \frac{1}{2}a^2 +\frac{1}{2}b^2$, and Lemma \ref{lemma:elementwise}, \ref{m_t},

\begin{align*}
E[<\nabla f(x_t), \frac{1}{\sqrt{v_{t-1}}+\epsilon} \odot m_{t-1}>] &=  E[< \nabla f(x_{t+1})\odot \sqrt{\frac{1}{\sqrt{v_{t-1}}+\epsilon}},m_{t} \odot \sqrt{\frac{1}{\sqrt{v_{t-1}}+\epsilon}}>] \\
&\leq \frac{G^2\mu_2}{2} + \frac{(\sigma^2+G^2)\mu_2}{2} 
\leq (\sigma^2+G^2)\mu_2 \\
\end{align*}
Similar, 
\begin{align*}
-E[<\nabla f(x_{t+1}), \frac{1}{\sqrt{v_{t}}+\epsilon} \odot m_{t}>] &=  -E [< \nabla f(x_{t+1})\odot \sqrt{\frac{1}{\sqrt{v_{t-1}}+\epsilon}},m_{t} \odot \sqrt{\frac{1}{\sqrt{v_{t-1}}+\epsilon}}>] \\
&\leq \frac{G^2\mu_2}{2} + \frac{(\sigma^2+G^2)\mu_2}{2} 
\leq (\sigma^2+G^2)\mu_2 \\ 
\end{align*}

Then,

\begin{align*} 
E[f(x_{t+1})] 
&\leq E[f(x_t)]+\frac{2\beta_1\eta\mu_2}{1-\beta_1}(\sigma^2+G^2) +\frac{L\eta^2\mu_2^2}{2}(\frac{\beta_1}{1-\beta_1})^2 (\sigma^2+G^2)\\
&+\frac{\eta \beta_1}{1-\beta_1} G\sqrt{\sigma^2+G^2} E[\sum_{j=1}^d \frac{1}{\sqrt{v_{t-1,j}}+\epsilon}-\frac{1}{\sqrt{v_{t,j}}+\epsilon}] \\
&+ \frac{L^2\eta^2\mu_2^2 }{2}(\frac{\beta_1}{1-\beta_1})^2 (\sigma^2+G^2) +  \frac{\eta^2\mu_2^2}{2}( \sigma^2 + G^2)
- 2\lambda\eta\mu_1 E[f(x_t) - f^*] \\
&+\frac{L\eta^2\beta_1^2(\sigma^2+G^2)}{(1-\beta_1)^2}E[\sum_{j=1}^d (\frac{1}{\sqrt{v_{t-1,j}}+\epsilon})^2-(\frac{1}{\sqrt{v_{t,j}}+\epsilon})^2]
+L\eta^2\mu_2^2(\sigma^2+G^2)\\
\end{align*}

\begin{align*}
E[f(x_{t+1}) - f^*]
&\leq(1- 2\lambda\eta\mu_1 ) E[f(x_t) -f^*] +\frac{2\beta_1\eta\mu_2}{1-\beta_1}(\sigma^2+G^2) +\frac{L\eta^2\mu_2^2}{2}(\frac{\beta_1}{1-\beta_1})^2 (\sigma^2+G^2)\\
&+\frac{\eta \beta_1}{1-\beta_1} G\sqrt{\sigma^2+G^2} E[\sum_{j=1}^d \frac{1}{\sqrt{v_{t-1,j}}+\epsilon}-\frac{1}{\sqrt{v_{t,j}}+\epsilon}] \\
&+ \frac{L^2\eta^2\mu_2^2 }{2}(\frac{\beta_1}{1-\beta_1})^2 (\sigma^2+G^2) +  \frac{\eta^2\mu_2^2}{2}( \sigma^2 + G^2)\\
&+\frac{L\eta^2\beta_1^2(\sigma^2+G^2)}{(1-\beta_1)^2}E[\sum_{j=1}^d (\frac{1}{\sqrt{v_{t-1,j}}+\epsilon})^2-(\frac{1}{\sqrt{v_{t,j}}+\epsilon})^2]
+L\eta^2\mu_2^2(\sigma^2+G^2)\\
&\leq (1- 2\lambda\eta\mu_1 ) E[f(x_t) -f^*] +( \frac{2\beta_1\eta\mu_2}{1-\beta_1} +\frac{L\eta^2\mu_2^2}{2}(\frac{\beta_1}{1-\beta_1})^2\\
&+\frac{\eta \beta_1}{1-\beta_1} E[\sum_{j=1}^d \frac{1}{\sqrt{v_{t-1,j}}+\epsilon}-\frac{1}{\sqrt{v_{t,j}}+\epsilon}] 
+ \frac{L^2\eta^2\mu_2^2 }{2}(\frac{\beta_1}{1-\beta_1})^2 +\frac{\eta^2\mu_2^2}{2}\\
&+\frac{L\eta^2\beta_1^2}{(1-\beta_1)^2}E[\sum_{j=1}^d (\frac{1}{\sqrt{v_{t-1,j}}+\epsilon})^2-(\frac{1}{\sqrt{v_{t,j}}+\epsilon})^2]
+L\eta^2\mu_2^2) (\sigma^2+G^2)\\
\end{align*}
The last inequality holds because $G\sqrt{\sigma^2+G^2} \leq \sigma^2+G^2$.

Let 
\begin{align*}
    \theta &= 1- 2\lambda\eta\mu_1\\
    \Theta_t &=( \frac{2\beta_1\eta\mu_2}{1-\beta_1} +\frac{L\eta^2\mu_2^2}{2}(\frac{\beta_1}{1-\beta_1})^2
 +\frac{\eta \beta_1}{1-\beta_1} E[\sum_{j=1}^d \frac{1}{\sqrt{v_{t-1,j}}+\epsilon}-\frac{1}{\sqrt{v_{t,j}}+\epsilon}] 
+ \frac{L^2\eta^2\mu_2^2 }{2}(\frac{\beta_1}{1-\beta_1})^2\\ &+\frac{\eta^2\mu_2^2}{2}
+\frac{L\eta^2\beta_1^2}{(1-\beta_1)^2}E[\sum_{j=1}^d (\frac{1}{\sqrt{v_{t-1,j}}+\epsilon})^2-(\frac{1}{\sqrt{v_{t,j}}+\epsilon})^2]
+L\eta^2\mu_2^2) (\sigma^2+G^2)
\end{align*}
then we have
\begin{align*}
E[f(x_{t+1})-f^*]&\leq \theta E[f(x_t)-f^*]+\Theta_t.
\end{align*}

Let $\Phi_t = E[f(x_t)-f^*]$, then $\Phi_1 = E[f(x_1)-f^*]$,
 
\begin{align*}
\Phi_{t+1}&\leq \theta\Phi_t + \Theta_t \leq \theta^2 \Phi_{t-1} +\theta \Theta_{t-1}+\Theta_{t}\\
&\cdots\\
&\leq \theta^t \Phi_{1} + \theta^{t-1}\Theta_1+\cdots+ \theta \Theta_{t-1}+\Theta_{t}\\
&\overset{\theta<1}{\leq} \theta^t \Phi_{1} + \Theta_1+\cdots+ \Theta_{t-1}+\Theta_{t}.
\end{align*}

Let $t=T$, 
\begin{align*}
\Phi_{T+1}&\leq \theta^T \Phi_{1} + \Theta_1+\cdots+ \Theta_{T-1}+\Theta_{T}\\
&\leq \theta^T \Phi_{1} + ( \frac{2\beta_1\eta\mu_2 T}{1-\beta_1} +\frac{L\eta^2\mu_2^2 T}{2}(\frac{\beta_1}{1-\beta_1})^2
 +\frac{\eta \beta_1}{1-\beta_1} E[\sum_{j=1}^d \frac{1}{\sqrt{v_{0,j}}+\epsilon}-\frac{1}{\sqrt{v_{T,j}}+\epsilon}]\\ 
&+ \frac{L^2\eta^2\mu_2^2T }{2}(\frac{\beta_1}{1-\beta_1})^2 +\frac{\eta^2\mu_2^2T}{2}\\
&+\frac{L\eta^2\beta_1^2}{(1-\beta_1)^2}E[\sum_{j=1}^d (\frac{1}{\sqrt{v_{0,j}}+\epsilon})^2-(\frac{1}{\sqrt{v_{T,j}}+\epsilon})^2]
+L\eta^2\mu_2^2 T) (\sigma^2+G^2)\\
&\leq \theta^T \Phi_{1} + ( \frac{2\beta_1\eta\mu_2 T}{1-\beta_1} +\frac{L\eta^2\mu_2^2 T}{2}(\frac{\beta_1}{1-\beta_1})^2
 +\frac{\eta \beta_1 d}{1-\beta_1} (\mu_2-\mu_1)
+ \frac{L^2\eta^2\mu_2^2 T}{2}(\frac{\beta_1}{1-\beta_1})^2\\ &+\frac{\eta^2\mu_2^2 T}{2}
+\frac{L\eta^2\beta_1^2 d}{(1-\beta_1)^2}(\mu_2^2-\mu_1^2)
+L\eta^2\mu_2^2 T) (\sigma^2+G^2)\\
&=  \theta^T \Phi_{1} +O(\eta T) + O(\eta^2 T) + O(\eta)+O(\eta^2) \\
\end{align*}

From the above inequality, $\eta$ should be set less than $O(\frac{1}{T})$ to ensure all items in the RHS small enough. 

Set $\eta =\frac{1}{T^2}$, then $\theta = 1- 2\lambda\eta\mu_1=1-\frac{2\lambda\mu_1}{T^2} $
\begin{align*}
\Phi_{T+1}&= \theta^T \Phi_{1} +O(\frac{1}{T}) + O(\frac{1}{T^3}) + O(\frac{1}{T^2})  +O(\frac{1}{T^4}) \\
&= \theta^T \Phi_{1} +O(\frac{1}{T}) \longrightarrow{0}
\end{align*}

With appropriate $\eta$, we can derive the convergence rate under P-L condition (strongly convex) case. 

The proof of \textsc{Sadam} is exactly same as \textsc{Adam}, by replacing the bounded pairs $(\mu_1, \mu_2)$ with $(\mu_3, \mu_4)$, and we can also get:
\begin{align*}
\Phi_{T+1}&\leq \theta^T \Phi_{1} + \Theta_1+\cdots+ \Theta_{T-1}+\Theta_{T}\\
&\leq \theta^T \Phi_{1} + ( \frac{2\beta_1\eta\mu_4 T}{1-\beta_1} +\frac{L\eta^2\mu_4^2 T}{2}(\frac{\beta_1}{1-\beta_1})^2
 +\frac{\eta \beta_1}{1-\beta_1} E[\sum_{j=1}^d \frac{1}{softplus(v_{0,j})}-\frac{1}{softplus(v_{T,j})}]\\
 &+ \frac{L^2\eta^2\mu_4^2T }{2}(\frac{\beta_1}{1-\beta_1})^2 +\frac{\eta^2\mu_4^2T}{2}\\
&+\frac{L\eta^2\beta_1^2}{(1-\beta_1)^2}E[\sum_{j=1}^d (\frac{1}{softplus(v_{0,j})})^2-(\frac{1}{softplus(v_{T,j})})^2]
+L\eta^2\mu_4^2 T) (\sigma^2+G^2)\\
&\leq \theta^T \Phi_{1} + ( \frac{2\beta_1\eta\mu_4 T}{1-\beta_1} +\frac{L\eta^2\mu_4^2 T}{2}(\frac{\beta_1}{1-\beta_1})^2
 +\frac{\eta \beta_1 d}{1-\beta_1} (\mu_4-\mu_3)
+ \frac{L^2\eta^2\mu_4^2 T}{2}(\frac{\beta_1}{1-\beta_1})^2\\ &+\frac{\eta^2\mu_4^2 T}{2}
+\frac{L\eta^2\beta_1^2 d}{(1-\beta_1)^2}(\mu_4^2-\mu_3^2)
+L\eta^2\mu_4^2 T) (\sigma^2+G^2)\\
&= \theta^T \Phi_{1} +O(\eta T) + O(\eta^2 T) + O(\eta)+O(\eta^2)\\
\end{align*}

By setting appropriate $\eta$, we can also prove the \textsc{Sadam} converges under PL condition (and strongly convex).

Set $\eta = O(\frac{1}{T^2})$,
\begin{align*}
E[f(x_{T+1})-f^*]&\leq (1-\frac{2\lambda \mu_3}{T^2})^T E[f(x_1)-f^*]+O(\frac{1}{T}).
\end{align*}
\end{proof}

Overall, we have proved \textsc{Adam} algorithm and \textsc{Sadam} in all commonly used conditions, our designed algorithms always enjoy the same convergence rate compared with \textsc{Adam}, and even get better results with appropriate choice of $\beta$ defined in {\em softplus} function. The proof procedure can be easily extended to other adaptive gradient algorithms, and theoretical results support the discussion and experiments in our main paper.

\end{document}